\documentclass{article}

\usepackage[final,nonatbib]{neurips_2021}

\usepackage{algorithm}

\usepackage[utf8]{inputenc} 
\usepackage[T1]{fontenc}    
\usepackage{hyperref}       
\usepackage{url}            
\usepackage{booktabs}       
\usepackage{amsfonts}       
\usepackage{nicefrac}       
\usepackage{microtype}      

\usepackage{graphicx}
\usepackage{amsthm}
\usepackage{amsmath}
\usepackage{amssymb}
\usepackage{bbm}
\usepackage[utf8]{inputenc}
\usepackage[english]{babel}
\usepackage{subcaption}
\usepackage[noend]{algorithmic}

\usepackage{makecell}

\usepackage[dvipsnames]{xcolor}

\usepackage{tikz}
\usetikzlibrary{decorations.pathreplacing,calc}
\newcommand{\tikzmark}[1]{\tikz[overlay,remember picture] \node (#1) {};}

\newcommand*{\AddNote}[4]{%
    \begin{tikzpicture}[overlay, remember picture]
        \draw [decoration={brace,amplitude=0.5em},decorate,ultra thick,ForestGreen]
            ($(#3)!(#1.north)!($(#3)-(0,1)$)$) --  
            ($(#3)!(#2.south)!($(#3)-(0,1)$)$)
                node [align=center, text width=1.0cm, pos=0.5, anchor=west] {#4};
    \end{tikzpicture}
}%

\usepackage[toc,page]{appendix}

\usepackage{commath}

\usepackage{cleveref}
\crefformat{footnote}{#2\footnotemark[#1]#3}

\newtheorem{theorem}{Theorem}

\newtheorem{lemma}{Lemma}
\newtheorem{proposition}{Proposition}

\newtheorem{definition}{Definition}

\title{Differentially Private Federated Bayesian Optimization with 
Distributed Exploration}

\author{%
  Zhongxiang Dai$^{\dagger}$, Bryan Kian Hsiang Low$^{\dagger}$, Patrick Jaillet$^{\S}$\\
  Dept. of Computer Science, National University of Singapore, Republic of Singapore$^{\dagger}$\\
  Dept. of Electrical Engineering and Computer Science, MIT, USA$^{\S}$\\
  \texttt{\{daizhongxiang,lowkh\}@comp.nus.edu.sg$^{\dagger}$,jaillet@mit.edu$^{\S}$}\\
 }

\begin{document}

\maketitle

\begin{abstract}
\emph{Bayesian optimization} (BO) has recently been extended to the \emph{federated learning} (FL) setting by the \emph{federated Thompson sampling} (FTS) algorithm, which has promising applications such as federated hyperparameter tuning.
However, FTS is not equipped with a rigorous privacy guarantee which is an important consideration in FL. Recent works have incorporated \emph{differential privacy} (DP) into the training of deep neural networks through a general framework for adding DP to iterative algorithms. Following this general DP framework, our work here integrates DP into FTS to preserve \emph{user-level privacy}. We also leverage the ability of this general DP framework to handle different parameter vectors, as well as the technique of local modeling for BO, to further improve the utility of our algorithm through \emph{distributed exploration} (DE). The resulting \emph{differentially private FTS with DE} (DP-FTS-DE) algorithm is endowed with theoretical guarantees for both the privacy and utility and is amenable to interesting theoretical insights about the \emph{privacy-utility trade-off}. We also use real-world experiments to show that DP-FTS-DE achieves high utility (competitive performance) with a strong privacy guarantee (small privacy loss) and induces a trade-off between privacy and utility.
\end{abstract}


\section{Introduction}
\vspace{-1mm}
\label{sec:introduction}
\emph{Bayesian optimization} (BO) has become popular for optimizing expensive-to-evaluate black-box functions, such as tuning the hyperparameters of \emph{deep neural networks} (DNNs)~\cite{shahriari2016taking}. Motivated by the growing computational capability of edge devices and concerns over sharing the raw data, BO has recently been extended to the \emph{federated learning} (FL) setting~\cite{mcmahan2016communication} to derive the setting of \emph{federated BO} (FBO)~\cite{dai2020federated}. The FBO setting allows multiple agents with potentially heterogeneous objective functions to collaborate in black-box optimization tasks without requiring them to share their raw data. 
For example, mobile phone users can use FBO to collaborate in optimizing the hyperparameters of their DNN models used 
in a smart keyboard application without sharing their sensitive raw data.
Hospitals can use FBO to collaborate with each other when selecting the patients to perform a medical test~\cite{yu2015predicting} without sharing the sensitive patient information.
An important consideration in  FL has been a 
rigorous protection of the privacy
of the users/agents, i.e., how to guarantee that by participating in a FL system, an agent would not reveal sensitive information about itself~\cite{kairouz2019advances}.
Furthermore, incorporating rigorous privacy preservation into BO has recently attracted increasing attention due to its importance to 
real-world BO applications~\cite{kharkovskii2020private,kusner2015differentially,dai2018privacy,zhou2020local}.
However, the state-of-the-art algorithm in the FBO setting, \emph{federated Thompson sampling} (FTS)~\cite{dai2020federated}, is not equipped with privacy guarantee and thus lacks rigorous protection of the sensitive agent information.

\emph{Differential privacy} (DP)~\cite{dwork2014algorithmic} provides a rigorous privacy guarantee for data release and has become the state-of-the-art method for designing privacy-preserving ML algorithms~\cite{ji2014differential}. 
Recently, DP has been applied to the iterative training of DNNs using \emph{stochastic gradient descent} (DP-SGD)~\cite{abadi2016deep} and the FL algorithm of \emph{federated averaging} (DP-FedAvg)~\cite{mcmahan2018learning}, which have achieved competitive performances (utility) with a strong privacy guarantee. 
Notably, these methods have followed a general framework for adding DP to generic iterative algorithms~\cite{mcmahan2018a} (referred to as \emph{the general DP framework} hereafter), which applies a \emph{subsampled Gaussian mechanism} (Sec.~\ref{sec:background}) in every iteration.
For an iterative algorithm (e.g., FedAvg) applied to a database with multiple \emph{records} (e.g., data from multiple users),
the general DP framework~\cite{mcmahan2018a} can hide the participation of any single record in the algorithm in a principled way. 
For example, DP-FedAvg~\cite{mcmahan2018learning} guarantees (with high probability) that an adversary, even with arbitrary side information, cannot infer whether the data from a particular user has been used by the algorithm, hence preserving \emph{user-level privacy}.
Unfortunately, FTS~\cite{dai2020federated} is not amenable to a straightforward integration of the general DP framework~\cite{mcmahan2018a} (Sec.~\ref{subsec:dp_fts}). So, we modify FTS to be compatible with the general DP framework  and hence introduce the DP-FTS algorithm to preserve  user-level privacy in the FBO setting.
In addition to the theoretical challenge of accounting for the impact of 
the integration of DP in our theoretical analysis, 
we have to ensure that DP-FTS preserves the practical performance advantage (utility) of FTS.
To this end, we leverage the ability of the general DP framework to handle different parameter vectors~\cite{mcmahan2018a}, as well as the method of local modeling for BO, to further improve the practical performance (utility) of DP-FTS.
Note that FTS, as well as DP-FTS, is able to achieve better performance (utility) than standard TS by \emph{accelerating exploration} using the information from the other agents (aggregated by the central server)~\cite{dai2020federated}. That is, an agent using FTS/DP-FTS benefits from needing to perform less exploration \emph{in the early stages}.
To improve the utility of DP-FTS even more, we further accelerate exploration in the early stages using our proposed \emph{distributed exploration} technique which is a combination of local modeling for BO and the ability of the general DP framework to handle different parameter vectors.
Specifically, we divide the entire search space into smaller local \emph{sub-regions} and let every agent explore only one local sub-region \emph{at initialization}.
As a result, compared with the entire search space, every agent can explore the local sub-region more effectively because 
its \emph{Gaussian process} (GP) surrogate (i.e., the surrogate used by BO to model the objective function) can model the objective function more accurately in a smaller local sub-region~\cite{eriksson2019scalable}.
Subsequently, in every BO iteration, the central server aggregates the information (vector) for every sub-region separately: For a sub-region, the aggregation (i.e., weighted average) gives more emphasis (i.e., weights) to the information (vectors) from those agents who are assigned to explore this particular sub-region.
Interestingly, this technique can be seamlessly integrated into the general DP framework due to its ability to 
process different parameter vectors (i.e., one vector for every sub-region) while still preserving the interpretation as a single subsampled Gaussian mechanism~\cite{mcmahan2018a} (Sec.~\ref{subsec:dp_fts_de}).\footnote{By analogy, the vectors for different sub-regions in our algorithm play a similar role to the parameters of different layers of a DNN in DP-FedAvg~\cite{mcmahan2018learning}.}
As a result, the information aggregated by the central server can help the agents explore every sub-region (hence the entire search space) more effectively in the early stages and thus significantly improve the practical convergence (utility), as demonstrated in our experiments (Sec.~\ref{sec:experiments}).
We refer to the resulting DP-FTS algorithm with \emph{distributed exploration} (DE) as DP-FTS-DE. 
Note that DP-FTS is a special case of DP-FTS-DE with only one sub-region (i.e., entire search space). So, we will refer to DP-FTS-DE as our main algorithm in the rest of this paper.

In this paper, we introduce the \emph{differentially private FTS with DE} (DP-FTS-DE) algorithm (Sec.~\ref{sec:dp_fts_de}),
the first algorithm with a rigorous guarantee on the user-level privacy in the FBO setting.
DP-FTS-DE guarantees that an adversary cannot infer whether an agent has participated in the algorithm, hence assuring every agent that its participation will not reveal its sensitive information.\footnote{Following~\cite{mcmahan2018learning}, we assume that the central server is trustworthy and that the clients are untrustworthy.}
We provide theoretical guarantees for both the privacy and utility of DP-FTS-DE, which combine to yield a number of elegant theoretical insights about the \emph{privacy-utility trade-off} (Sec.~\ref{sec:theoretical_analysis}).
Next, 
we 
empirically
demonstrate that DP-FTS-DE delivers an effective performance with a strong privacy guarantee and induces a favorable trade-off between privacy and utility in real-world applications (Sec.~\ref{sec:experiments}).

\vspace{-1mm}
\section{Background}
\label{sec:background}
\vspace{-2mm}

\textbf{Bayesian optimization (BO).} \quad 
BO aims to maximize an objective function $f$ on a domain $\mathcal{X}\subset\mathbb{R}^D$ through sequential queries, i.e., to find $\mathbf{x}^*\in\arg\max_{\mathbf{x}\in\mathcal{X}}f(\mathbf{x})$.\footnote{For simplicity, we assume $\mathcal{X}$ to be discrete, but our theoretical analysis can be extended to compact domains through suitable discretizations~\cite{chowdhury2017kernelized}.} Specifically, in iteration $t\in[T]$ (we use $[N]$ to represent $\{1,\ldots,N\}$ for brevity),
an input $\mathbf{x}_t\in\mathcal{X}$ is selected to be queried to yield an output observation: $y_t\triangleq f(\mathbf{x}_t) + \zeta$ where $\zeta$ is sampled from a Gaussian noise with variance $\sigma^2$: $\zeta \sim \mathcal{N}(0,\sigma^2)$.
To select the $\mathbf{x}_t$'s intelligently, BO uses a \emph{Gaussian process} (GP)~\cite{rasmussen2004gaussian} as a surrogate to model the objective function $f$.
A GP is defined by its mean function $\mu$ and kernel function $k$. We assume w.l.o.g.~that $\mu(\mathbf{x})=0$ and $k(\mathbf{x},\mathbf{x}')\leq1,\forall\mathbf{x},\mathbf{x}'\in\mathcal{X}$.
We mainly focus on the widely used \emph{squared exponential} (SE) kernel in this work.
In iteration $t+1$, 
given the first $t$ input-output pairs, the GP posterior is given by $\mathcal{GP}(\mu_{t}(\cdot),\sigma_t^2(\cdot,\cdot))$ where
\begin{equation}
\begin{split}
    \mu_t(\mathbf{x}) \triangleq \mathbf{k}_{t}(\mathbf{x})^\top(\mathbf{K}_{t}+\lambda\mathbf{I})^{-1}\mathbf{y}_{t}\ ,
    \sigma_t^2(\mathbf{x},\mathbf{x}') \triangleq k(\mathbf{x},\mathbf{x}')-\mathbf{k}_{t}(\mathbf{x})^\top(\mathbf{K}_{t}+\lambda\mathbf{I})^{-1}\mathbf{k}_{t}(\mathbf{x}')
\end{split}
\label{gp_posterior}
\end{equation}
in which $\mathbf{k}_{t}(\mathbf{x})\triangleq (k(\mathbf{x}, \mathbf{x}_{t'}))^{\top}_{t'\in[t]}$, $\mathbf{y}_t\triangleq (y_{t'})^{\top}_{t'\in[t]}$, $\mathbf{K}_{t}\triangleq (k(\mathbf{x}_{t'}, \mathbf{x}_{t''}))_{t',t''\in[t]}$ and 
$\lambda>0$
is a regularization parameter~\cite{chowdhury2017kernelized}.
In iteration $t+1$, the GP posterior~\eqref{gp_posterior} is used to select the next query $\mathbf{x}_{t+1}$. For example, the \emph{Thompson sampling} (TS)~\cite{chowdhury2017kernelized} algorithm firstly samples a function $f_{t+1}$ from the GP posterior~\eqref{gp_posterior} and then chooses $\mathbf{x}_{t+1}=\arg\max_{\mathbf{x}\in\mathcal{X}}f_{t+1}(\mathbf{x})$.
BO algorithms are usually analyzed in terms of \emph{regret}. A hallmark for well-performing BO algorithms is to be asymptotically \emph{no-regret},
which requires the \emph{cumulative regret} $R_T\triangleq \sum^T_{t=1}(f(\mathbf{x}^*)-f(\mathbf{x}_t))$ to grow sub-linearly.

\emph{Random Fourier features} (RFFs)~\cite{rahimi2008random} has been adopted to approximate the kernel function $k$ using $M$-dimensional random features $\boldsymbol{\phi}$: $k(\mathbf{x},\mathbf{x}') \approx \boldsymbol{\phi}(\mathbf{x})^{\top}\boldsymbol{\phi}(\mathbf{x}')$. RFFs offers a high-probability guarantee on the approximation quality: $\sup_{\mathbf{x},\mathbf{x}'\in\mathcal{X}}|k(\mathbf{x},\mathbf{x}') - \boldsymbol{\phi}(\mathbf{x})^{\top}\boldsymbol{\phi}(\mathbf{x}')| \leq \varepsilon$ where $\varepsilon=\mathcal{O}(M^{-1/2})$~\cite{rahimi2008random}.
Of note, RFFs makes it particularly convenient to approximately sample a function from the GP posterior.
Specifically, define $\boldsymbol{\Phi}_t \triangleq (\boldsymbol{\phi}(\mathbf{x}_{t'}))^{\top}_{t'\in[t]}$ (i.e., a $t\times M$-dimensional matrix), $\boldsymbol{\Sigma}_t \triangleq \boldsymbol{\Phi}_t^{\top}\boldsymbol{\Phi}_t+\lambda\mathbf{I}$, and $\boldsymbol{\nu}_t \triangleq \boldsymbol{\Sigma}_t^{-1}\boldsymbol{\Phi}_t^{\top}\mathbf{y}_t$.
To sample a function $\widetilde{f}$ from approximate GP posterior, we only need to sample 
\begin{equation}
\boldsymbol{\omega} \sim \mathcal{N}(\boldsymbol{\nu}_t, \lambda\boldsymbol{\Sigma}_t^{-1})
\label{eq:sample_w}
\end{equation}
and set $\widetilde{f}(\mathbf{x})=\boldsymbol{\phi}(\mathbf{x})^{\top}\boldsymbol{\omega},\forall \mathbf{x}\in\mathcal{X}$.
RFFs has been adopted by FTS~\cite{dai2020federated} since it allows 
avoiding the sharing of raw data and improves the communication efficiency~\cite{dai2020federated}.

\textbf{Federated Bayesian Optimization (FBO).} \quad 
FBO involves 
$N$ agents $\mathcal{A}_1,\ldots,\mathcal{A}_N$.
Every agent $\mathcal{A}_n$
attempts to maximize its objective function $f^{n}: \mathcal{X} \rightarrow \mathbb{R}$, i.e., to find $\mathbf{x}^{n,*}\in\arg\max_{\mathbf{x}\in\mathcal{X}}f^n(\mathbf{x})$, by querying $\mathbf{x}^n_t$ and observing $y^n_t,\forall t\in[T]$.
Without loss of generality, our theoretical analyses mainly focus on the perspective of agent $\mathcal{A}_1$.
That is, we derive an upper bound on the cumulative regret of $\mathcal{A}_1$: $R^1_T\triangleq \sum^T_{t=1}(f^1(\mathbf{x}^{1,*})-f^1(\mathbf{x}^1_t))$ in Sec.~\ref{sec:theoretical_analysis}.
We characterize the similarity between $\mathcal{A}_1$ and $\mathcal{A}_n$ by $d_n\triangleq\max_{\mathbf{x}\in\mathcal{X}}|f^1(\mathbf{x})-f^n(\mathbf{x})|$ such that $d_1=0$ and a smaller $d_n$ indicates a larger degree of similarity between $\mathcal{A}_1$ and $\mathcal{A}_n$.
Following the work of~\cite{dai2020federated}, we assume that all participating agents share the same set of random features $\boldsymbol{\phi}(\mathbf{x}),\forall \mathbf{x}\in\mathcal{X}$.
In our theoretical analysis, we assume that all objective functions have a bounded norm induced by the \emph{reproducing kernel Hilbert space} (RKHS) associated with the kernel $k$, i.e., $\norm{f^n}_k \leq B,\forall n\in[N]$.
The work of~\cite{dai2020federated} has introduced the FTS algorithm for the FBO setting. 
In iteration $t+1$ of FTS, every agent $\mathcal{A}_n$ ($2\leq n \leq N$) samples a vector $\boldsymbol{\omega}_{n,t}$ from its GP posterior~\eqref{eq:sample_w} and sends it to $\mathcal{A}_1$.
Next, with probability $p_t\in(0,1]$, $\mathcal{A}_1$ chooses the next query using a function $f^1_{t+1}$ sampled from its own GP posterior: $\mathbf{x}_{t+1}=\arg\max_{\mathbf{x}\in\mathcal{X}}f^1_{t+1}(\mathbf{x})$; with probability $1-p_t$, $\mathcal{A}_1$ firstly randomly samples an agent $\mathcal{A}_n$ ($2\leq n \leq N$) and then chooses $\mathbf{x}_{t+1}=\arg\max_{\mathbf{x}\in\mathcal{X}}\boldsymbol{\phi}(\mathbf{x})^{\top}\boldsymbol{\omega}_{n,t}$.

\textbf{Differential Privacy (DP).} \quad
DP provides a rigorous framework for privacy-preserving data release~\cite{dwork2006calibrating}.
Consistent with that of~\cite{mcmahan2018learning}, we define two datasets as \emph{adjacent} if they differ by the data of a single user/agent, which leads to the definition of user-level DP:
\begin{definition}
A randomized mechanism $\mathcal{M}:\mathcal{D} \rightarrow \mathcal{R}$ 
satisfies $(\epsilon,\delta)$-DP if for any two adjacent datasets $D_1$ and $D_2$ and any subset of outputs $\mathcal{S}\subset \mathcal{R}$, $\mathbb{P}(\mathcal{M}(D_1) \in \mathcal{S}) \leq e^{\epsilon}\  \mathbb{P}(\mathcal{M}(D_2) \in \mathcal{S}) + \delta\ .$
\vspace{-1mm}
\end{definition}
Here, $\epsilon$ and $\delta$ are DP parameters such that the smaller they are, the better the privacy guarantee.
Intuitively, user-level DP ensures that adding or removing any single user/agent 
has an imperceptible impact on the output of the algorithm.
DP-FedAvg~\cite{mcmahan2018learning} has added user-level DP into the FL setting by adopting a general DP framework~\cite{mcmahan2018a}.
In DP-FedAvg, the central server applies a \emph{subsampled Gaussian mechanism} to the vectors (gradients) received from multiple agents in every iteration $t$:
\vspace{-1mm}
\begin{enumerate}
\itemsep-0.8mm
    \item Select a subset of agents by choosing every agent with a fixed probability $q$,
    \item Clip the vector $\boldsymbol{\omega}_{n,t}$ from every selected agent $n$ so that its $L_2$ norm is upper-bounded by $S$, 
    \item Add Gaussian noise (variance proportional to $S^2$) to the weighted average of clipped vectors.
\end{enumerate} 
\vspace{-1mm}
The central server then broadcasts the vector produced by step $3$ to all agents.
As a result of the general DP framework, they are able to not only provide a rigorous privacy guarantee, but also use the \emph{moments accountant} method~\cite{abadi2016deep} to calculate the \emph{privacy loss}.\footnote{\label{ftnote:privacy:loss}Given a 
$\delta$, the privacy loss is defined as an upper bound on 
$\epsilon$ calculated by the moments accountant method.}
More recently, the work of~\cite{mcmahan2018a} has formalized the methods used by~\cite{abadi2016deep} and~\cite{mcmahan2018learning} as a general DP framework that is applicable to generic iterative algorithms. Notably, the general DP framework~\cite{mcmahan2018a} can naturally process different parameter vectors (e.g., parameters of different layers of a DNN), which is an important property that allows us to integrate distributed exploration (Sec.~\ref{subsec:de}) into our algorithm.


\vspace{-1mm}
\section{Differentially Private FTS with Distributed Exploration}
\label{sec:dp_fts_de}
\vspace{-1mm}
Here we firstly introduce how we modify FTS to integrate DP to derive the DP-FTS algorithm (Sec.~\ref{subsec:dp_fts}).
Then, we describe distributed exploration which can be seamlessly integrated into DP-FTS to further improve utility (Sec.~\ref{subsec:de}). Lastly, we present the complete DP-FTS-DE algorithm (Sec.~\ref{subsec:dp_fts_de}).

\vspace{-1mm}
\subsection{Differentially Private Federated Thompson Sampling (DP-FTS)}
\label{subsec:dp_fts}
\vspace{-1mm}


The original FTS algorithm~\cite{dai2020federated} is not amenable to a straightforward integration of the general DP framework. This is because FTS~\cite{dai2020federated} requires an agent to receive all vectors $\boldsymbol{\omega}_{n,t}$'s from the other agents (Sec.~\ref{sec:background}), so the transformations to the vectors required by the general DP framework (e.g., subsampling and weighted averaging of the vectors) cannot be easily incorporated.
Therefore, we modify FTS by (a) adding a central server for performing the privacy-preserving transformations, and (b) passing a single aggregated vector (instead of one vector from every agent) to the agents.
Fig.~\ref{fig:setting}a illustrates our DP-FTS algorithm which is obtained by integrating the general DP framework into modified FTS.
Every iteration $t$ of DP-FTS consists of the following steps:

\textbf{\textcircled{{\scriptsize 1}}, \textcircled{{\scriptsize 2}} (by Agents):}
Every agent $\mathcal{A}_n$ samples a vector $\boldsymbol{\omega}_{n,t}$~\eqref{eq:sample_w} using its own current history of $t$ input-output pairs (step \textcircled{{\scriptsize 1}}) and sends $\boldsymbol{\omega}_{n,t}$ to the central server (step \textcircled{{\scriptsize 2}}). 

\textbf{\textcircled{{\scriptsize 3}}, \textcircled{{\scriptsize 4}} (by Central Server):}
Next, the central server processes the $N$ received vectors $\boldsymbol{\omega}_{n,t}$ using a subsampled Gaussian mechanism (step \textcircled{{\scriptsize 3}}): 
It (a) selects a subset of agents $\mathcal{S}_t$ by choosing each agent with probability $q$, 
(b) clips the vector $\boldsymbol{\omega}_{n,t}$ of every selected agent s.t.~its $L_2$ norm is upper-bounded by $S$, and 
(c) calculates a \emph{weighted average} of the clipped vectors using weights $\{\varphi_{n},\forall n\in[N]\}$, and adds to it a Gaussian noise.
The final vector $\boldsymbol{\omega}_t$ is then broadcast to all agents (step \textcircled{{\scriptsize 4}}).

\textbf{\textcircled{{\scriptsize 5}} (by Agents):}
After an agent $\mathcal{A}_n$ receives the vector $\boldsymbol{\omega}_t$ from the central server, it can choose the next query $\mathbf{x}^n_{t+1}$ (step \textcircled{{\scriptsize 5}}): With probability $p_{t+1} \in (0,1]$, $\mathcal{A}_n$ chooses $\mathbf{x}^n_{t+1}$ using standard TS by firstly sampling a function $f^n_{t+1}$ from its own GP posterior of $\mathcal{GP}(\mu^{n}_{t}(\cdot),\beta_{t+1}^2\sigma^n_t(\cdot,\cdot)^2)$~\eqref{gp_posterior} where $\beta_{t} \triangleq B+\sigma\sqrt{2(\gamma_{t-1} + 1 + \log(4/\delta)}$,\footnote{\label{footnote:gamma:delta}$\gamma_{t-1}$ is the max information about the objective function from any $t-1$ queries and $\delta\in(0,1)$ (Theorem~\ref{theorem:dp_fts_de}).} and then choosing $\mathbf{x}^n_{t+1}=\arg\max_{\mathbf{x}\in\mathcal{X}}f^n_{t+1}(\mathbf{x})$; with probability $1-p_{t+1}$, $\mathcal{A}_n$ chooses $\mathbf{x}^n_{t+1}$ using $\boldsymbol{\omega}_t$ received from the central server: $\mathbf{x}^n_{t+1}=\arg\max_{\mathbf{x}\in\mathcal{X}}\boldsymbol{\phi}(\mathbf{x})^{\top}\boldsymbol{\omega}_t$.
Consistent with that of~\cite{dai2020federated}, $(p_t)_{t\in\mathbb{Z}^{+}}$ is chosen as a monotonically increasing sequence such that $p_t \in (0,1],\forall t$ and $p_t\rightarrow 1$ as $t\rightarrow \infty$.
This choice of the sequence of $(p_t)_{t\in\mathbb{Z}^{+}}$ ensures that in the early stage (i.e., when $1-p_t$ is large), an agent can leverage the information from the other agents (via $\boldsymbol{\omega}_t$) to improve its convergence by accelerating its exploration. 

After choosing $\mathbf{x}^n_{t+1}$ and observing $y^n_{t+1}$, $\mathcal{A}_n$ adds $(\mathbf{x}^n_{t+1},y^n_{t+1})$ to its history and samples a new vector $\boldsymbol{\omega}_{n,t+1}$ (step \textcircled{{\scriptsize 1}}). 
Next, $\mathcal{A}_n$ sends $\boldsymbol{\omega}_{n,t+1}$ to the central server (step \textcircled{{\scriptsize 2}}), and the algorithm is repeated.
The detailed algorithm will be presented in Sec.~\ref{subsec:dp_fts_de} since DP-FTS is equivalent to the DP-FTS-DE algorithm with $P=1$ sub-region.

\vspace{-1mm}
\subsection{Distributed Exploration (DE)}
\label{subsec:de}
\vspace{-1mm}
To accelerate the practical convergence (utility) of DP-FTS,
we introduce 
DE to further accelerate the exploration in the early stages (Sec.~\ref{sec:introduction}).
At the beginning of BO, a small number of initial points are usually selected from the entire domain using an exploration method (e.g., random search) to warm-start BO. We use DE to allow every agent to explore a smaller local sub-region 
\emph{at initialization}, which is easier to model for the GP surrogate~\cite{eriksson2019scalable}, and leverage the ability of the general DP framework to handle different parameter vectors~\cite{mcmahan2018a} to integrate DE into DP-FTS in a seamless way.

\begin{figure}
     \centering
     \begin{tabular}{ccc}
         \hspace{-6mm}
         \includegraphics[width=0.33\linewidth]{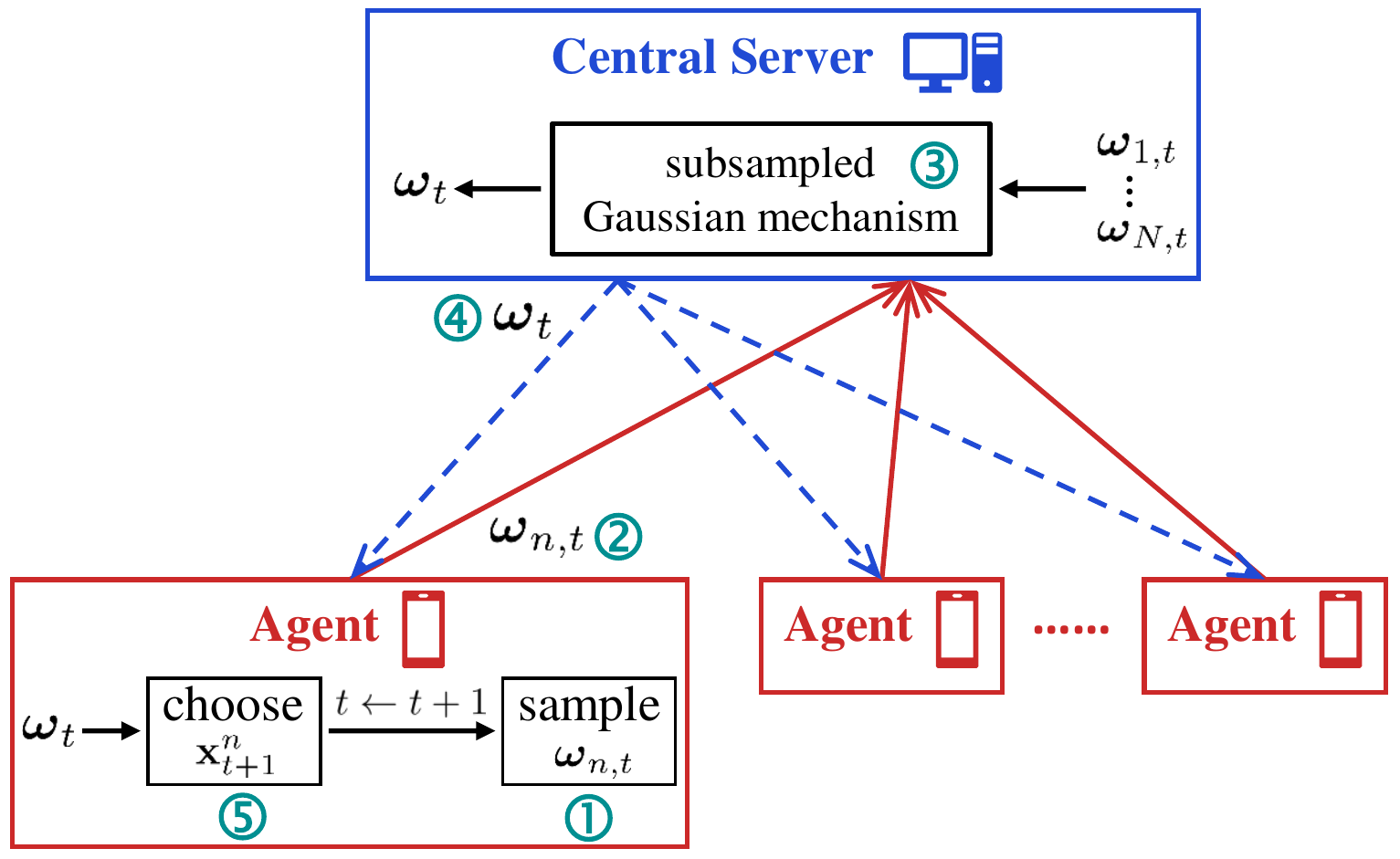}& \hspace{-3mm} 
         \includegraphics[width=0.29\linewidth]{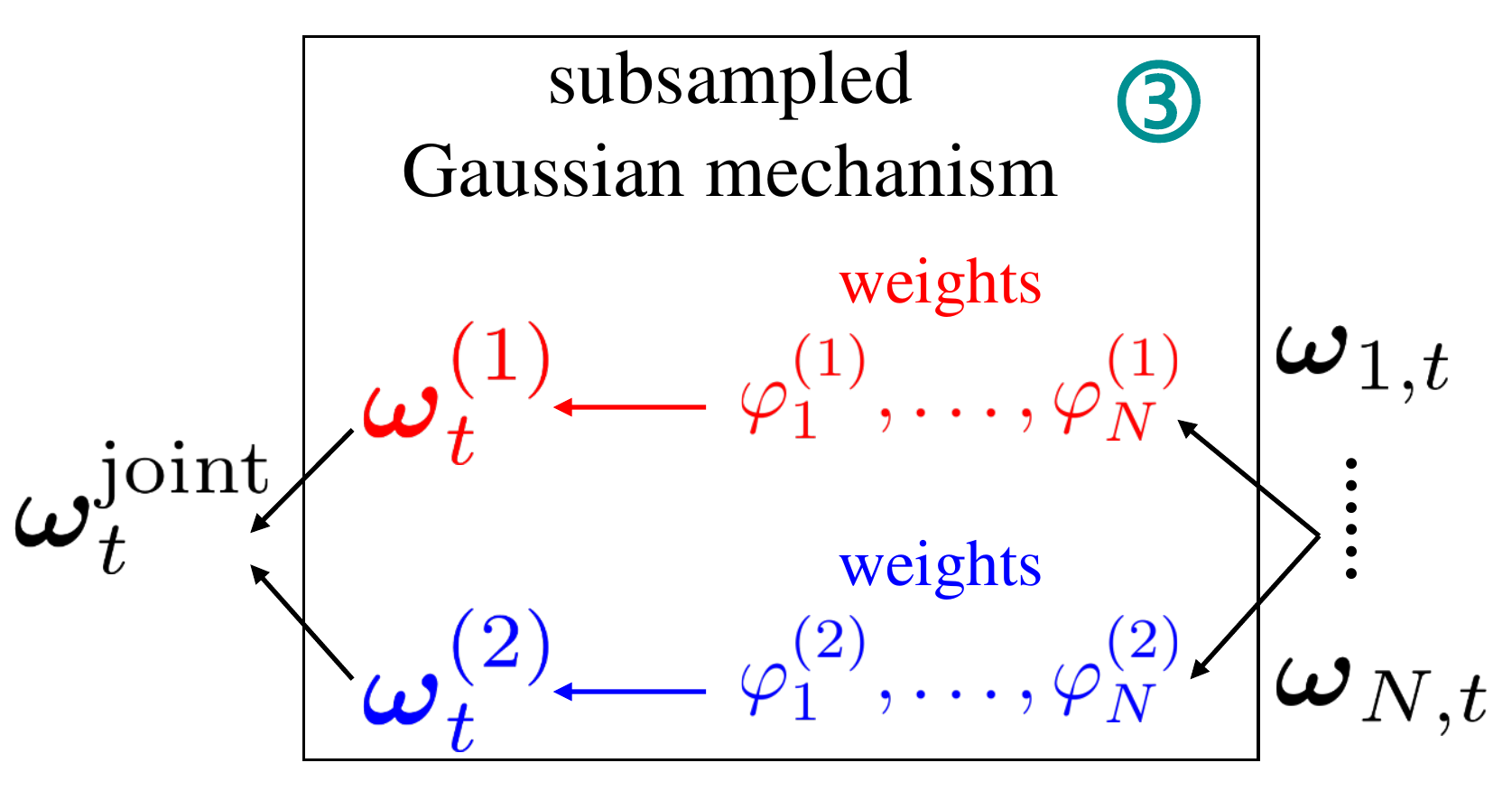}& \hspace{-4mm} 
         \includegraphics[width=0.24\linewidth]{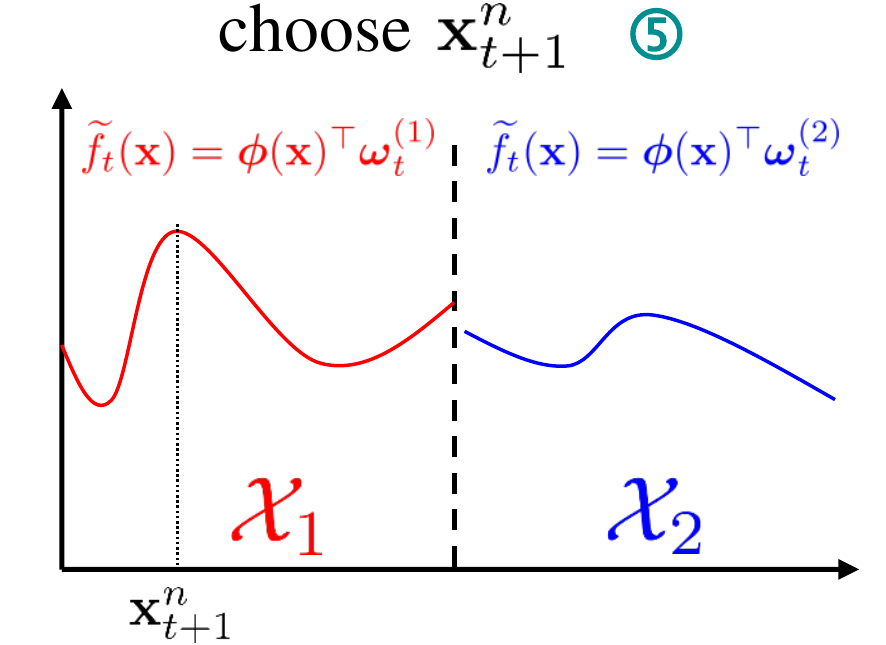}\\
         {(a)} & {(b)} & {(c)}
     \end{tabular}
\vspace{-2mm}
     \caption{
     (a) DP-FTS algorithm (without distributed exploration).
     (b-c) Replacing steps \textcircled{{\scriptsize 3}} and \textcircled{{\scriptsize 5}} in (a)
    with that in (b) and (c) respectively to derive the DP-FTS-DE algorithm ($P=2$).
     }
     \label{fig:setting}
\vspace{-4mm}
\end{figure}

Specifically, we partition the input domain $\mathcal{X}$ into $P\geq1$ disjoint sub-regions: $\mathcal{X}_1,\ldots,\mathcal{X}_P$ such that $\cup_{i=1,\ldots,P}\mathcal{X}_i=\mathcal{X}$ and $\mathcal{X}_i \cap \mathcal{X}_j=\emptyset,\forall i\ne j$.
At initialization, we assign every agent $\mathcal{A}_n$ to explore (i.e., choose the initial points randomly from) a particular sub-region $\mathcal{X}_{i_n}$.\footnote{For simplicity, we choose the sub-regions to be hyper-rectangles with equal volumes and assign an approximately equal number of agents ($\approx N/P$) to explore every sub-region.}
Note that if an agent $\mathcal{A}_n$ is assigned to explore a sub-region $\mathcal{X}_{i_n}$ (instead of exploring the entire domain), its vector $\boldsymbol{\omega}_{n,t}$~\eqref{eq:sample_w} sent to the central server is more informative about its objective function \emph{in this sub-region} $\mathcal{X}_{i_n}$.\footnote{Because its GP surrogate can model the objective function in this local sub-region more accurately~\cite{eriksson2019scalable}.}
As a result, for a sub-region $\mathcal{X}_i$, the vectors from those agents exploring $\mathcal{X}_i$ contain information that can help better explore $\mathcal{X}_i$.
So, we let the central server construct a separate vector $\boldsymbol{\omega}^{(i)}_t$ for every sub-region $\mathcal{X}_i$
and when constructing $\boldsymbol{\omega}^{(i)}_t$, give
more weights to the vectors from those agents exploring $\mathcal{X}_i$ because they are more informative about $\mathcal{X}_i$.
Consequently, the central server needs to construct $P$ different vectors $\{\boldsymbol{\omega}^{(i)}_t,\forall i\in[P]\}$ with each $\boldsymbol{\omega}^{(i)}_t$ using a separate set of weights $\{\varphi^{(i)}_n,\forall n\in[N]\}$.
Interestingly, from the perspective of the general DP framework~\cite{mcmahan2018a}, the different vectors $\{\boldsymbol{\omega}^{(i)}_t,\forall i\in[P]\}$ can be interpreted as analogous to the parameters of different layers of a DNN and can thus be naturally handled by the general DP framework.
After receiving the $P$ vectors from the central server, 
every agent uses $\boldsymbol{\omega}^{(i)}_t$ to reconstruct the sampled function in the sub-region $\mathcal{X}_i$: $\widetilde{f}_t(\mathbf{x})=\boldsymbol{\phi}(\mathbf{x})^{\top}\boldsymbol{\omega}^{(i)}_t,\forall \mathbf{x}\in\mathcal{X}_i$, and then (with probability $1-p_t$) chooses the next query by maximizing the sampled functions from all sub-regions (Fig.~\ref{fig:setting}c); see more details in Sec.~\ref{subsec:dp_fts_de}.

After initialization, every agent is allowed to query any input in the entire domain $\mathcal{X}$ regardless of the sub-region it is assigned to.
So, as $t$ becomes larger, every agent is likely to have explored (and become informative about) more sub-regions in addition to the one it is assigned to.
In this regard, for every sub-region $\mathcal{X}_i$, we make the set of weights $\{\varphi^{(i)}_n,\forall n\in[N]\}$ \emph{adaptive} such that they (a) give more weights to those agents exploring $\mathcal{X}_i$ when $t$ is small and (b) gradually become uniform among all agents as $t$ becomes large. 
We adopt the widely used softmax weighting with temperature $\mathcal{T}$, which has been shown to provide well-calibrated uncertainty for weighted ensemble of GP experts~\cite{cohen2020healing}.
Concretely, 
we let
$\varphi^{(i)}_n=\frac{\exp(( a \mathbb{I}^{(i)}_n + 1)/\mathcal{T})}{\sum^N_{n=1} \exp(( a \mathbb{I}^{(i)}_n + 1)/\mathcal{T})}$,\footnote{ $\mathbb{I}^{(i)}_n$ is an indicator variable that equals $1$ if agent $n$ is assigned to explore $\mathcal{X}_i$ and equals $0$ otherwise. $a>0$ is a constant and we set $a=15$ in all our experiments.
}
and gradually increase the temperature $\mathcal{T}$ from $1$ to $+\infty$ (more details in App.~\ref{app:experiments}).
The dependence of the weights on $t$ only requires minimal modifications to the algorithm and the theoretical results. So, we drop this dependence for simplicity.

\vspace{-1mm}
\subsection{DP-FTS-DE Algorithm}
\label{subsec:dp_fts_de}
\vspace{-1mm}
Our complete DP-FTS-DE algorithm after integrating DE (Sec.~\ref{subsec:de}) into DP-FTS (Sec.~\ref{subsec:dp_fts}) is presented in Algo.~\ref{alg:DP-FTS-DE} (central server's role) and Algo.~\ref{alg:BO} (agent's role), with the steps in circle corresponding to those in Fig.~\ref{fig:setting}a.
DP-FTS-DE differs from DP-FTS in two major aspects:
Firstly, at initialization ($t=0$), every agent only explores a local sub-region instead of the entire domain (line $2$ of Algo.~\ref{alg:BO}). Secondly, instead of a single vector $\boldsymbol{\omega}_t$, the central server produces and broadcasts 
$P$ vectors $\{\boldsymbol{\omega}^{(i)}_t,\forall i\in[P]\}$,
each
corresponding to a different sub-region $\mathcal{X}_i$ and using a different set of weights.
Applying different transformations to different vectors (e.g., parameters of different DNN layers) can be naturally incorporated into the general DP framework~\cite{mcmahan2018a}.
Different transformations performed by our central server to produce $P$ vectors can be interpreted as \emph{a single subsampled Gaussian mechanism} producing a single joint vector $\boldsymbol{\omega}^{\text{joint}}_t\triangleq (\boldsymbol{\omega}^{(i)}_t)_{i\in[P]}$ (Fig.~\ref{fig:setting}b).
Next, we present these transformations performed by the central server (lines $5$-$12$ of Algo.~\ref{alg:DP-FTS-DE}) from this perspective.

\textbf{Subsampling:} 
To begin with, after receiving the vectors $\boldsymbol{\omega}_{n,t}$'s from the agents (lines $3$-$4$ of Algo.~\ref{alg:DP-FTS-DE}), the central server firstly chooses a random subset of agents $\mathcal{S}_t$ by selecting each agent with probability $q$ (line 6).
Next, for every selected agent $n\in\mathcal{S}_t$, the central server constructs a $P\times M$-dimensional joint vector: $\boldsymbol{\omega}^{\text{joint}}_{n,t}\triangleq (N\varphi^{(i)}_n\boldsymbol{\omega}_{n,t})_{i\in[P]}$.

\vspace{-2mm}
\begin{algorithm}[H]
\begin{algorithmic}[1]
    \STATE $\boldsymbol{\omega}^{\text{joint}}_{-1}=\mathbf{0}$
	\FOR{iterations $t=0,1,2,\ldots, T$}
    	\FOR{agents $n=1,2,\ldots, N$ \textbf{in parallel}}
    	    \STATE $\boldsymbol{\omega}_{n,t}\leftarrow$ \textbf{BO-Agent-}$\boldsymbol{\mathcal{A}_n}$($t$, $\boldsymbol{\omega}^{\text{joint}}_{t-1}$) \qquad\quad\quad\quad\, \tikzmark{right} \quad\,{\color{ForestGreen}\textbf{\textcircled{{\scriptsize 2}}}}
    	\ENDFOR
    	\STATE $\boldsymbol{\omega}^{(i)}_t=\mathbf{0},\forall i\in[P]$ \tikzmark{top}
    	\STATE Choose a random subset  $\mathcal{S}_t\subset[N]$ of agents 
    	\FOR{sub-regions $i=1,2,\ldots, P$}
	    \FOR{agents $n\in\mathcal{S}_t$}
            \STATE $\widehat{\boldsymbol{\omega}}_{n,t}=\boldsymbol{\omega}_{n,t}\Big/\max\big(1, \frac{\norm{\boldsymbol{\omega}_{n,t}}_2}{S/\sqrt{P}}\big)$
	        \STATE  $\boldsymbol{\omega}^{(i)}_t \mathrel{+}= (\varphi^{(i)}_n/q)\ \widehat{\boldsymbol{\omega}}_{n,t}$
    	\ENDFOR
    	\STATE $\boldsymbol{\omega}^{(i)}_t \mathrel{+}=\mathcal{N}\left(\mathbf{0},(z\varphi_{\max}S/q)^2\mathbf{I}\right)$ \tikzmark{bottom}
    	\ENDFOR
    	\STATE Broadcast 
	$\boldsymbol{\omega}^{\text{joint}}_t=(\boldsymbol{\omega}^{(i)}_t)_{i\in[P]}$
    	to all agents \quad\quad\quad\,\, {\color{ForestGreen}\textbf{\textcircled{{\scriptsize 4}}}}
        \STATE Update the privacy loss using the moments accountant method~\cref{ftnote:privacy:loss}
	\ENDFOR
\end{algorithmic}
\AddNote{top}{bottom}{right}{\textbf{\textcircled{{\scriptsize 3}}}}
\caption{DP-FTS-DE (central server)}
\label{alg:DP-FTS-DE}
\vspace{-3.5mm}
\end{algorithm}
\vspace{-7.5mm}
\begin{algorithm}[H]
\begin{algorithmic}[1]
    \IF{$t=0$}
        \STATE Randomly select and query $N_{\text{init}}$ initial points from sub-region $\mathcal{X}_{i_n}$ 
    \ELSE
        \STATE Sample $r$ from the uniform distribution over $[0,1]$: $r\sim U(0,1)$ \qquad\quad \tikzmark{right} \tikzmark{top}
        \IF{$r\leq p_t$} 
        \STATE \quad $\mathbf{x}^n_t=\arg\max_{\mathbf{x}\in\mathcal{X}}f^n_t(\mathbf{x})$, where $f^n_t \sim \mathcal{GP}(\mu^n_{t}(\cdot),\beta_{t+1}^2\sigma^n_t(\cdot,\cdot)^2)$
        \ELSE
        \STATE \quad $\mathbf{x}^n_t=\arg\max_{\mathbf{x}\in\mathcal{X}}\boldsymbol{\phi}(\mathbf{x})^{\top}\boldsymbol{\omega}^{(i^{[\mathbf{x}]})}_{t-1}$.\footnotemark \tikzmark{bottom}
        \ENDIF
    \ENDIF
    \STATE Query $\mathbf{x}^n_t$ to observe $y^n_t$
    \STATE Sample $\boldsymbol{\omega}_{n,t}$ and send it to central server \qquad\quad\quad\quad\quad\quad\quad\quad\quad\quad\qquad {\color{ForestGreen}\textbf{\textcircled{{\scriptsize 1}}}, \textbf{\textcircled{{\scriptsize 2}}}}
\end{algorithmic}
\AddNote{top}{bottom}{right}{\textbf{\textcircled{{\scriptsize 5}}}}
\caption{\textbf{BO-Agent-}$\boldsymbol{\mathcal{A}_n}$($t$, $\boldsymbol{\omega}^{\text{joint}}_{t-1}=(\boldsymbol{\omega}^{(i)}_{t-1})_{i\in[P]}$)}
\label{alg:BO}
\vspace{-3.8mm}
\end{algorithm}
\vspace{-5mm}
\footnotetext{$i^{[\mathbf{x}]}$ represents the sub-region $\mathbf{x}$ is assigned to.}

\textbf{Clipping:}
Next, clip every selected vector $\boldsymbol{\omega}_{n,t}$ 
to obtain $\widehat{\boldsymbol{\omega}}_{n,t}$ 
such that $||\widehat{\boldsymbol{\omega}}_{n,t}||_2 \leq S/\sqrt{P}$.
This is equivalent to clipping 
$\boldsymbol{\omega}^{\text{joint}}_{n,t}$
to obtain: $\widehat{\boldsymbol{\omega}}^{\text{joint}}_{n,t}\triangleq (N\varphi^{(i)}_n\widehat{\boldsymbol{\omega}}_{n,t})_{i\in[P]}$, whose $L_2$ norm is bounded by
$||\widehat{\boldsymbol{\omega}}^{\text{joint}}_{n,t}||_2 \leq ({N^2\varphi_{\max}^2\sum^P_{i=1}\norm{\widehat{\boldsymbol{\omega}}_{n,t}}^2_2})^{1/2}\leq N\varphi_{\max}S$ 
where $\varphi_{\max}\triangleq\max_{i\in[P],n\in[N]}\varphi^{(i)}_{n}$.

\textbf{Weighted Average:}
Next, calculate a weighted average of the clipped joint vectors by giving equal weights\footnote{
This weight is only used 
to aid interpretation 
and is different from the weights $\varphi^{(i)}_n$'s used in our algorithm.} to all agents: $\boldsymbol{\omega}^{\text{joint}}_{t}=(qN)^{-1}\sum_{n\in\mathcal{S}_t} \widehat{\boldsymbol{\omega}}^{\text{joint}}_{n,t}$.\footnote{The summation is divided by $qN$ (i.e., expected number of agents selected) to make it unbiased~\cite{mcmahan2018learning}.}
Note that $\boldsymbol{\omega}^{\text{joint}}_{t}$ results from the concatenation of the vectors from all sub-regions: $\boldsymbol{\omega}^{\text{joint}}_{t}= (\boldsymbol{\omega}^{(i)}_t)_{i\in[P]}$ where $\boldsymbol{\omega}^{(i)}_t=({qN})^{-1}\sum_{n\in\mathcal{S}_t} N\varphi^{(i)}_n\widehat{\boldsymbol{\omega}}_{n,t}={q}^{-1}\sum_{n\in\mathcal{S}_t}\varphi^{(i)}_n\widehat{\boldsymbol{\omega}}_{n,t}$ (line $10$ of Algo.~\ref{alg:DP-FTS-DE}).
\\
\textbf{Gaussian Noise:}
Finally, add to each element of $\boldsymbol{\omega}^{\text{joint}}_{t}=(\boldsymbol{\omega}^{(i)}_t)_{i\in[P]}$ 
a zero-mean Gaussian noise with a standard deviation of $z({N\varphi_{\max}S})/({qN})=z\varphi_{\max}S/q$ (line $11$).

Next, 
the output $\boldsymbol{\omega}^{\text{joint}}_t=(\boldsymbol{\omega}^{(i)}_t)_{i\in[P]}$ 
is broadcast to all agents.
After an agent $\mathcal{A}_n$ receives $\boldsymbol{\omega}^{\text{joint}}_t$, 
with probability $p_{t+1}$, $\mathcal{A}_n$ chooses the next query $\mathbf{x}^n_{t+1}$ by maximizing a function $f^n_{t+1}$ sampled from its own GP posterior $\mathcal{GP}(\mu^n_{t}(\cdot),\beta_{t+1}^2\sigma^n_t(\cdot,\cdot)^2)$~\eqref{gp_posterior} (line 6 of Algo.~\ref{alg:BO}, $\beta_{t} \triangleq B+\sigma\sqrt{2(\gamma_{t-1} + 1 + \log(4/\delta)}$\cref{footnote:gamma:delta}); with probability $1-p_{t+1}$, $\mathcal{A}_n$ uses $\boldsymbol{\omega}^{\text{joint}}_t=[\boldsymbol{\omega}^{(i)}_t]_{i\in[P]}$ received from the central server to choose $\mathbf{x}^n_{t+1}$ by maximizing the reconstructed functions for all sub-regions (line 8 of Algo.~\ref{alg:BO}), as illustrated in Fig.~\ref{fig:setting}c.
Finally, it queries $\mathbf{x}^n_{t+1}$ to observe $y^n_{t+1}$, samples a new vector $\boldsymbol{\omega}_{n,t+1}$ and sends it to the central server (line 10 of Algo.~\ref{alg:BO}), and the algorithm is repeated.

\vspace{-1.8mm}
\section{Theoretical Analysis}
\label{sec:theoretical_analysis}
\vspace{-1.8mm}
In this section, 
we present theoretical guarantees on both the 
privacy and utility
of our DP-FTS-DE, which combine to yield interesting insights about the \emph{privacy-utility trade-off}.


\begin{proposition}[Privacy Guarantee]
\label{proposition:dp}
There exist constants $c_1$ and $c_2$ such that 
for fixed $q$ and $T$ and any $\epsilon<c_1q^2T, \delta>0$, DP-FTS-DE (Algo.~\ref{alg:DP-FTS-DE}) is $(\epsilon,\delta)$-DP if 
$z\geq c_2 q\sqrt{T \log(1/\delta)}/\epsilon$.
\vspace{-1mm}
\end{proposition}
Proposition~\ref{proposition:dp} formalizes our privacy guarantee, and its proof (App.~\ref{app:proposition_1}) follows directly from Theorem $1$ of~\cite{abadi2016deep}.
Proposition~\ref{proposition:dp} shows that a larger $z$ (i.e., larger variance for Gaussian noise), a smaller $q$ 
(i.e., smaller expected no.~of selected agents) 
and a smaller $T$ (i.e., smaller no.~of iterations) all improve the privacy guarantee because for a fixed $\delta$, they all allow $\epsilon$ to be smaller.
\begin{theorem}[Utility Guarantee]
\label{theorem:dp_fts_de}
Assume that all $f^n$'s lie in the RKHS of kernel $k$: $\norm{f^n}_k \leq B,\forall n\in[N]$.
Let $\gamma_t$ be the max information gain on $f^1$ from any set of $t$ observations, $\varepsilon$ denote an upper bound on the approximation error of RFFs (Sec.~\ref{sec:background}), $\mathcal{C}_t\triangleq \{n\in[N] \big| ||\boldsymbol{\omega}_{n,t}||_2 > S/\sqrt{P} \}$, $\delta\in(0,1)$, $\lambda=1+2/T$ and $\beta_{t} \triangleq B+\sigma\sqrt{2(\gamma_{t-1} + 1 + \log(4/\delta)}$ (Algo.~\ref{alg:BO}).
Choose $(p_t)_{t\in\mathbb{Z}^+}$ to be monotonically increasing s.t.~$1-p_t=\mathcal{O}(1/t^2)$.
Then, with probability of at least $1 - \delta$ ($\tilde{\mathcal{O}}$ hides all logarithmic factors),
\[
\textstyle R^1_T=\tilde{\mathcal{O}}\Big(\left(B+1/p_1\right)\gamma_T\sqrt{T}+{\sum}^T_{t=1}\psi_t+B\sum^T_{t=1}\vartheta_t\Big)
\]
where $\psi_t\triangleq \tilde{\mathcal{O}}( (1-p_t)P\varphi_{\max}{q}^{-1}(\Delta_{t} +  zS\sqrt{M}))$, $\Delta_{t} \triangleq \sum^N_{n=1} \Delta_{n,t}$, $\Delta_{n, t} \triangleq \tilde{\mathcal{O}}(\varepsilon B t^2 + B + \sqrt{M} + d_n + \sqrt{\gamma_t})$,
and
$\vartheta_t \triangleq (1-p_t)\sum^P_{i=1} \sum_{n\in\mathcal{C}_t}\varphi^{(i)}_n$.
\end{theorem}

%
Theorem~\ref{theorem:dp_fts_de} (proof in App.~\ref{app:proof_regret_bound}) gives an upper bound on the cumulative regret of agent $\mathcal{A}_1$.
Note that all three terms in the regret upper bound grow sub-linearly:
The first term is sub-linear because $\gamma_T=\mathcal{O}((\log T)^{D+1})$ for the SE kernel ($D$ is the dimension of the input space), and the second and third terms are both sub-linear since we have chosen $1-p_t=\mathcal{O}(1/t^2)$.\footnote{The requirement of $1-p_t=\mathcal{O}(1/t^2)$ is needed only to ensure that DP-FTS-DE is no-regret even if all other agents are heterogeneous and is hence a conservative choice. Therefore, it is reasonable to make $1-p_t$ decrease slower in practice as we show in our experiments (Sec.~\ref{sec:experiments}).}
So,
DP-FTS-DE preserves the no-regret property of FTS~\cite{dai2020federated}.
That is, agent $\mathcal{A}_1$ 
is asymptotically \emph{no-regret}
even if all other agents are heterogeneous, i.e., all other agents have significantly different objective functions from $\mathcal{A}_1$.
This is particularly desirable because it ensures the robustness of DP-FTS-DE against the heterogeneity of agents, which is an important challenge in both FL and FBO~\cite{dai2020federated,kairouz2019advances} and has also been an important consideration for other works on the theoretical convergence of FL~\cite{li2018federated,li2019convergence}.
To prove this robust guarantee, we have upper-bounded the \emph{worst-case error} 
induced by the use of
\emph{any} set of agents, which explains the dependence of the regret bound on $d_n\triangleq\max_{\mathbf{x}\in\mathcal{X}}|f^1(\mathbf{x})-f^n(\mathbf{x})|$ and $(p_t)_{t\in\mathbb{Z}^+}$.
Specifically, larger $d_n$'s indicate larger differences between the objective functions of $\mathcal{A}_1$ and the other agents and hence lead to a worse regret upper bound (through the term $\psi_t$).
Smaller values of the sequence $(p_t)_{t\in\mathbb{Z}^+}$ increase the utilization of information from the other agents (line $8$ of Algo.~\ref{alg:BO} becomes more likely to be executed), hence inflating the worst-case error resulting from these information.

Theorem~\ref{theorem:dp_fts_de}, when interpreted together with Proposition~\ref{proposition:dp}, also reveals some interesting theoretical insights regarding the \textbf{privacy-utility trade-off}.
Firstly, a larger $z$ (i.e., larger variance for  Gaussian noise)
improves the privacy guarantee (Proposition~\ref{proposition:dp}) and yet results in a worse utility since it leads to a worse regret upper bound (through $\psi_t$).
Secondly, a larger $q$ (i.e., more selected agents in each iteration) improves the utility since it tightens the regret upper bound (by reducing the value of $\psi_t$) at the expense of a worse privacy guarantee (Proposition~\ref{proposition:dp}).
A general guideline for choosing the values of $z$ and $q$ is to aim for a good utility while ensuring that the privacy loss is within the single-digit range (i.e., $<10$)~\cite{abadi2016deep}.
The value of the clipping threshold $S$ exerts no impact on the privacy guarantee. However, $S$ affects the regret upper bound (hence the utility) through two conflicting effects:
Firstly, a smaller $S$ reduces the value of $\psi_t$ and hence the regret bound. However, a smaller $S$ 
is likely to enlarge the cardinality of the set $\mathcal{C}_t$, hence increasing the value of 
$\vartheta_t$ (Theorem~\ref{theorem:dp_fts_de}).
Intuitively, a smaller $S$ impacts the performance positively by reducing the noise variance (line $11$ of Algo.~\ref{alg:DP-FTS-DE}) and yet negatively by causing more vectors to be clipped (line $9$ of Algo.~\ref{alg:DP-FTS-DE}).
A general guide on the selection of $S$ is to choose a small value while ensuring that a small number of vectors are clipped.

Regarding the dependence of the regret upper bound on the number of random features $M$, in addition to the dependence through $\Delta_{n,t}$ which has been analyzed by~\cite{dai2020federated}, the integration of DP introduces 
another dependence 
that implicitly affects the \textbf{privacy-utility trade-off}.
Besides increasing the value of $\psi_t$, a larger $M$  enlarges the value of $\vartheta_t$ as a larger dimension for the vectors $\boldsymbol{\omega}_{n,t}$'s is expected to increase their $L_2$ norms, hence increasing the cardinality of the set $\mathcal{C}_t$ and consequently the value of $\vartheta_t$ (Theorem~\ref{theorem:dp_fts_de}).
So, the additional dependence due to the integration of DP loosens the regret upper bound with an increasing $M$.
As a result, if $M$ is larger, we can either \emph{sacrifice privacy to preserve utility} by reducing $z$ or increasing $q$ (both can counter the increase of $\psi_t$), or \emph{sacrifice utility to preserve privacy} by keeping $z$ and $q$ unchanged.
The number of sub-regions $P$ also induces a trade-off about the performance of our algorithm, which is partially reflected by Theorem~\ref{theorem:dp_fts_de}.
Specifically, the regret bound depends on $P$ through three terms. 
Two of the terms ($P$ in $\psi_t$ and the summation of $P$ terms in $\vartheta_t$) arise due to the worst-case nature of our regret bound, as discussed earlier: They cause the regret upper bound to increase with $P$ due to the accumulation of the worst-case errors resulting from the $P$ vectors: $\{\boldsymbol{\omega}^{(i)}_t,\forall i\in[P]\}$.
Regarding the third term (in the definition of $\mathcal{C}_t$), a larger $P$ is expected to increase the cardinality of the set $\mathcal{C}_t$ (similar to the effect of a larger $M$ discussed above), consequently loosening the regret upper bound by inflating the value of $\vartheta_t$. In this case, as described above, we can choose to sacrifice either privacy or utility.
On the other hand, a larger $P$ (i.e., larger number of sub-regions) can improve the practical performance (utility) because it allows every agent to explore only a smaller sub-region which can be better modeled by its GP surrogate (Sec.~\ref{subsec:de}).
As a result of the worst-case nature of the regret bound mentioned earlier, the latter positive effect leading to better practical utility is not reflected by Theorem~\ref{theorem:dp_fts_de}.
Therefore, we instead verify this trade-off induced by $P$ about the practical performance in our experiments (Fig.~\ref{fig:effect_of_p_with_dp} in App.~\ref{app:synth_exp}).\vspace{-1.1mm}

\vspace{-1.5mm}
\section{Experiments}
\label{sec:experiments}
\vspace{-1.8mm}
Note that the privacy loss calculated by the moments accountant method is an upper bound on the value of $\epsilon$ for a given value of $\delta$.\cref{ftnote:privacy:loss}
When calculating the privacy loss, 
we follow the practice of~\cite{mcmahan2018learning} and set $\delta=1/N^{1.1}$.
All error bars denote standard errors.
Due to space constraints, some experimental details are deferred to App.~\ref{app:experiments}.

\vspace{-1.2mm}
\subsection{Synthetic Experiments}
\label{subsec:synth_exp}
\vspace{-1.3mm}
\begin{figure}
     \centering
     \begin{tabular}{cccc}
         \includegraphics[width=0.23\linewidth]{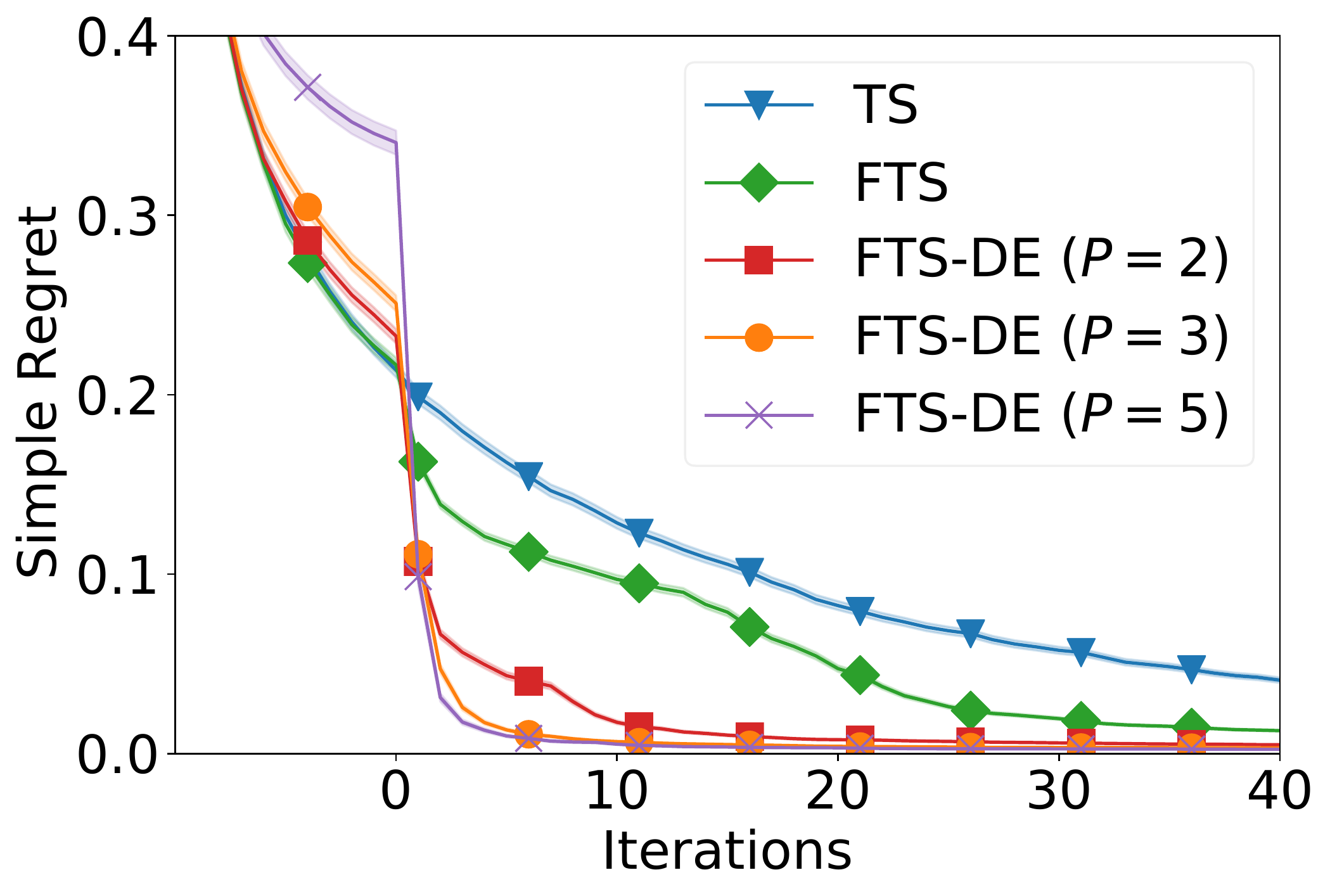} & \hspace{-6mm} 
         \includegraphics[width=0.23\linewidth]{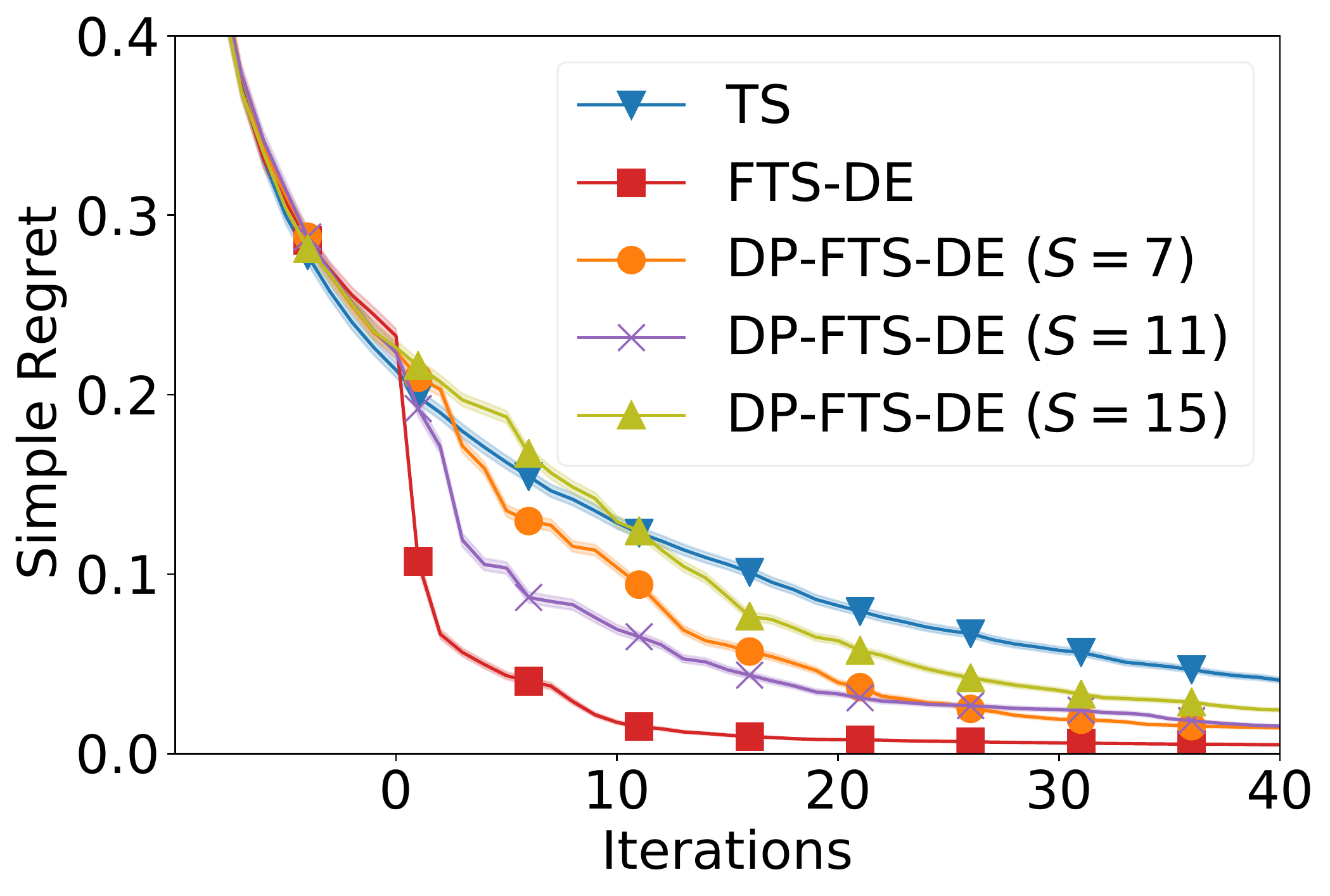}& \hspace{-6mm}
         \includegraphics[width=0.23\linewidth]{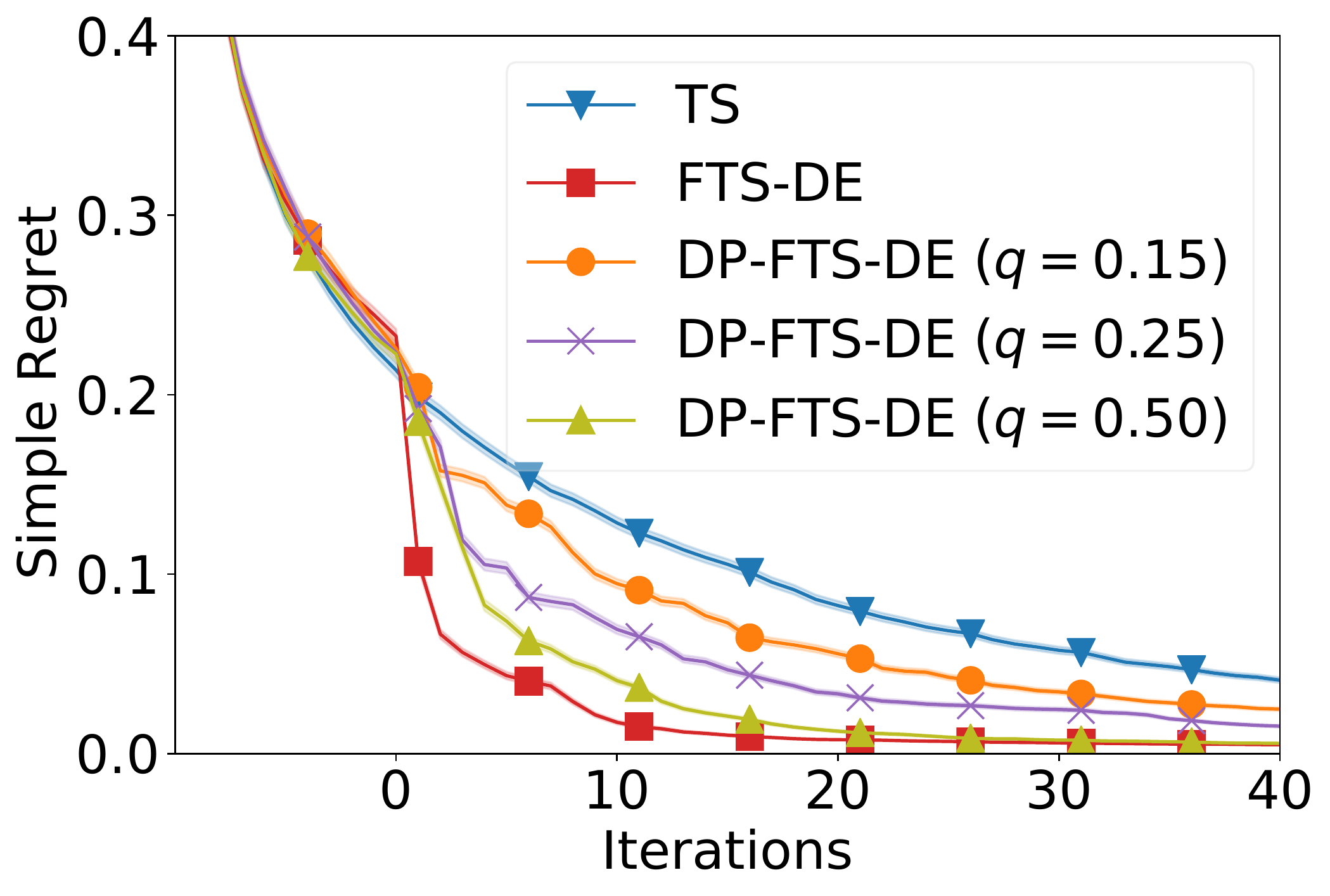} & \hspace{-6mm} 
         \includegraphics[width=0.23\linewidth]{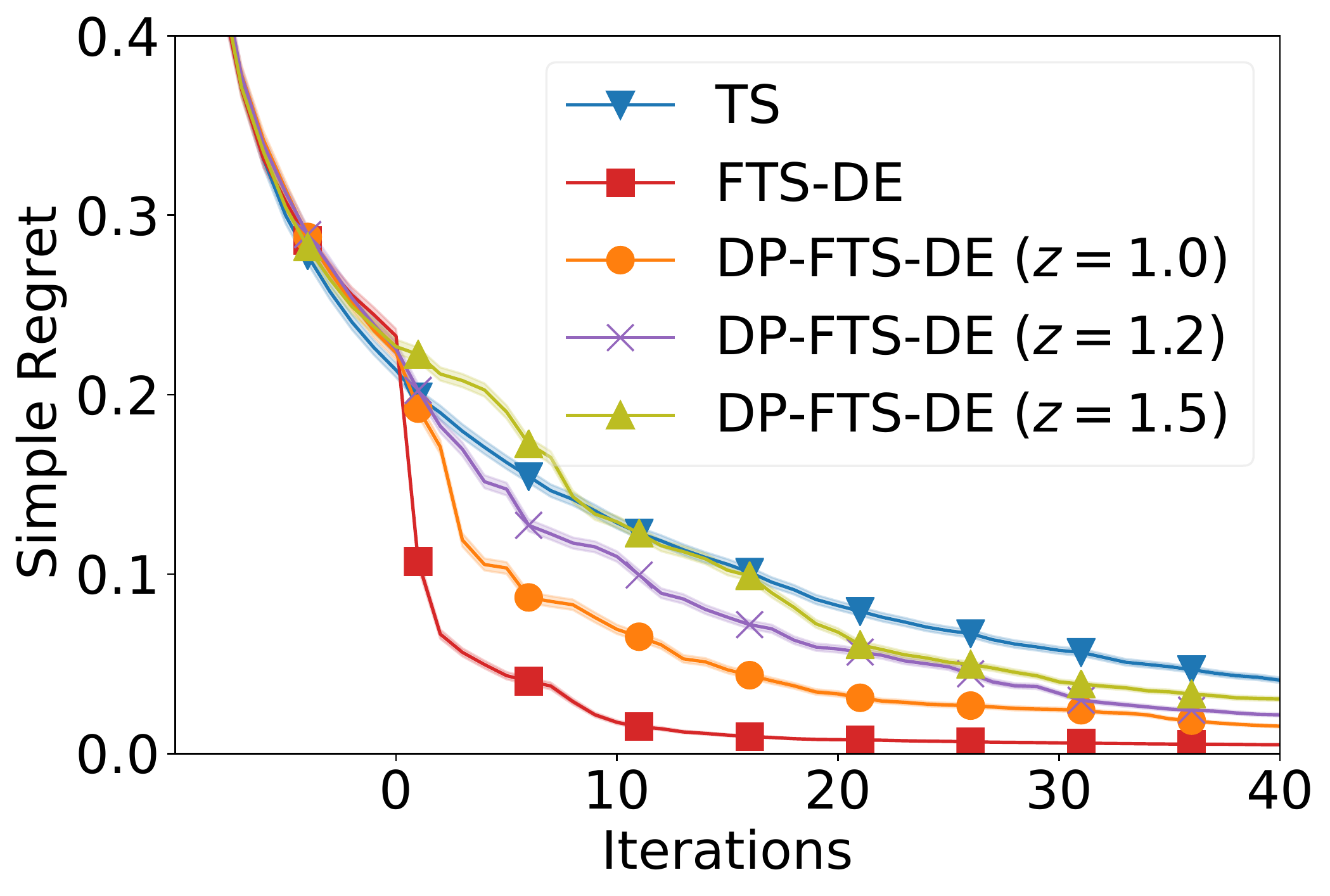}\\
         {(a)} & {(b)} & {(c)} & {(d)}
     \end{tabular}
\vspace{-2.5mm}
     \caption{
     (a) Benefit of DE.
     (b) Impact of $S$.
     (c) Impact of $q$;
     privacy losses after $T=40$ iterations are $5.93,9.91,20.12$ for the respective $q=0.15,0.25,0.5$.
     (d) Impact of $z$;
     privacy loss are $9.91,7.39,5.22$ for the respective $z=1.0,1.2,1.5$.
     Every curve is averaged over $N=200$ agents, each further averaged over $5$ runs with different random initializations of size $N_{\text{init}}=10$.
     The results before iteration $0$ correspond to the initialization period.
     }
     \label{fig:synth_1}
\vspace{-6mm}
\end{figure}

In synthetic experiments, we firstly sample a function from a GP with the SE kernel, and then apply different small random perturbations to the values of the sampled function to obtain the objective functions of $N=200$ different agents. 
We choose $M=50$ and $1-p_t=1/\sqrt{t},\forall t\in\mathbb{Z}^+$.
We firstly demonstrate the performance advantage of modified FTS and FTS-DE (without DP) over standard TS.
As shown in Fig.~\ref{fig:synth_1}a, FTS converges faster than standard TS and more importantly, FTS-DE (Sec.~\ref{subsec:de}) further improves the performance of FTS considerably.
Moreover, using a larger number $P$ of sub-regions (i.e., smaller sub-regions) brings more performance benefit. 
This is consistent with our analysis of DE (Sec.~\ref{subsec:de}) suggesting that smaller sub-regions are easier to model for the GP surrogate and hence can be better explored.
Moreover, we have also verified that after the integration of DP, DP-FTS-DE ($P=2$) can still achieve significantly better convergence (utility) than DP-FTS for the same level of privacy guarantee (Fig.~\ref{fig:synth_3} in App.~\ref{app:synth_exp}).
These results justify the practical significance of our DE technique  (Sec.~\ref{subsec:de}).
We have also shown (Fig.~\ref{fig:synth_4}, App.~\ref{app:synth_exp}) that both components in our DE technique (i.e., letting every agent explore only a local sub-region and giving more weights to those agents exploring the sub-region) are necessary for the performance of DE.

Fig.~\ref{fig:synth_1}b explores the impact of the clipping threshold $S$. From Fig.~\ref{fig:synth_1}b, an overly small $S$ may hinder the performance since it causes too many vectors to be clipped, and an excessively large $S$ may also be detrimental due to increasing the variance of the added Gaussian noise. This corroborates our analysis in Sec.~\ref{sec:theoretical_analysis}.
The value of $S=11$, which delivers the best performance in Fig.~\ref{fig:synth_1}b, has been chosen such that only a small percentage ($0.8\%$) of the vectors are clipped.
Figs.~\ref{fig:synth_1}c and d show the privacy-utility trade-off of our DP-FTS-DE algorithm induced by $q$ and $z$.
The results verify our theoretical insights regarding the impact of the parameters $q$ and $z$ on the privacy-utility trade-off (Sec.~\ref{sec:theoretical_analysis}), i.e., a larger $q$ (smaller $z$) leads to a better utility at the expense of a greater privacy loss.
Lastly, we also verify the robustness of our algorithm against agent heterogeneity. In particular, we show that when the objective functions of different agents are significantly different, our FTS-DE algorithm is still able to perform comparably with standard TS and letting the impact of the other agents decay faster (i.e., letting $1-p_t$ decrease faster) can improve the performance of FTS-DE in this scenario (see Fig.~\ref{fig:exp:hetero:agents} in App.~\ref{app:synth:more:results}).

\vspace{-1.3mm}
\subsection{Real-World Experiments}
\label{subsec:real_exp}
\vspace{-1.3mm}
\begin{figure}
     \centering
     \begin{tabular}{ccc}
         \includegraphics[width=0.28\linewidth]{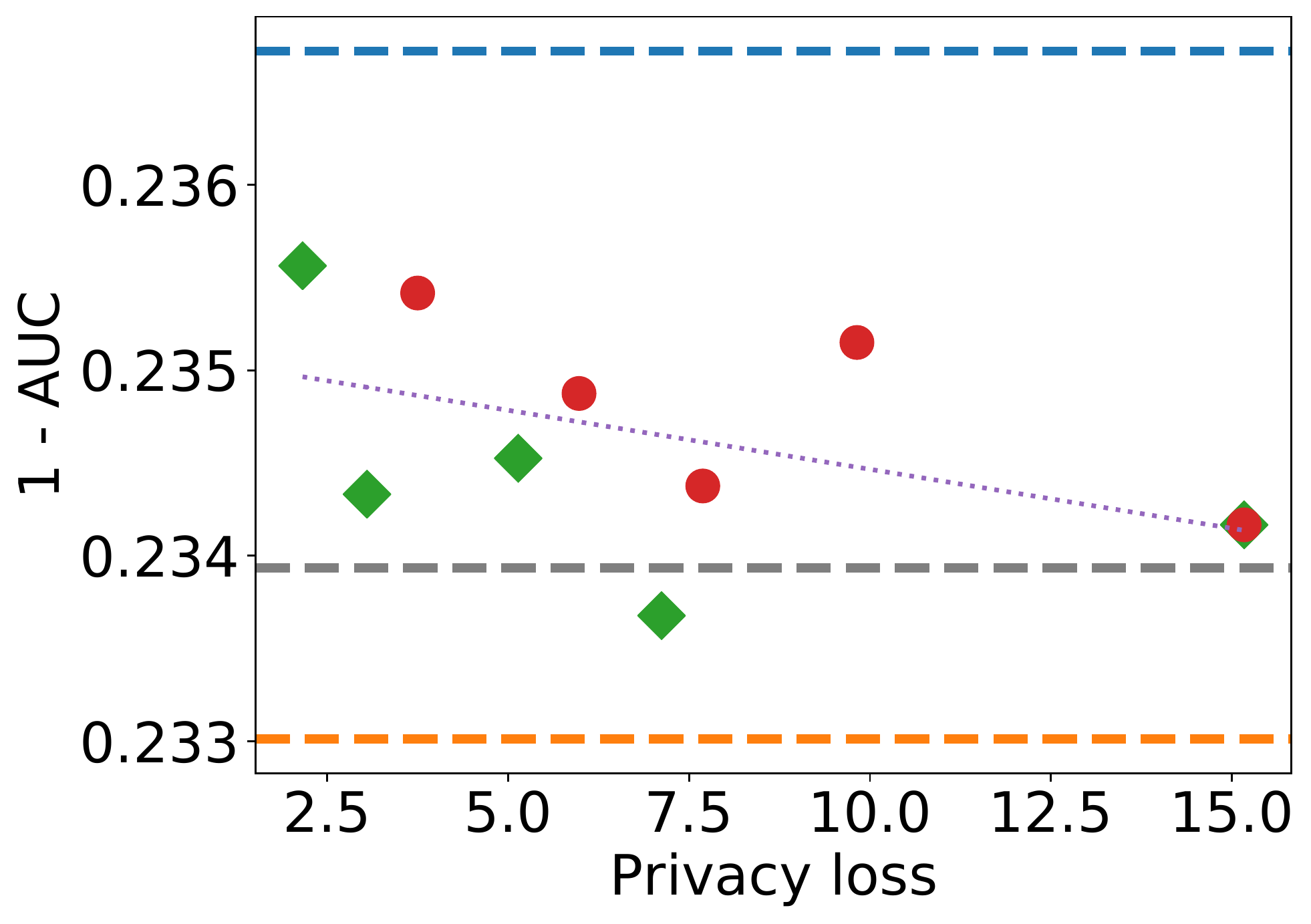} & \hspace{-4mm}
         \includegraphics[width=0.28\linewidth]{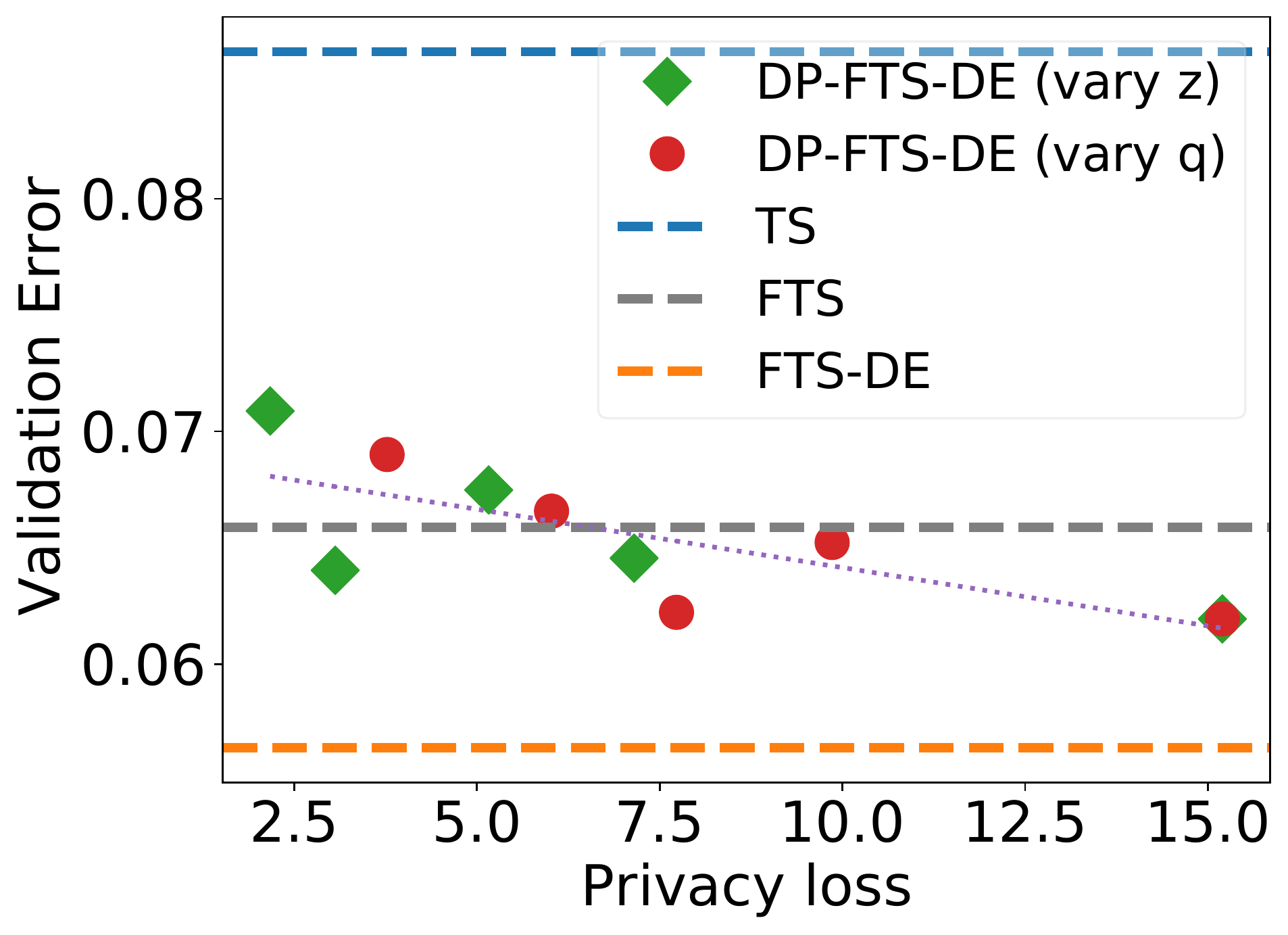} & \hspace{-4mm}
         \includegraphics[width=0.28\linewidth]{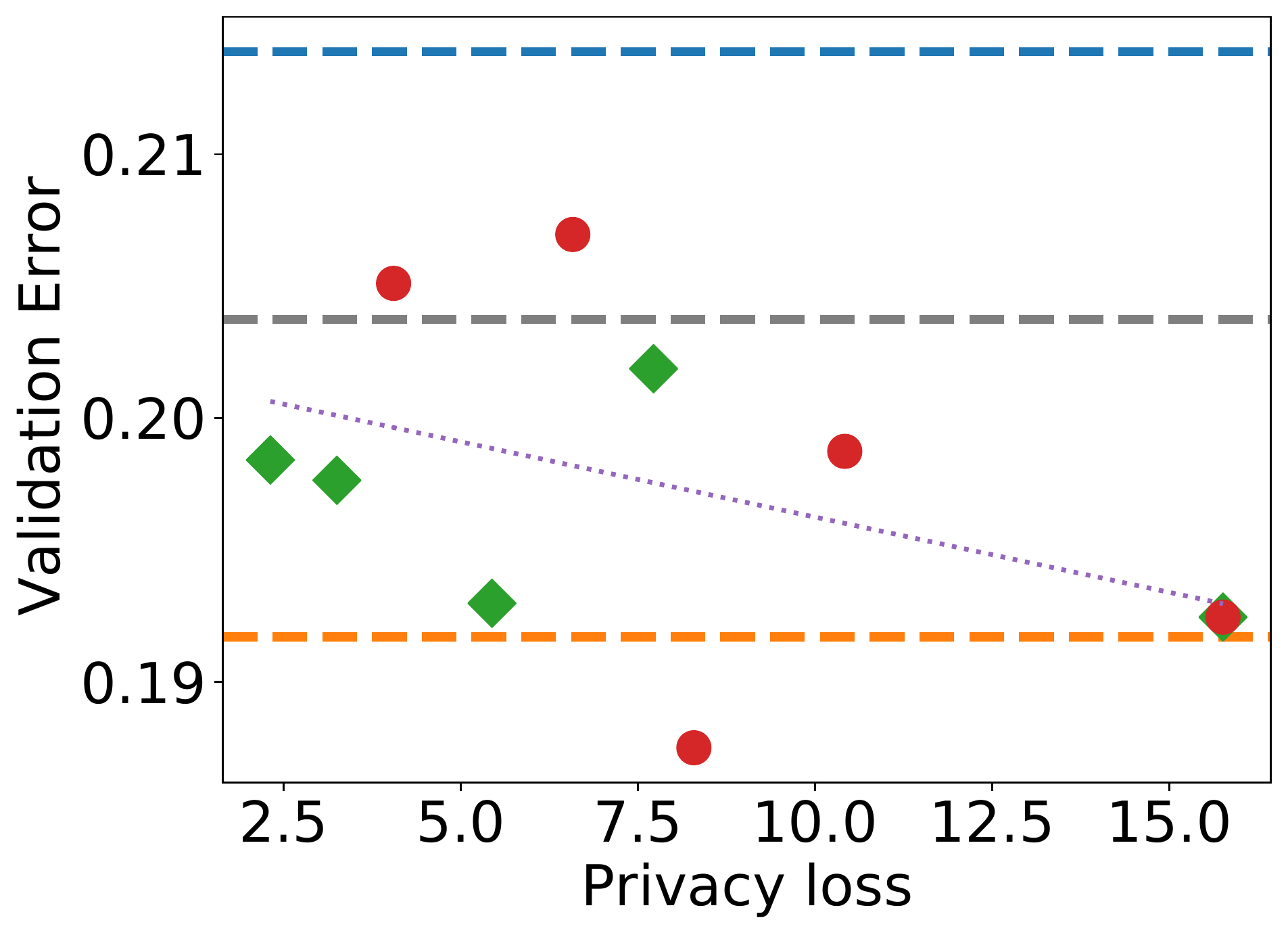} \\
         {(a)} & {(b)} & {(c)}\\
         \vspace{-1mm}
         \includegraphics[width=0.28\linewidth]{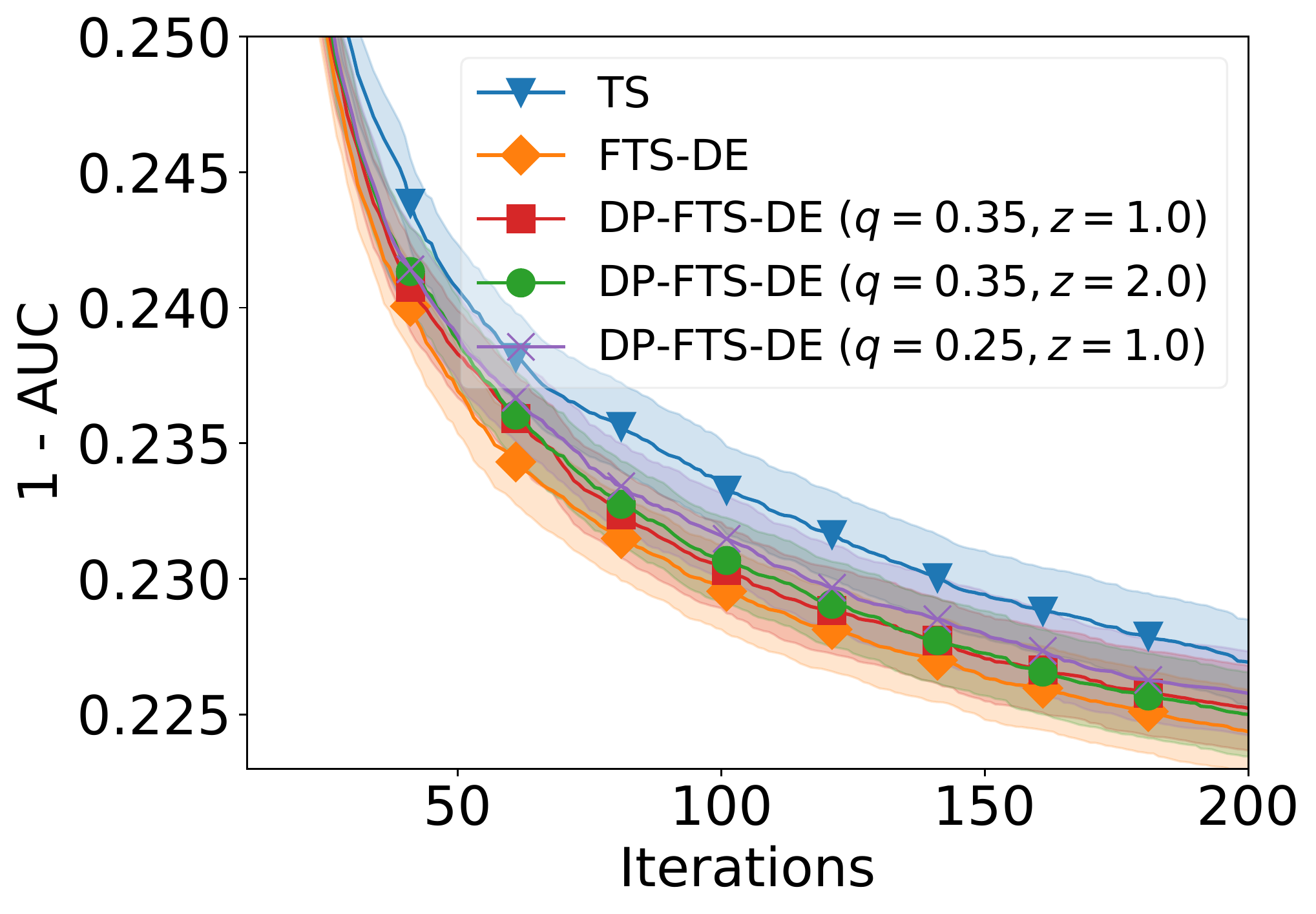} & \hspace{-4mm}
         \includegraphics[width=0.28\linewidth]{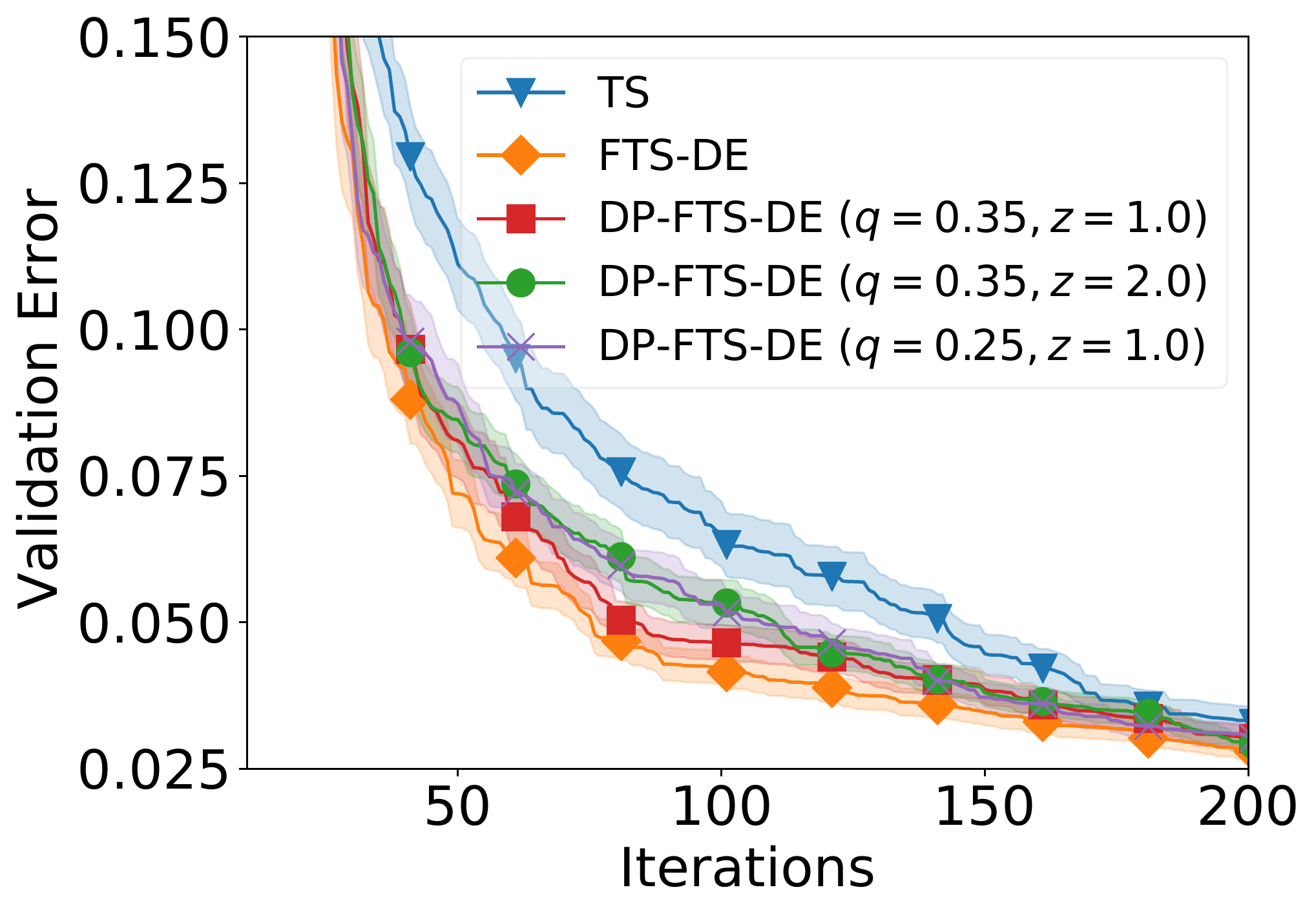} & \hspace{-4mm}
         \includegraphics[width=0.28\linewidth]{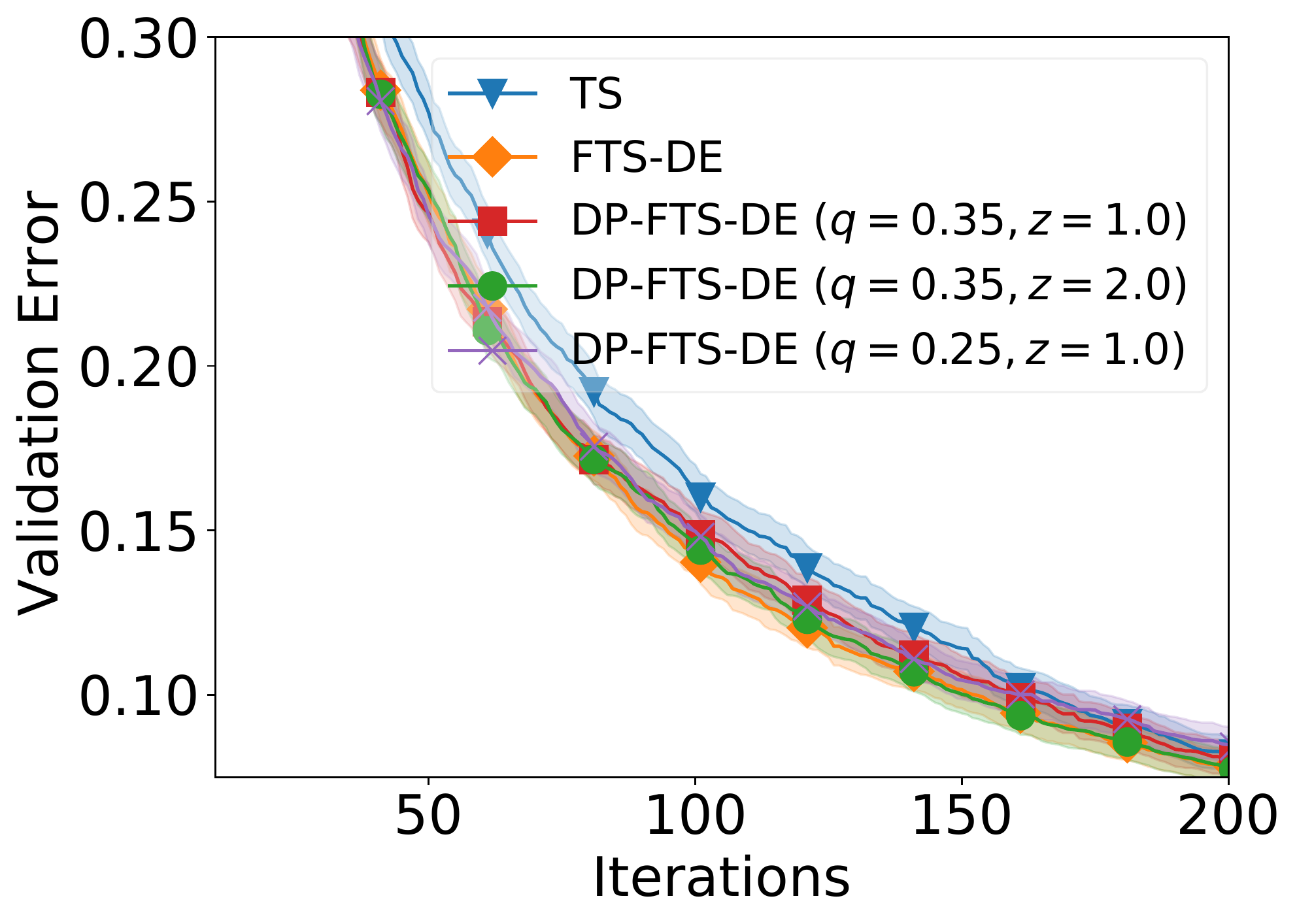}\\
         {(d)} & {(e)} & {(f)}
     \end{tabular}
\vspace{-2mm}
     \caption{(a,b,c) Privacy loss vs.~performance after $60$ iterations for landmine detection, human activity recognition, and EMNIST. 
     The more to the \emph{left} (\emph{bottom}), the better the privacy (utility). 
     (d,e,f) Convergence results for some settings in each experiment.
     Results are averaged over $N$ agents,
     each further averaged over $100$ (a,d) and $10$ (b,c,e,f) random initializations ($N_{\text{init}}=10$).
     Some error bars overlap because the agents 
     are highly heterogeneous and thus have significantly different scales.
     }
     \label{fig:real_exp}
\vspace{-5mm}
\end{figure}
We adopt $3$ commonly used datasets in FL and FBO~\cite{dai2020federated,smith2017federated}.
We firstly use a landmine detection dataset with $N=29$ landmine fields~\cite{xue2007multi} and tune $2$ hyperparameters of SVM for landmine detection.
Next, we use data collected using mobile phone sensors when $N=30$ subjects are performing $6$ activities~\cite{anguita2013public} and tune $3$ hyperparameters of logistic regression for activity classification.
Lastly, we use the images of handwritten characters by $N=50$ persons from EMNIST (a commonly used benchmark in FL)~\cite{cohen2017emnist} and tune $3$ hyperparameters of a convolutional neural network used for image classification.
In all $3$ experiments, we choose $P=4$, $S=22.0$, $M=100$, and $1-p_t=1/t$.
In the practical deployment of our algorithm, if the values of these parameters are tuned by running additional experiments, the additional privacy loss can be easily accounted for using existing techniques~\cite{abadi2016deep}.

Figs.~\ref{fig:real_exp}a,b,c plot the privacy (horizontal axis, more to the left is better) and utility (vertical axis, lower is better) after $60$ iterations\footnote{In practice, we recommend that the agents switch to local TS after the value of $1-p_t$ becomes so small (e.g., after $60$ iterations) that the probability of using the information from the central server is negligible (i.e., line $8$ of Algo.~\ref{alg:BO} is unlikely to be executed), after which no privacy loss is incurred.}.
The green dots correspond to $z=1, 1.6, 2, 3, 4$ ($q=0.35$) and the red dots represent $q=0.1, 0.15, 0.2, 0.25, 0.35$ ($z=1$).
Results show that with small privacy loss (in the single digit range), DP-FTS-DE achieves a competitive performance (utility) and significantly outperforms standard TS in all settings.
The figures also reveal a clear trade-off between privacy and utility, i.e., a smaller privacy loss (more to the left) generally results in a worse utility (larger vertical value).
In addition, these two observations can also be obtained from Figs.~\ref{fig:real_exp}d,e,f which plot some convergence results:
DP-FTS-DE and FTS-DE converge faster than TS; a smaller privacy loss (i.e., larger $z$ or smaller $q$) in general leads to a slower convergence.
Figs.~\ref{fig:real_exp}a,b,c also justify the importance of 
DE (Sec.~\ref{subsec:de}) since FTS-DE (and some settings of DP-FTS-DE) 
outperforms FTS without DE in all experiments. We also verify the importance of DE when DP is integrated 
in App.~\ref{app:real_exp} (Fig.~\ref{fig:har_compare_with_dp_fts}).
Furthermore, we demonstrate the robustness of our results against the choices of the weights and the number of sub-regions in App.~\ref{app:real:exp:different:weights}.
Lastly, our DP-FTS-DE can be easily adapted to use R\'enyi DP~\cite{wang2019subsampled}\footnote{We only need to modify step 6 of Algo.~\ref{alg:DP-FTS-DE} s.t.~we randomly select a fixed number of $Nq$ agents.}, which, although requires modifications to our theoretical analysis, offers slightly better privacy loss 
with comparable performances (Fig.~\ref{fig:rdp} in App.~\ref{app:real_exp}).

\vspace{-2mm}
\section{Related Works}
\vspace{-2mm}
\label{sec:related_works}
BO has been extensively studied recently under different settings~\cite{balakrishnan2020efficient,chowdhury2017kernelized,hoang2018decentralized,kharkovskii2020nonmyopic,nguyen2021top,nguyen2021trusted,shahriari2016taking,srinivas2009gaussian}.
Recent works have added privacy preservation to BO by applying DP to the output 
of BO~\cite{kusner2015differentially}, using a different notion of privacy other than DP~\cite{dai2018privacy}, adding DP to outsourced BO~\cite{kharkovskii2020private}, or adding local DP to BO~\cite{zhou2020local}.
However, none of these works can tackle the FBO setting considered in this paper.
Collaborative BO involving multiple agents has also been considered by the work of~\cite{sim2021collaborative}, however,~\cite{sim2021collaborative} has assumed that all agents share the same objective function and focused on the issue of fairness.
Moreover, BO in the multi-agent setting has also been studied from the perspective of game theory~\cite{dai2020r2,sessa2019no}.
Our method also
shares similarity with 
parallel BO~\cite{contal2013parallel,daxberger2017distributed,desautels2014parallelizing,garcia2019fully,hernandez2017parallel,kandasamy2018parallelised}. However, parallel BO optimizes a single objective function while we allow agents to have different objective functions.
Our algorithm is also related to multi-fidelity BO~\cite{dai2019bayesian,kandasamy2016gaussian,zhang2020bayesian,zhang2017information} because utilizing information received from the central server can be viewed as querying a low-fidelity function.
Our DE technique  
bears similarity to~\cite{eriksson2019scalable} which has also used separate GP surrogates to model different local sub-regions (hyper-rectangles) and shown significantly improved performance.

FL has attracted significant attention in recent years~\cite{kairouz2019advances,li2021federated,li2021model,li2019federated2,li2020practical,li2019federated,mcmahan2016communication}.
In particular, privacy preservation using DP has been an important topic for FL~\cite{hu2020personalized,li2020privacy,truex2019hybrid,wei2020federated}, including both central DP (with trusted central server)~\cite{mcmahan2018learning} and local~\cite{kasiviswanathan2011can,warner1965randomized,zhao2020local} or distributed DP~\cite{bittau2017prochlo,cheu2019distributed,dwork2006our,shi2011privacy} (without trusted central server).
In addition to our setting of FBO which can be equivalently called \emph{federated GP bandit}, other sequential decision-making problems have also been extended to the federated setting, including federated multi-armed bandit~\cite{shi2021federated,zhu2021federated}, federated linear bandit~\cite{dubey2020differentially}, and federated reinforcement learning~\cite{fan2021fault}.
Lastly, federated hyperparameter tuning (i.e., hyperparameter tuning of ML models in the federated setting) has been attracting growing attention recently~\cite{holly2021evaluation,khodak2021federated}.

\vspace{-2mm}
\section{Conclusion and Future Works}
\label{sec:conclusion}
\vspace{-2mm}
We introduce DP-FTS-DE, which equips FBO with rigorous privacy guarantee and is amenable to privacy-utility trade-off both theoretically and empirically.
Since our method 
aims to
promote larger-scale adoption of FBO, it may bring the negative societal impact of fairness: some users may become discriminated against by the algorithm. This can be mitigated by extending our method to consider fairness~\cite{sim2021collaborative}, suggesting an interesting future work.
A limitation of our work is that the privacy loss offered by the moments accountant method~\cite{abadi2016deep} is not state-of-the-art. 
For example, the more advanced privacy-preserving technique of Gaussian DP~\cite{bu2020deep,dong2019gaussian} can deliver smaller privacy losses than moments accountant and has been widely applied in practice.
So, a potential future work is to extend our method to adopt more advanced privacy-preserving techniques such as Gaussian DP.
Another limitation of our paper is that we have not accounted for the different network capabilities of different agents, which is a common problem in the federated setting. That is, our algorithm requires every agent to send its updated parameter vector to the central server in every iteration (lines 3 and 4 of Algo.~\ref{alg:DP-FTS-DE}), which may not be realistic in some scenarios since the messages from some agents may not be received by the central server. Therefore, accounting for this issue using a principled method represents an interesting future work.
Other potential future works include adapting our method for other sequential decision-making problems such as reinforcement learning~\cite{fan2021fault}, and incorporating risk aversion~\cite{nguyen2021optimizing,nguyen2021value} into our method for safety-critical domains such as clinical applications (Sec.~\ref{sec:introduction}).

\begin{ack}
This research/project is supported by the National Research Foundation, Singapore under its Strategic Capability Research Centres Funding Initiative. Any opinions, findings and conclusions or recommendations expressed in this material are those of the author(s) and do not reflect the views of National Research Foundation, Singapore.
\end{ack}

\bibliography{pfbo}
\bibliographystyle{abbrv}

\newpage
\appendix
\onecolumn


\section{Proof of Theoretical Results}
\label{app:proof:theoretical:results}
\subsection{Proof of Proposition~\ref{proposition:dp}}
\label{app:proposition_1}
%
%

Proposition~\ref{proposition:dp} follows directly from the DP guarantee of the works of~\cite{abadi2016deep} and~\cite{mcmahan2018learning} (e.g., Theorem $1$ of~\cite{abadi2016deep}). Therefore, to prove the validity of Proposition~\ref{proposition:dp}, we only need to show that the joint subsampled Gaussian mechanism we apply in every iteration (Sec.~\ref{subsec:dp_fts_de}) is the same as the one adopted by~\cite{abadi2016deep} and~\cite{mcmahan2018learning}.
Therefore, we demonstrate here that the interpretation of our privacy-preserving transformations as a single subsampled Gaussian mechanism, which we have described in Sec.~\ref{subsec:dp_fts_de}, is equivalent to the transformations adopted by the work of~\cite{mcmahan2018learning}.

Firstly, our subsampling step (step 6 of Algo.~\ref{alg:DP-FTS-DE}) is the same as the one adopted by~\cite{mcmahan2018learning} since we both use the same subsampling technique, i.e., select every agent with a fixed probability $q$. 
Secondly, we both clip the (joint) vector from every selected agent (step 9 of Algo.~\ref{alg:DP-FTS-DE}) to ensure that its $L_2$ norm is bounded: $\norm{\widehat{\boldsymbol{\omega}}^{\text{joint}}_{n,t}}_2 \leq N\varphi_{\max}S,\forall n\in\mathcal{S}_t$.
Thirdly, we have adopted one of the two weighted-average estimators proposed by~\cite{mcmahan2018learning}, i.e., the unbiased estimator.
Specifically, we set the weight (we follow~\cite{mcmahan2018learning} and denote the weight of agent $\mathcal{A}_n$ by $\omega_n$ here) of every agent to be $\omega_n=1,\forall n\in[N]$. As a result, the unbiased estimator leads to: $\boldsymbol{\omega}^{\text{joint}}_{t}=\frac{1}{q\sum^N_{n=1}\omega_n}\sum_{n\in\mathcal{S}_t} \omega_n \widehat{\boldsymbol{\omega}}^{\text{joint}}_{n,t}=\frac{1}{qN}\sum_{n\in\mathcal{S}_t} \widehat{\boldsymbol{\omega}}^{\text{joint}}_{n,t}$.
Lastly, we have calculated the sensitivity (which determines the variance of the Gaussian noise) in the same way as~\cite{mcmahan2018learning}, i.e., using Lemma 1 of~\cite{mcmahan2018learning}. In particular, note that our clipping step has ensured that $\norm{\omega_n\widehat{\boldsymbol{\omega}}^{\text{joint}}_{n,t}}_2 \leq N\varphi_{\max}S,\forall n\in\mathcal{S}_t$; according to Lemma 1 of~\cite{mcmahan2018learning}, we have that the sensitivity can be upper-bounded by: $\mathbb{S}\leq\frac{N\varphi_{\max}S}{q\sum^N_{n=1}\omega_n}=\varphi_{\max}S/q$, which leads to the standard deviation of the Gaussian noise we have added (step 11 of Algo.~\ref{alg:DP-FTS-DE}): $z\mathbb{S}=z\varphi_{\max}S/q$.

To conclude, the single joint subsampled Gaussian mechanism performed by our DP-FTS-DE algorithm in every iteration is the same as the one adopted by~\cite{mcmahan2018learning}. Therefore, the DP guarantee of~\cite{mcmahan2018learning} and~\cite{abadi2016deep} is also valid for our DP-FTS-DE algorithm, hence justifying the validity of our Proposition~\ref{proposition:dp}.

\subsection{Proof of Theorem~\ref{theorem:dp_fts_de}}
\label{app:proof_regret_bound}
In this section, we prove Theorem~\ref{theorem:dp_fts_de}, which gives an upper bound on the cumulative regret of agent $\mathcal{A}_1$ running our DP-FTS-DE algorithm.
The proof of Theorem~\ref{theorem:dp_fts_de} makes use of the proof of~\cite{dai2020federated}, 
and the main technical challenge is how to take into account the impacts of (a) our modification to the original FTS algorithm by incorporating a central server and an aggregation through weighted averaging (first paragraph of Sec.~\ref{subsec:dp_fts}), (b) the privacy-preserving transformations (lines $5$-$11$ of Algo.~\ref{alg:DP-FTS-DE}), and (c) distributed exploration (DE) (Sec.~\ref{subsec:de}).
Since we are mainly interested in the asymptotic regret upper bound, we ignore the impact of the initialization period. Considering initialization would only add a constant term $2BN_{\text{init}}$ to the upper bound on the cumulative regret in Theorem~\ref{theorem:dp_fts_de} ($N_{\text{init}}$ is the number of initial inputs selected during initialization), and hence would not affect the asymptotic no-regret property of our algorithm.

Note that as we have mentioned in Sections~\ref{sec:background} and~\ref{sec:theoretical_analysis}, we prove here an upper bound on the cumulative regret of agent $\mathcal{A}_1$, i.e., $R^1_T\triangleq \sum^T_{t=1}(f^1(\mathbf{x}^{1,*})-f^1(\mathbf{x}^1_t))$.
To simplify notations, we drop the superscript $1$ in the subsequent analysis, i.e., we use $f$ to denote $f^1$, $f_t$ to denote $f^1_t$, $\mathbf{x}_t$ to denote $\mathbf{x}^1_t$, $\mathbf{x}^*$ to denote $\mathbf{x}^{1,*}$, etc.
Similarly, we use $\mu_{t-1}$ and $\sigma_{t-1}$ to represent the GP posterior mean and standard deviation of $\mathcal{A}_1$ at iteration $t$.

\subsubsection{Definitions and Supporting Lemmas}
We firstly define some notations we use to represent the privacy-preserving transformations, which are consistent with those in the main text.
In iteration $t$, we use $\boldsymbol{\omega}_{n,t}$ to denote the vector the central server receives from agent $\mathcal{A}_n$.
For a given set of agents $\mathcal{C} \in \{1,\ldots,N\}$, define $\widetilde{\varphi}^{(i)}_{\mathcal{C}}\triangleq\sum_{n\in \mathcal{C}}\varphi^{(i)}_n$, i.e., the total weight of those agents in the set $\mathcal{C}$ for the sub-region $\mathcal{X}_i$.
Next, we define $N$ indicator (Bernoulli) random variables $\mathbb{I}_n,\forall n\in[N]$, where $\mathbb{P}(\mathbb{I}_n = 1)=q,\forall n\in[N]$. These indicator random variables will be used to account for the subsampling step (i.e., step 6 of Algo.~\ref{alg:DP-FTS-DE}).
Denote by $\widehat{\boldsymbol{\omega}}_{n,t}$ the resulting vector after $\boldsymbol{\omega}_{n,t}$ is clipped to have a maximum $L_2$ norm of $S/\sqrt{P}$ (i.e., step 9 of Algo.~\ref{alg:DP-FTS-DE}): 
\[
\widehat{\boldsymbol{\omega}}_{n,t} \triangleq \frac{\boldsymbol{\omega}_{n,t}}{\max(1,\frac{\norm{\boldsymbol{\omega}_{n,t}}_2}{S/\sqrt{P}})}.
\]
An immediate consequence is that $\forall \mathbf{x}\in\mathcal{X}$:
\begin{equation}
\begin{split}
|\boldsymbol{\phi}(\mathbf{x})^{\top} \widehat{\boldsymbol{\omega}}_{n,t}|=|\boldsymbol{\phi}(\mathbf{x})^{\top} \frac{\boldsymbol{\omega}_{n,t}}{\max(1,\frac{\norm{\boldsymbol{\omega}_{n,t}}_2}{S/\sqrt{P}})}| = |\boldsymbol{\phi}(\mathbf{x})^{\top} \boldsymbol{\omega}_{n,t}|\frac{1}{\max(1,\frac{\norm{\boldsymbol{\omega}_{n,t}}_2}{S/\sqrt{P}})} \leq |\boldsymbol{\phi}(\mathbf{x})^{\top} \boldsymbol{\omega}_{n,t}|.
\end{split}
\label{eq:clip_make_norm_smaller}
\end{equation}
Denote by $\boldsymbol{\eta}$ the added Gaussian noise vector (i.e., step 11 of Algo.~\ref{alg:DP-FTS-DE}): $\boldsymbol{\eta} \sim \mathcal{N}\left(\mathbf{0}, (z\varphi_{\max} S/q)^2\mathbf{I}\right)$. Next, define 
\begin{equation}
\boldsymbol{\omega}^{(i)}_t \triangleq \frac{\sum^N_{n=1}\mathbb{I}_n\varphi^{(i)}_n\widehat{\boldsymbol{\omega}}_{n,t}}{q}+\boldsymbol{\eta}.
\label{eq:define:omega:i:t:camera}
\end{equation}
As a result, for agent $\mathcal{A}_1$ at iteration $t>1$, with probability $1-p_t$, the query $\mathbf{x}^1_t$ is selected using the $\boldsymbol{\omega}^{(i)}_{t}$'s: $\mathbf{x}^1_t=\arg\max_{\mathbf{x}\in\mathcal{X}}\boldsymbol{\phi}(\mathbf{x})^{\top}\boldsymbol{\omega}^{(i^{[\mathbf{x}]})}_{t}$, where $i^{[\mathbf{x}]}$ represents the sub-region $\mathbf{x}$ is assigned to.
This corresponds to line 8 of Algo.~\ref{alg:BO}. Note that to simplify the notations in the subsequent analyses, we have slightly deviated from the indexing from Algo.~\ref{alg:BO} by using $\boldsymbol{\omega}^{(i^{[\mathbf{x}]})}_{t}$ instead of $\boldsymbol{\omega}^{(i^{[\mathbf{x}]})}_{t-1}$. To be consistent with Algo.~\ref{alg:BO}, we can simply replace all appearances of $\boldsymbol{\omega}^{(i^{[\mathbf{x}]})}_{t}$ by $\boldsymbol{\omega}^{(i^{[\mathbf{x}]})}_{t-1}$ in our proof.

Let $\delta\in(0,1)$, recall that we have defined in Theorem~\ref{theorem:dp_fts_de} that $\beta_{t} \triangleq B+\sigma\sqrt{2(\gamma_{t-1} + 1 + \log(4/\delta)}$
and define $c_t \triangleq \beta_t (1 + \sqrt{2\log(|\mathcal{X}|t^2)})$ for all $t\in\mathbb{Z}^+$.
Denote by $A_t$ the event that agent $\mathcal{A}_1$ chooses $\mathbf{x}^1_t$ by maximizing a sampled function from its own GP posterior belief (i.e., $\mathbf{x}^1_t=\arg\max_{\mathbf{x}\in\mathcal{X}}f^1_t(\mathbf{x})$, as in line $6$ of Algo.~\ref{alg:BO}), 
which happens with probability $p_t$; 
denote by $B_t$ the event that $\mathcal{A}_1$ chooses $\mathbf{x}^1_t$ using the information received from the central server: $\mathbf{x}^1_t=\arg\max_{\mathbf{x}\in\mathcal{X}}\boldsymbol{\phi}(\mathbf{x})^{\top}\boldsymbol{\omega}^{(i^{[\mathbf{x}]})}_{t}$ (line $8$ of Algo.~\ref{alg:BO}),
which happens with probability $(1-p_t)$.

Next, we denote as $\mathcal{F}_{t}$ the filtration which includes the history of selected inputs and observed outputs of agent $\mathcal{A}_1$ until (including) iteration $t$.
Now we define two events that are $\mathcal{F}_{t-1}$-measurable.
\begin{lemma}[Lemma 1 of~\cite{dai2020federated}]
\label{lemma:uniform_bound}
Let $\delta \in (0, 1)$. Define $E^f(t)$ as the event that $|\mu_{t-1}(\mathbf{x}) - f(\mathbf{x})| \leq \beta_t \sigma_{t-1}(\mathbf{x})$ for all $\mathbf{x}\in \mathcal{X}$.
We have that $\mathbb{P}\left[E^f(t)\right] \geq 1 - \delta / 4$ for all $t\geq 1$.
\end{lemma}

\begin{lemma}[Lemma 2 of~\cite{dai2020federated}]
\label{lemma:uniform_bound_t}
Define $E^{f_t}(t)$ as the event that $|f_t(\mathbf{x}) - \mu_{t-1}(\mathbf{x})| \leq \beta_t \sqrt{2\log(|\mathcal{X}|t^2)} \sigma_{t-1}(\mathbf{x})$.
We have that $\mathbb{P}\left[E^{f_t}(t) | \mathcal{F}_{t-1}\right] \geq 1 - 1 / t^2$ for any possible filtration $\mathcal{F}_{t-1}$.
\end{lemma}

Note that conditioned on both events $E^{f}(t)$ and $E^{f_t}(t)$, we have that for all $x\in \mathcal{X}$ and all $t\geq 1$:
\begin{equation}
\begin{split}
|f(\mathbf{x}) - f_t(\mathbf{x})| &\leq |f(\mathbf{x}) - \mu_{t-1}(\mathbf{x})| + |\mu_{t-1}(\mathbf{x}) - f_t(\mathbf{x})|\\
&= \beta_t\sigma_{t-1}(\mathbf{x}) + \beta_t \sqrt{2\log(|\mathcal{X}|t^2)} \sigma_{t-1}(\mathbf{x})=c_t\sigma_{t-1}(\mathbf{x}).
\end{split}
\label{eq:combine_two_events}
\end{equation}

Next, at every iteration $t$, we define a set of \emph{saturated points}, i.e., the set of ``bad'' inputs at iteration $t$.
Intuitively, these inputs are considered as ``bad'' because their corresponding function values have relatively large differences from the value of the global maximum of $f$.
\begin{definition}
\label{def:saturated_set}
At iteration $t$, define the set of saturated points as
\[
S_t = \{ \mathbf{x} \in \mathcal{X} : \Delta(\mathbf{x}) > c_t \sigma_{t-1}(\mathbf{x}) \}
\]
in which $\Delta(\mathbf{x}) \triangleq f(\mathbf{x}^*) - f(\mathbf{x})$ and $\mathbf{x}^* \in \arg\max_{\mathbf{x}\in \mathcal{X}}f(\mathbf{x})$.
\end{definition}
Note that $\Delta(\mathbf{x}^*) \triangleq f(\mathbf{x}^*) - f(\mathbf{x}^*) = 0 < c_t \sigma_{t-1}(\mathbf{x}^*)$ for all $t\geq 1$.
Therefore, $\mathbf{x}^*$ is always unsaturated for all $t\geq 1$. $S_t$ is $\mathcal{F}_{t-1}$-measurable.

Consistent with the main text, we define $\widetilde{f}^n_{t}(\mathbf{x}) \triangleq \boldsymbol{\phi}(\mathbf{x})^{\top}\boldsymbol{\omega}_{n,t}, \forall \mathbf{x} \in \mathcal{X}$, i.e., $\widetilde{f}^n_{t}$ is the sampled function from agent $\mathcal{A}_n$'s GP posterior with RFFs approximation at iteration $t$.

\begin{lemma}[Lemma 4 of~\cite{dai2020federated}]
\label{lemma:bound_gt_gn}
Given any $\delta \in (0, 1)$. 
We have that for all agents $\mathcal{A}_n,\forall n=1,\ldots,N$, all $\mathbf{x}\in \mathcal{X}$ and all $t\geq 1$, with probability of at least $1 - \delta/2$,
\[
|\widetilde{f}^n_{t}(\mathbf{x}) - f^{n}(\mathbf{x})|\leq \tilde{\Delta}_{n,t},
\]
where $\beta'_{t} = B+\sigma\sqrt{2(\gamma_{t-1} + 1 + \log(8N/\delta)}$, and
\[
\tilde{\Delta}_{n,t} \triangleq \varepsilon\frac{(t+1)^2}{\lambda}\left(B +  \sqrt{2\log\left(\frac{4\pi^2t^2N}{3\delta}\right)}  \right) + \beta'_{t+1} + \sqrt{2\log\frac{2\pi^2t^2N}{3\delta} + M}.
\]
\end{lemma}
Note that a difference between our Lemma~\ref{lemma:bound_gt_gn} above and Lemma 4 of~\cite{dai2020federated} is that in their proof, they assumed that the number of observations from agent $\mathcal{A}_n$ is a constant $t_n$; in contrast, we have made use of the fact that in the setting of our DP-FTS-DE algorithm, the number of observations from the other agents are growing with $t$ because all agents are running DP-FTS-DE concurrently.
Furthermore, we define 
\begin{equation}
\tilde{\Delta}^{(i)}_{t}\triangleq \sum^N_{n=1}\varphi^{(i)}_n \tilde{\Delta}_{n,t}.
\label{eq:define_delta_tilde_it}
\end{equation}

The next lemma gives a uniform upper bound on the difference between the sampled function $f_t$ from agent $\mathcal{A}_1$ and a weighted combination of the sampled function from all agents, 
which holds throughout all sub-regions $\mathcal{X}_i,\forall i\in[P]$.
\begin{lemma}
\label{lemma:bound_gt_ft}
Denote by $\varepsilon$ an upper bound on the approximation error of RFFs approximation (Sec.~\ref{sec:background}): $\sup_{\mathbf{x},\mathbf{x}'\in \mathcal{X}}|k(\mathbf{x}, \mathbf{x}') - \boldsymbol{\phi}(\mathbf{x})^{\top}\boldsymbol{\phi}(\mathbf{x}')| \leq \varepsilon$.
At iteration $t$, conditioned on the events $E^{f}(t)$ and $E^{f_t}(t)$, we have that for 
all $\mathbf{x}\in \mathcal{X}$ and all $i\in[P]$, with probability $\geq 1 - \delta/2$,
\[
|\sum^N_{n=1}\varphi^{(i)}_n \widetilde{f}^n_t(\mathbf{x}) - f_t(\mathbf{x})|\leq \Delta^{(i)}_{t},
\]
in which $\Delta^{(i)}_{t} \triangleq \sum^N_{n=1} \varphi^{(i)}_n \Delta_{n,t}$, and
\begin{equation}
\begin{split}
\Delta_{n, t} &\triangleq \tilde{\Delta}_{n,t} + d_n + c_t\\
&=\varepsilon\frac{(t+1)^2}{\lambda}\left(B +  \sqrt{2\log\left(\frac{4\pi^2t^2N}{3\delta}\right)}  \right) + \beta'_{t+1} + \sqrt{2\log\frac{2\pi^2t^2N}{3\delta} + M} + d_n + c_t\\
&=\tilde{\mathcal{O}}(\varepsilon B t^2 + B + \sqrt{M} + d_n + \sqrt{\gamma_t}).
\end{split}
\end{equation}
\end{lemma}
\begin{proof}

\begin{equation}
\begin{split}
|\sum^N_{n=1}\varphi^{(i)}_n \widetilde{f}^n_t(\mathbf{x}) - f_t(\mathbf{x})|=|\sum^N_{n=1}\varphi^{(i)}_n \widetilde{f}^n_t(\mathbf{x}) - \sum^N_{n=1}\varphi^{(i)}_n f_t(\mathbf{x})| &\leq \sum^N_{n=1}\varphi^{(i)}_n|\widetilde{f}^n_t(\mathbf{x})-f_t(\mathbf{x})|\\
&\leq \sum^N_{n=1} \varphi^{(i)}_n\Delta_{n,t},
\end{split}
\end{equation}
where the last inequality results from Lemma 5 of~\cite{dai2020federated}.
\end{proof}
Note that Lemma~\ref{lemma:bound_gt_ft} above takes into account our modifications to the original FTS algorithm~\cite{dai2020federated} by including a central server and using an aggregation (i.e., weighted average) of the vectors from all agents (first paragraph of Sec.~\ref{subsec:dp_fts}).
Next, define $\Delta_t\triangleq\sum^N_{n=1}\Delta_{n,t}$. Note that $\tilde{\Delta}^{(i)}_t\leq \Delta^{(i)}_t\leq \Delta_t,\forall i\in[P]$, and that 
\begin{equation}
\sum^N_{n=1}\tilde{\Delta}_{n,t} \leq \sum^N_{n=1}\Delta_{n,t}=\Delta_t,
\label{eq:upper_bound_Delta}
\end{equation}
which will be useful in subsequent proofs.

\subsubsection{Main Proof}
The following lemma lower-bounds the probability that the selected input $\mathbf{x}_t$ is unsaturated.
\begin{lemma}[Lemma 7 of~\cite{dai2020federated}]
\label{lemma:prob_unsaturated}
For any filtration $\mathcal{F}_{t-1}$, conditioned on the event $E^f(t)$, we have that with probability $\geq 1 - \delta/2$,
\[
\mathbb{P}\left(\mathbf{x}_t \in \mathcal{X}\setminus S_t | \mathcal{F}_{t-1} \right) \geq P_t,
\]
in which $P_t \triangleq p_t (p - 1/t^2)$ and $p=\frac{1}{4e\sqrt{\pi}}$.
\end{lemma}
\begin{proof}
Firstly, we have that
\begin{equation}
\begin{split}
\mathbb{P}&\left(\mathbf{x}_t \in \mathcal{X}\setminus S_t | \mathcal{F}_{t-1} \right) \geq \mathbb{P}\left(\mathbf{x}_t \in \mathcal{X}\setminus S_t | \mathcal{F}_{t-1},A_t \right)\mathbb{P}(A_t)=\mathbb{P}\left(\mathbf{x}_t \in \mathcal{X}\setminus S_t | \mathcal{F}_{t-1},A_t \right)p_t.
\end{split}
\label{eq:unsaturated_prob_over_all_eq}
\end{equation}
Next, we can lower-bound
the probability $\mathbb{P}\left(\mathbf{x}_t \in \mathcal{X}\setminus S_t | \mathcal{F}_{t-1},A_t \right)$
following Lemma 7 of~\cite{dai2020federated}, which 
leads to $\mathbb{P}\left(\mathbf{x}_t \in \mathcal{X}\setminus S_t | \mathcal{F}_{t-1},A_t \right) \geq (p - 1/t^2)$
and completes the proof.
\end{proof}

Next, we derive an upper bound on the expected instantaneous regret of our DP-FTS-DE algorithm.
\begin{lemma}
\label{lemma:upper_bound_expected_regret}
For any filtration $\mathcal{F}_{t-1}$, conditioned on the event $E^{f}(t)$, we have that with probability of $\geq 1 - 5\delta/8$
\[
\mathbb{E}[r_t |\mathcal{F}_{t-1}] \leq c_t\left(1 + \frac{10}{p p_1}\right)\mathbb{E}\left[\sigma_{t-1}(x_t)| \mathcal{F}_{t-1} \right] + 4B\mathbb{E}\left[\vartheta_t |\mathcal{F}_{t-1}\right] + \psi_t + \frac{2B}{t^2},
\]
in which $r_t$ is the instantaneous regret: $r_t \triangleq f(\mathbf{x}^*) - f(\mathbf{x}_t)$, $\vartheta_t \triangleq (1-p_t)\sum^P_{i=1} \widetilde{\varphi}^{(i)}_{\mathcal{C}_t}$, and
\[
\psi_t\triangleq (1-p_t)P\Bigg[\left(\frac{\varphi_{\max} + 2}{q}+6\right) \Delta_{t} +B\left(\frac{2}{q} +\frac{N\varphi_{\max}}{q} \right) + \frac{2zS\varphi_{\max}}{q}\sqrt{2M\log\frac{8M}{\delta}}\Bigg].
\]
\end{lemma}
\begin{proof}
Firstly, we define $\overline{\mathbf{x}}_t$ as the unsaturated input at iteration $t$ with the smallest posterior standard deviation according to agent $\mathcal{A}_1$'s own GP posterior:
\begin{equation}
\overline{\mathbf{x}}_t \triangleq {\arg\min}_{\mathbf{x}\in\mathcal{X}\setminus S_t}\sigma_{t-1}(\mathbf{x}).
\end{equation}
Following this definition, for any $\mathcal{F}_{t-1}$ such that $E^{f}(t)$ is true, we have that 
\begin{equation}
\begin{split}
\mathbb{E}\left[\sigma_{t-1}(\mathbf{x}_t) | \mathcal{F}_{t-1}\right] &\geq \mathbb{E}\left[\sigma_{t-1}(\mathbf{x}_t) | \mathcal{F}_{t-1}, \mathbf{x}_t \in \mathcal{X} \setminus S_t\right]\mathbb{P}\left(\mathbf{x}_t \in \mathcal{X} \setminus S_t | \mathcal{F}_{t-1}\right)\\
&\geq \sigma_{t-1}(\overline{\mathbf{x}}_t)P_t,
\end{split}
\label{eq:bound_x_t_bar_Pt}
\end{equation}
in which the last inequality follows from the definition of $\overline{\mathbf{x}}_t$ and Lemma~\ref{lemma:prob_unsaturated}.

Next, conditioned on both events $E^{f}(t)$ and $E^{f_t}(t)$, we have that
\begin{equation}
\begin{split}
r_t=\Delta(\mathbf{x}_t)&=f(\mathbf{x}^*) - f(\overline{\mathbf{x}}_t) + f(\overline{\mathbf{x}}_t) - f(\mathbf{x}_t)\\
&\stackrel{(a)}{\leq} \Delta(\overline{\mathbf{x}}_t) + f_t(\overline{\mathbf{x}}_t) + c_t\sigma_{t-1}(\overline{\mathbf{x}}_t) - f_t(\mathbf{x}_t) + c_t\sigma_{t-1}(\mathbf{x}_t)\\
&\stackrel{(b)}{\leq} c_t\sigma_{t-1}(\overline{\mathbf{x}}_t) + c_t\sigma_{t-1}(\overline{\mathbf{x}}_t) + c_t\sigma_{t-1}(\mathbf{x}_t) + f_t(\overline{\mathbf{x}}_t) - f_t(\mathbf{x}_t)\\
&=c_t(2\sigma_{t-1}(\overline{\mathbf{x}}_t) + \sigma_{t-1}(\mathbf{x}_t)) + \underline{f_t(\overline{\mathbf{x}}_t) - f_t(\mathbf{x}_t)},
\end{split}
\label{eq:r_t_bound_original}
\end{equation}
in which $(a)$ follows from the definition of $\Delta(\mathbf{x})$ and equation~\eqref{eq:combine_two_events},
and $(b)$ results from the fact that $\overline{\mathbf{x}}_t$ is unsaturated.
Denote by $\overline{i}$ the sub-region to which $\overline{\mathbf{x}}_t$ belongs given $\mathcal{F}_{t-1}$.
Next, we analyze the expected value of the underlined term given $\mathcal{F}_{t-1}$:
\begin{equation}
\begin{split}
\mathbb{E}&\left[f_t(\overline{\mathbf{x}}_t) - f_t(\mathbf{x}_t) | \mathcal{F}_{t-1}\right] \\
&\stackrel{(a)}{=} \mathbb{P}\left(A_t\right)\mathbb{E}\left[f_t(\overline{\mathbf{x}}_t) - f_t(\mathbf{x}_t) | \mathcal{F}_{t-1}, A_t\right] + \mathbb{P}\left(B_t\right) \mathbb{E}\left[f_t(\overline{\mathbf{x}}_t) - f_t(\mathbf{x}_t) | \mathcal{F}_{t-1}, B_{t}\right]\\
&\stackrel{(b)}{\leq} \mathbb{P}\left(B_t\right) \mathbb{E}\left[f_t(\overline{\mathbf{x}}_t) - f_t(\mathbf{x}_t) | \mathcal{F}_{t-1}, B_{t}\right]\\
&\stackrel{(c)}{\leq} \mathbb{P}\left(B_t\right)\sum^P_{i=1}\mathbb{P}\left[\mathbf{x}_t\in\mathcal{X}_i\right]\mathbb{E}\left[f_t(\overline{\mathbf{x}}_t) - f_t(\mathbf{x}_t) | \mathcal{F}_{t-1}, B_{t}, \mathbf{x}_t\in\mathcal{X}_i \right]\\
&\leq \mathbb{P}\left(B_t\right)\sum^P_{i=1}\mathbb{E}\left[f_t(\overline{\mathbf{x}}_t) - f_t(\mathbf{x}_t) | \mathcal{F}_{t-1}, B_{t}, \mathbf{x}_t\in\mathcal{X}_i \right]\\
&\stackrel{(d)}{\leq} \mathbb{P}\left(B_t\right) \sum^P_{i=1}\mathbb{E}\left[\sum^N_{n=1}\varphi^{(i)}_n \widetilde{f}^n_t (\overline{\mathbf{x}}_t)+\Delta^{(i)}_{t} - \sum^N_{n=1}\varphi^{(i)}_n \widetilde{f}^n_t (\mathbf{x}_t) + \Delta^{(i)}_{t} \bigg| \mathcal{F}_{t-1}, B_{t}, \mathbf{x}_t\in\mathcal{X}_i\right]\\
&\stackrel{(e)}{\leq} \mathbb{P}\left(B_t\right)\sum^P_{i=1} \mathbb{E}\left[ \left[\boldsymbol{\phi}(\overline{\mathbf{x}}_t)^{\top}-\boldsymbol{\phi}(\mathbf{x}_t)^{\top}\right] \underline{\sum^N_{n=1} \varphi^{(i)}_n\boldsymbol{\omega}_{n,t}} + 2\Delta^{(i)}_{t} \bigg| \mathcal{F}_{t-1}, B_{t}, \mathbf{x}_t\in\mathcal{X}_i\right],
\end{split}
\label{eq:long_proof_1}
\end{equation}
in which $(a)$ and $(c)$ result from the tower rule of expectation; 
$(b)$ follows since conditioned on the event $A_t$, i.e., $\mathbf{x}_t=\arg\max_{\mathbf{x}\in\mathcal{X}}f_t(\mathbf{x})$, we have that $f_t(\overline{\mathbf{x}}_t) - f_t(\mathbf{x}_t) \leq 0$;
$(d)$ results from Lemma~\ref{lemma:bound_gt_ft} and hence holds with probability $\geq 1-\delta/2$;
$(e)$ is a consequence of the definition of $\widetilde{f}^n_t$: $\widetilde{f}^n_t(\mathbf{x})=\boldsymbol{\phi}(\mathbf{x})^{\top}\boldsymbol{\omega}_{n,t}, \forall \mathbf{x} \in \mathcal{X}$.
Next, we further decompose the underlined term $\underline{\sum^N_{n=1} \varphi^{(i)}_n\boldsymbol{\omega}_{n,t}}$ by:
\begin{equation}
\begin{split}
\sum^N_{n=1} &\varphi^{(i)}_n\boldsymbol{\omega}_{n,t} = \frac{\sum^N_{n=1} q \varphi^{(i)}_n\boldsymbol{\omega}_{n,t}}{q} = \frac{\sum^N_{n=1} \mathbb{E}_{\mathbb{I}_n}[\mathbb{I}_n] \varphi^{(i)}_n\boldsymbol{\omega}_{n,t}}{q} = \mathbb{E}_{\mathbb{I}_{1:n}}\left[\frac{\sum^N_{n=1} \mathbb{I}_n \varphi^{(i)}_n\boldsymbol{\omega}_{n,t}}{q}\right]\\
&= \mathbb{E}_{\mathbb{I}_{1:n}}\bigg[\frac{\sum^N_{n=1} \mathbb{I}_n \varphi^{(i)}_n\boldsymbol{\omega}_{n,t}}{q} - \frac{\sum^N_{n=1} \mathbb{I}_n \varphi^{(i)}_n\widehat{\boldsymbol{\omega}}_{n,t}}{q} - \boldsymbol{\eta} + \frac{\sum^N_{n=1} \mathbb{I}_n \varphi^{(i)}_n\widehat{\boldsymbol{\omega}}_{n,t}}{q} + \boldsymbol{\eta} \bigg]\\
&= \mathbb{E}_{\mathbb{I}_{1:n}}\bigg[\frac{\sum^N_{n=1} \mathbb{I}_n \varphi^{(i)}_n\boldsymbol{\omega}_{n,t}}{q} - \frac{\sum^N_{n=1} \mathbb{I}_n \varphi^{(i)}_n\widehat{\boldsymbol{\omega}}_{n,t}}{q} - \boldsymbol{\eta} + \boldsymbol{\omega}^{(i)}_t \bigg],
\end{split}
\label{eq:long_proof_2}
\end{equation}
where in the last equality we have made use of the definition of $\boldsymbol{\omega}^{(i)}_t$~\eqref{eq:define:omega:i:t:camera}.
Next, we plug~\eqref{eq:long_proof_2} back into~\eqref{eq:long_proof_1}:
\begin{equation}
\begin{split}
\mathbb{E}&\left[f_t(\overline{\mathbf{x}}_t) - f_t(\mathbf{x}_t) | \mathcal{F}_{t-1}\right]\\
&\leq \mathbb{P}\left(B_t\right)\sum^P_{i=1} \mathbb{E}\Bigg[ \left[\boldsymbol{\phi}(\overline{\mathbf{x}}_t)^{\top}-\boldsymbol{\phi}(\mathbf{x}_t)^{\top}\right] \mathbb{E}_{\mathbb{I}_{1:n}}\bigg[\frac{\sum^N_{n=1} \mathbb{I}_n \varphi^{(i)}_n\boldsymbol{\omega}_{n,t}}{q} - \frac{\sum^N_{n=1} \mathbb{I}_n \varphi^{(i)}_n\widehat{\boldsymbol{\omega}}_{n,t}}{q} - \boldsymbol{\eta} \bigg] + \\
&\qquad\qquad \left[\boldsymbol{\phi}(\overline{\mathbf{x}}_t)^{\top}\mathbb{E}_{\mathbb{I}_{1:n}}[\boldsymbol{\omega}^{(i)}_t] - \boldsymbol{\phi}(\mathbf{x}_t)^{\top}\mathbb{E}_{\mathbb{I}_{1:n}}[\boldsymbol{\omega}^{(i)}_t]\right] + 2\Delta^{(i)}_{t} \bigg| \mathcal{F}_{t-1}, B_{t}, \mathbf{x}_t\in\mathcal{X}_i\Bigg]\\
&\leq \mathbb{P}\left(B_t\right)\sum^P_{i=1} \mathbb{E}\Bigg[ \underbrace{\left[\boldsymbol{\phi}(\overline{\mathbf{x}}_t)^{\top}-\boldsymbol{\phi}(\mathbf{x}_t)^{\top}\right] \mathbb{E}_{\mathbb{I}_{1:n}}\bigg[\frac{\sum^N_{n=1} \mathbb{I}_n \varphi^{(i)}_n\boldsymbol{\omega}_{n,t}}{q} - \frac{\sum^N_{n=1} \mathbb{I}_n \varphi^{(i)}_n\widehat{\boldsymbol{\omega}}_{n,t}}{q}\bigg]}_{A_1} + \\
&\qquad \quad \Big[\underbrace{\phi(\overline{\mathbf{x}}_t)^{\top}\mathbb{E}_{\mathbb{I}_{1:n}}[\boldsymbol{\omega}^{(i)}_t] - \phi(\overline{\mathbf{x}}_t)^{\top} \boldsymbol{\omega}^{(i)}_t}_{A_2} + \underbrace{\boldsymbol{\phi}(\overline{\mathbf{x}}_t)^{\top} \boldsymbol{\omega}^{(i)}_t - \boldsymbol{\phi}(\overline{\mathbf{x}}_t)^{\top} \boldsymbol{\omega}^{(\overline{i})}_t}_{A_3} + \underbrace{\boldsymbol{\phi}(\overline{\mathbf{x}}_t)^{\top} \boldsymbol{\omega}^{(\overline{i})}_t - \boldsymbol{\phi}(\mathbf{x}_t)^{\top} \boldsymbol{\omega}^{(i)}_t}_{A_4} +\\
&\qquad \quad \underbrace{\boldsymbol{\phi}(\mathbf{x}_t)^{\top} \boldsymbol{\omega}^{(i)}_t - \boldsymbol{\phi}(\mathbf{x}_t)^{\top}\mathbb{E}_{\mathbb{I}_{1:n}}[\boldsymbol{\omega}^{(i)}_t]}_{A_5} \Big] \underbrace{ - \left[ \boldsymbol{\phi}(\overline{\mathbf{x}}_t)^{\top} - \boldsymbol{\phi}(\mathbf{x}_t)^{\top} \right]\boldsymbol{\eta}}_{A_6} + 2\Delta^{(i)}_{t} \bigg| \mathcal{F}_{t-1}, B_{t}, \mathbf{x}_t\in\mathcal{X}_i\Bigg]
\end{split}
\label{eq:exp_regret_unfinished}
\end{equation}
Next, we separately upper-bound the terms $A_1$ to $A_6$.
Firstly, we bound the term $A_1$. 
Define $\mathcal{C}_t \triangleq \{n\in[N] \Big| \norm{\boldsymbol{\omega}_{n,t}}_2 > S/\sqrt{P} \}$, which is the same as the definition in Theorem~\ref{theorem:dp_fts_de}. That is, $\mathcal{C}_t$ contains the indices of those agents whose vector of $\boldsymbol{\omega}_{n,t}$ has a larger $L_2$ norm than $S/\sqrt{P}$ in iteration $t$. $A_1$ can thus be analyzed as:
\begin{equation}
\begin{split}
\bigg|&\left[\boldsymbol{\phi}(\overline{\mathbf{x}}_t)^{\top}-\boldsymbol{\phi}(\mathbf{x}_t)^{\top}\right] \mathbb{E}_{\mathbb{I}_{1:n}}\bigg[\frac{\sum^N_{n=1} \mathbb{I}_n \varphi^{(i)}_n\boldsymbol{\omega}_{n,t}}{q} - \frac{\sum^N_{n=1} \mathbb{I}_n \varphi^{(i)}_n\widehat{\boldsymbol{\omega}}_{n,t}}{q} \bigg] \bigg|\\
&\stackrel{(a)}{=}\bigg|\left[\boldsymbol{\phi}(\overline{\mathbf{x}}_t)^{\top}-\boldsymbol{\phi}(\mathbf{x}_t)^{\top}\right]\sum^N_{n=1}\varphi^{(i)}_n(\boldsymbol{\omega}_{n,t}-\widehat{\boldsymbol{\omega}}_{n,t})\bigg|\\
&\stackrel{(b)}{=}\bigg|\left[\boldsymbol{\phi}(\overline{\mathbf{x}}_t)^{\top}-\boldsymbol{\phi}(\mathbf{x}_t)^{\top}\right]\sum_{n\in \mathcal{C}_t}\varphi^{(i)}_n(\boldsymbol{\omega}_{n,t}-\widehat{\boldsymbol{\omega}}_{n,t})\bigg|\\
&=\bigg|\sum_{n\in \mathcal{C}_t}\varphi^{(i)}_n \left[\boldsymbol{\phi}(\overline{\mathbf{x}}_t)^{\top}-\boldsymbol{\phi}(\mathbf{x}_t)^{\top}\right]\left[\boldsymbol{\omega}_{n,t}-\widehat{\boldsymbol{\omega}}_{n,t}\right]\bigg|\\
&=\bigg|\sum_{n\in \mathcal{C}_t}\varphi^{(i)}_n \left[\boldsymbol{\phi}(\overline{\mathbf{x}}_t)^{\top}\boldsymbol{\omega}_{n,t} + \boldsymbol{\phi}(\mathbf{x}_t)^{\top}\widehat{\boldsymbol{\omega}}_{n,t} - \boldsymbol{\phi}(\overline{\mathbf{x}}_t)^{\top}\widehat{\boldsymbol{\omega}}_{n,t} - \boldsymbol{\phi}(\mathbf{x}_t)^{\top}\boldsymbol{\omega}_{n,t}\right]\bigg|\\
&\leq \sum_{n\in \mathcal{C}_t}\varphi^{(i)}_n \left[\big|\boldsymbol{\phi}(\overline{\mathbf{x}}_t)^{\top}\boldsymbol{\omega}_{n,t}\big| + \big|\boldsymbol{\phi}(\mathbf{x}_t)^{\top}\widehat{\boldsymbol{\omega}}_{n,t}\big| + \big|\boldsymbol{\phi}(\overline{\mathbf{x}}_t)^{\top}\widehat{\boldsymbol{\omega}}_{n,t}\big| + \big|\boldsymbol{\phi}(\mathbf{x}_t)^{\top}\boldsymbol{\omega}_{n,t}\big|\right]\\
&\leq \sum_{n\in \mathcal{C}_t}\varphi^{(i)}_n \left[\big|\boldsymbol{\phi}(\overline{\mathbf{x}}_t)^{\top}\boldsymbol{\omega}_{n,t}\big| + \big|\boldsymbol{\phi}(\mathbf{x}_t)^{\top}\boldsymbol{\omega}_{n,t}\big| + \big|\boldsymbol{\phi}(\overline{\mathbf{x}}_t)^{\top}\boldsymbol{\omega}_{n,t}\big| + \big|\boldsymbol{\phi}(\mathbf{x}_t)^{\top}\boldsymbol{\omega}_{n,t}\big|\right]\\
&\leq 2\sum_{n\in \mathcal{C}_t}\varphi^{(i)}_n \left[\big| \widetilde{f}^n_t(\overline{\mathbf{x}}_t)\big| + \big| \widetilde{f}^n_t(\mathbf{x})\big|\right]\\
&\stackrel{(c)}{\leq} 2\sum_{n\in \mathcal{C}_t}\varphi^{(i)}_n (\tilde{\Delta}_{n,t} + B + \tilde{\Delta}_{n,t} + B)\\
&= 4\sum_{n\in \mathcal{C}_t}\varphi^{(i)}_n \tilde{\Delta}_{n,t} + 4\sum_{n\in \mathcal{C}_t}\varphi^{(i)}_n B\\
&\stackrel{(d)}{\leq} 4\sum^N_{n=1}\varphi^{(i)}_n \tilde{\Delta}_{n,t} + 4B\widetilde{\varphi}^{(i)}_{\mathcal{C}_t}\\
&\stackrel{(e)}{=}4\left( \tilde{\Delta}^{(i)}_t + B\widetilde{\varphi}^{(i)}_{\mathcal{C}_t} \right),
\end{split}
\label{eq:A1}
\end{equation}
in which $(a)$ follows since $\mathbb{E}_n[\mathbb{I}_n]=q,\forall n\in[N]$;
$(b)$ follows since for those agents $n\not\in\mathcal{C}_t$, $\boldsymbol{\omega}_{n,t}-\widehat{\boldsymbol{\omega}}_{n,t}=\mathbf{0}$ because the vector $\boldsymbol{\omega}_{n,t}$ is not clipped;
$(c)$ results from Lemma~\ref{lemma:bound_gt_gn} and that $|f^n(\mathbf{x})| \leq B,\forall \mathbf{x}\in\mathcal{X},n\in[N]$ (this is because of our assumption that $\norm{f^n}_k \leq B,\forall n\in[N]$, Sec.~\ref{sec:background});
$(d)$ follows from the definition of $\widetilde{\varphi}^{(i)}_{\mathcal{C}_t}\triangleq\sum_{n\in \mathcal{C}_t}\varphi^{(i)}_n$;
$(e)$ results from the definition of $\tilde{\Delta}^{(i)}_t$~\eqref{eq:define_delta_tilde_it}.

Subsequently, we upper-bound the terms $A_2$ and $A_5$. For any $\mathbf{x}\in\mathcal{X}$, we have that
\begin{equation}
\begin{split}
\Big|\boldsymbol{\phi}&(\mathbf{x})^{\top}\mathbb{E}_{\mathbb{I}_{1:n}}[\boldsymbol{\omega}^{(i)}_t] - \boldsymbol{\phi}(\mathbf{x})^{\top} \boldsymbol{\omega}^{(i)}_t\Big|=\Big|\boldsymbol{\phi}(\mathbf{x})^{\top} \left( \mathbb{E}_{\mathbb{I}_{1:n}} \left[\frac{\sum^N_{n=1} \mathbb{I}_n \varphi^{(i)}_n\widehat{\boldsymbol{\omega}}_{n,t}}{q} \right]-\frac{\sum^N_{n=1} \mathbb{I}_n \varphi^{(i)}_n\widehat{\boldsymbol{\omega}}_{n,t}}{q}\right) \Big|\\
&=\Big|\boldsymbol{\phi}(\mathbf{x})^{\top} \left(  \frac{\sum^N_{n=1} q \varphi^{(i)}_n\widehat{\boldsymbol{\omega}}_{n,t}}{q} -\frac{\sum^N_{n=1} \mathbb{I}_n \varphi^{(i)}_n\widehat{\boldsymbol{\omega}}_{n,t}}{q}\right) \Big|\\
&=\Big| \boldsymbol{\phi}(\mathbf{x})^{\top} \frac{1}{q} \sum^N_{n=1}(q-\mathbb{I}_n)\varphi^{(i)}_n\widehat{\boldsymbol{\omega}}_{n,t} \Big|
\leq \frac{1}{q} \sum^N_{n=1} \Big| (q-\mathbb{I}_n)\varphi^{(i)}_n \boldsymbol{\phi}(\mathbf{x})^{\top}\widehat{\boldsymbol{\omega}}_{n,t} \Big|\\
&\stackrel{(a)}{\leq} \frac{1}{q} \sum^N_{n=1}\varphi^{(i)}_n \big|\boldsymbol{\phi}(\mathbf{x})^{\top}\widehat{\boldsymbol{\omega}}_{n,t} \big| \stackrel{(b)}{\leq} \frac{1}{q} \sum^N_{n=1}\varphi^{(i)}_n \big|\boldsymbol{\phi}(\mathbf{x})^{\top}\boldsymbol{\omega}_{n,t} \big|\\
&= \frac{1}{q} \sum^N_{n=1}\varphi^{(i)}_n \big|\widetilde{f}^n_t(\mathbf{x}) \big| = \frac{1}{q} \sum^N_{n=1}\varphi^{(i)}_n \big|\widetilde{f}^n_t(\mathbf{x}) - f^n(\mathbf{x}) + f^n(\mathbf{x}) \big|\\
&\leq \frac{1}{q} \sum^N_{n=1}\varphi^{(i)}_n \left(|\widetilde{f}^n_t(\mathbf{x}) - f^n(\mathbf{x})| + |f^n(\mathbf{x})|\right)\\
&\stackrel{(c)}{\leq} \frac{1}{q} \sum^N_{n=1}\varphi^{(i)}_n \left(\tilde{\Delta}_{n,t} + B\right)\\
&\stackrel{(d)}{=} \frac{1}{q} \left(\tilde{\Delta}^{(i)}_{t} + B\right),
\end{split}
\label{eq:A2_A5}
\end{equation}
in which $(a)$ follows since $|q-\mathbb{I}_n| \leq 1$;
$(b)$ results from~\eqref{eq:clip_make_norm_smaller};
$(c)$ results from Lemma~\ref{lemma:bound_gt_gn} and that 
$|f^n(\mathbf{x})| \leq B,\forall \mathbf{x}\in\mathcal{X},n\in[N]$;
$(d)$ results from the definition of $\tilde{\Delta}^{(i)}_t$~\eqref{eq:define_delta_tilde_it}.

Next, we upper-bound the term $A_3$, which arises because the sub-regions $i$ and $\overline{i}$ may be different.
We have that for any $\mathbf{x}\in\mathcal{X}$,
\begin{equation}
\begin{split}
\Big|\boldsymbol{\phi}(\mathbf{x})^{\top}\boldsymbol{\omega}^{(i)}_t &- \boldsymbol{\phi}(\mathbf{x})^{\top}\boldsymbol{\omega}^{(\overline{i})}_t\Big|=\Big|\boldsymbol{\phi}(\mathbf{x})^{\top}\left(\boldsymbol{\omega}^{(i)}_t - \boldsymbol{\omega}^{(\overline{i})}_t\right)\Big|\\
&=\Big|\boldsymbol{\phi}(\mathbf{x})^{\top}\left( \frac{\sum^N_{n=1} \mathbb{I}_n \varphi^{(i)}_n\widehat{\boldsymbol{\omega}}_{n,t}}{q} - \frac{\sum^N_{n=1} \mathbb{I}_n \varphi^{(\overline{i})}_n\widehat{\boldsymbol{\omega}}_{n,t}}{q}\right) \Big|\\
&\leq \frac{1}{q} \sum^N_{n=1} \mathbb{I}_n \Big|\varphi^{(i)}_n-\varphi^{(\overline{i})}_n\Big| \Big|\boldsymbol{\phi}(\mathbf{x})^{\top} \widehat{\boldsymbol{\omega}}_{n,t} \Big|\\
&\leq \frac{1}{q} \sum^N_{n=1} \varphi_{\max} \Big|\boldsymbol{\phi}(\mathbf{x})^{\top} \widehat{\boldsymbol{\omega}}_{n,t} \Big|\\
&\stackrel{(a)}{\leq} \frac{1}{q} \sum^N_{n=1} \varphi_{\max} \Big|\boldsymbol{\phi}(\mathbf{x})^{\top} \boldsymbol{\omega}_{n,t} \Big|\\
&\stackrel{(b)}{\leq} \frac{1}{q} \sum^N_{n=1} \varphi_{\max} \left( \tilde{\Delta}_{n,t}+B \right)\\
&\stackrel{(c)}{\leq} \frac{\varphi_{\max}}{q}\left(\Delta_t + NB\right),
\end{split}
\label{eq:A3}
\end{equation}
$(a)$ follows because of~\eqref{eq:clip_make_norm_smaller};
$(b)$ results from Lemma~\ref{lemma:bound_gt_gn} and that 
$|f^n(\mathbf{x})| \leq B,\forall \mathbf{x}\in\mathcal{X},n\in[N]$;
$(c)$ follows from~\eqref{eq:upper_bound_Delta}.

Next, regarding $A_4$, note that conditioned on the event $B_t$, $\mathbf{x}_t$ is selected by: $\mathbf{x}_t=\arg\max_{\mathbf{x}\in\mathcal{X}}\boldsymbol{\phi}(\mathbf{x})^{\top}\boldsymbol{\omega}^{(i^{[\mathbf{x}]})}_t$ in which $i^{[\mathbf{x}]}$ represents the sub-region $\mathbf{x}$ belongs to. Therefore, because $\overline{\mathbf{x}}_t\in\mathcal{X}_{\overline{i}}$ and $\mathbf{x}_t\in\mathcal{X}_i$ (since we are conditioning on this event), we have that $\boldsymbol{\phi}(\overline{\mathbf{x}}_t)^{\top} \boldsymbol{\omega}^{(\overline{i})}_t - \boldsymbol{\phi}(\mathbf{x}_t)^{\top} \boldsymbol{\omega}^{(i)}_t \leq 0$. In other words, $A_4 \leq 0$.

Finally, The term $A_6$ can be upper-bounded using standard Gaussian concentration inequality: 
\begin{equation}
\begin{split}
\Big|\left[ \boldsymbol{\phi}(\overline{\mathbf{x}}_t)^{\top} - \boldsymbol{\phi}(\mathbf{x}_t)^{\top} \right]\boldsymbol{\eta}\Big| &\leq \norm{\boldsymbol{\phi}(\overline{\mathbf{x}}_t) - \boldsymbol{\phi}(\mathbf{x}_t)}_2 \norm{\boldsymbol{\eta}}_2\\
&\leq \left(\norm{\boldsymbol{\phi}(\overline{\mathbf{x}}_t)}_2 + \norm{\boldsymbol{\phi}(\mathbf{x}_t)}_2\right)\norm{\boldsymbol{\eta}}_2 \stackrel{(a)}{\leq} 2\norm{\boldsymbol{\eta}}_2\\
&\stackrel{(b)}{\leq} \frac{2zS\varphi_{\max}}{q}\sqrt{2M\log\frac{8M}{\delta}},
\end{split}
\label{eq:A6}
\end{equation}
where $(a)$ follows since the random features have been constructed such that $\norm{\boldsymbol{\phi}(\mathbf{x})}^2_2 = \sigma_0^2 \leq 1$~\cite{dai2020federated},
and $(b)$ follows from standard Gaussian concentration inequality and hence holds with probability $>1-\delta/8$.

Now we can exploit the upper bounds on the terms $A_1$ to $A_6$ we have derived above (equations~\eqref{eq:A1},~\eqref{eq:A2_A5},~\eqref{eq:A3},~\eqref{eq:A6}), and continue to upper-bound $\mathbb{E}\left[f_t(\overline{\mathbf{x}}_t)-f_t(\mathbf{x}_t)|\mathcal{F}_{t-1}\right]$ following~\eqref{eq:exp_regret_unfinished}:
\begin{equation}
\begin{split}
    \mathbb{E}[&f_t(\overline{\mathbf{x}}_t)-f_t(\mathbf{x}_t)|\mathcal{F}_{t-1}] \leq \mathbb{P}(B_t)\sum^P_{i=1}\mathbb{E}\Bigg[\underbrace{4 \left(\tilde{\Delta}^{(i)}_t + B\widetilde{\varphi}^{(i)}_{\mathcal{C}_t}\right)}_{A_1} + \underbrace{\frac{2}{q} \left(\tilde{\Delta}^{(i)}_{t} + B\right)}_{A_2+A_5} +\\
    &\qquad \underbrace{\frac{\varphi_{\max}}{q}\left(\Delta_t + NB\right)}_{A_3} + \underbrace{\frac{2zS\varphi_{\max}}{q}\sqrt{2M\log\frac{8M}{\delta}}}_{A_6}+ 2\Delta^{(i)}_{t} \bigg| \mathcal{F}_{t-1}, B_{t},\mathbf{x}_t\in\mathcal{X}_i\Bigg]\\
    &= (1-p_t)\sum^P_{i=1}\Bigg[4 \left(\tilde{\Delta}^{(i)}_t + B\mathbb{E}\left[\widetilde{\varphi}^{(i)}_{\mathcal{C}_t}|\mathcal{F}_{t-1}\right]\right) + \frac{2}{q} \left(\tilde{\Delta}^{(i)}_{t} +B\right) + \\
    &\qquad \frac{\varphi_{\max}}{q}\left(\Delta_t + NB\right)+ \frac{2zS\varphi_{\max}}{q}\sqrt{2M\log\frac{8M}{\delta}} + 2\Delta^{(i)}_{t}\Bigg]\\
    &= (1-p_t)\sum^P_{i=1}\Bigg[4B \mathbb{E}\left[\widetilde{\varphi}^{(i)}_{\mathcal{C}_t}|\mathcal{F}_{t-1}\right] + \left(\frac{2}{q}+4\right) \tilde{\Delta}^{(i)}_{t} + 2\Delta^{(i)}_{t} + \frac{\varphi_{\max}}{q} \Delta_t +\\
    &\qquad B\left(\frac{2}{q} +\frac{N\varphi_{\max}}{q} \right) + \frac{2zS\varphi_{\max}}{q}\sqrt{2M\log\frac{8M}{\delta}}\Bigg]\\
    &\stackrel{(a)}{\leq} (1-p_t)\sum^P_{i=1}\Bigg[4B \mathbb{E}\left[\widetilde{\varphi}^{(i)}_{\mathcal{C}_t}|\mathcal{F}_{t-1}\right] + \left(\frac{2}{q}+6+\frac{\varphi_{\max}}{q}\right) \Delta_{t} +\\
    &\qquad B\left(\frac{2}{q} +\frac{N\varphi_{\max}}{q} \right) + \frac{2zS\varphi_{\max}}{q}\sqrt{2M\log\frac{8M}{\delta}}\Bigg]\\
    &=4B \mathbb{E} \left[(1-p_t)\sum^P_{i=1} \widetilde{\varphi}^{(i)}_{\mathcal{C}_t}|\mathcal{F}_{t-1}\right] + (1-p_t)\Bigg[P\left(\frac{2}{q}+6+\frac{\varphi_{\max}}{q}\right) \Delta_{t} +\\
    &\qquad PB\left(\frac{2}{q} +\frac{N\varphi_{\max}}{q} \right) + P\frac{2zS\varphi_{\max}}{q}\sqrt{2M\log\frac{8M}{\delta}}\Bigg]\\
    &= 4B \mathbb{E}\left[\vartheta_t|\mathcal{F}_{t-1}\right] + \psi_t,
\end{split}
\label{eq:proof_exp_regret_final}
\end{equation}
where $(a)$ follows because $\tilde{\Delta}^{(i)}_t\leq \Delta^{(i)}_t\leq \Delta_t,\forall i\in[P]$.
In the last equality, we have made use of the definitions of $\vartheta_t$ and $\psi_t$.
Note that since we have made use of Lemma~\ref{lemma:bound_gt_ft}~\eqref{eq:long_proof_1} which holds with probability $\geq 1-\delta/2$, and Gaussian concentration inequality~\eqref{eq:A6} which holds with probability $\geq 1-\delta/8$, equation~\eqref{eq:proof_exp_regret_final} 
holds with probability $\geq 1 - \delta/2 - \delta/8=1-5\delta/8$.

Finally, we plug~\eqref{eq:proof_exp_regret_final} back into~\eqref{eq:r_t_bound_original}:
\begin{equation}
\begin{split}
\mathbb{E}&\left[r_t | \mathcal{F}_{t-1}\right] \\
&\leq \mathbb{E}\left[c_t(2\sigma_{t-1}(\overline{\mathbf{x}}_t) + \sigma_{t-1}(\mathbf{x}_t)) + f_t(\overline{\mathbf{x}}_t) - f_t(\mathbf{x}_t) | \mathcal{F}_{t-1}\right] + 2B\mathbb{P}\left[\overline{E^{f_t}(t)} | \mathcal{F}_{t-1} \right]\\
&\leq \mathbb{E}\left[c_t(2\sigma_{t-1}(\overline{\mathbf{x}}_t) + \sigma_{t-1}(\mathbf{x}_t)) | \mathcal{F}_{t-1}\right] + \mathbb{E}\left[f_t(\overline{\mathbf{x}}_t) - f_t(\mathbf{x}_t) | \mathcal{F}_{t-1}\right] + 2B\mathbb{P}\left[\overline{E^{f_t}(t)} | \mathcal{F}_{t-1} \right]\\
& \stackrel{(a)}{\leq} \frac{2c_t}{P_t}\mathbb{E}\left[\sigma_{t-1}(\mathbf{x}_t)| \mathcal{F}_{t-1} \right] + c_t\mathbb{E}\left[\sigma_{t-1}(\mathbf{x}_t) | \mathcal{F}_{t-1}\right] + 4B \mathbb{E}\left[\vartheta_t|\mathcal{F}_{t-1}\right] + \psi_t + \frac{2B}{t^2}\\
&\leq c_t\left(1 + \frac{2}{P_t}\right)\mathbb{E}\left[\sigma_{t-1}(\mathbf{x}_t)| \mathcal{F}_{t-1} \right] + 4B \mathbb{E}\left[\vartheta_t|\mathcal{F}_{t-1}\right] + \psi_t + \frac{2B}{t^2}\\
&\stackrel{(b)}{\leq} c_t\left(1 + \frac{10}{pp_1}\right)\mathbb{E}\left[\sigma_{t-1}(\mathbf{x}_t)| \mathcal{F}_{t-1} \right] + 4B \mathbb{E}\left[\vartheta_t|\mathcal{F}_{t-1}\right] + \psi_t + \frac{2B}{t^2},
\end{split}
\label{eq:expected_inst_regret}
\end{equation}
in which $(a)$ follows from~\eqref{eq:bound_x_t_bar_Pt} and~\eqref{eq:proof_exp_regret_final}, and $(b)$ follows since:
\begin{equation}
\frac{2}{P_t} = \frac{2}{p_t(p - \frac{1}{t^2})}\leq \frac{10}{p p_t}\leq \frac{10}{pp_1},
\end{equation}
which was valid because $1/(p-1/t^2) \leq 5/p$ and $p_t\geq p_1$ for all $t\geq 1$.

Note that since the proof of~\eqref{eq:expected_inst_regret} makes use of~\eqref{eq:proof_exp_regret_final}, therefore,~\eqref{eq:expected_inst_regret}, as well as Lemma~\ref{lemma:upper_bound_expected_regret}, also holds with probability of $\geq 1-5\delta/8$.

\end{proof}

\begin{lemma}
Given $\delta \in (0, 1)$, then with probability of at least $1 - \delta$,
\begin{equation*}
\begin{split}
R_T\leq &c_T\left(1 + \frac{10}{p p_1}\right) \mathcal{O}(\sqrt{T\gamma_T}) + \sum^T_{t=1}\psi_t + \frac{B\pi^2}{3} + 4B\sum^T_{t=1}\vartheta_t +\\
&\left[c_T\left(1 + \frac{4B}{p} + \frac{10}{p p_1}
\right) + \psi_1 + 4B \right]\sqrt{2T\log\frac{8}{\delta}},
\end{split}
\end{equation*}
in which 
$\gamma_T$ is the maximum information gain about $f$ obtained from any set of $T$ observations.
\end{lemma}
\begin{proof}
The proof resembles the that of Lemma 11 of~\cite{dai2020federated}, and is hence omitted. A difference from Lemma 11 of~\cite{dai2020federated} is that an error probability of $\delta/8$
has been used in the Azuma-Hoeffding Inequality in the proof here.
\end{proof}

Finally, we are ready to prove Theorem~\ref{theorem:dp_fts_de}. Recall that $c_t = \mathcal{O}\left(\left(B + \sqrt{\gamma_t+\log(1/\delta)}\right)\sqrt{\log t}\right)$. Therefore, 
\begin{equation}
\begin{split}
R_T &= \mathcal{O}\Bigg(\frac{1}{p_1}\left(B + \sqrt{\gamma_T + \log\frac{1}{\delta}}\right)\sqrt{\log T}\sqrt{T\gamma_T} + \sum^T_{t=1}\psi_t + B\sum^T_{t=1}\vartheta_t + \\
&\qquad \left(B + \frac{1}{p_1}\right)\left(B + \sqrt{\gamma_T+\log\frac{1}{\delta}}\right)\sqrt{\log T}\sqrt{T\log\frac{1}{\delta}}\Bigg)\\
&=\mathcal{O}\left(\left(B + \frac{1}{p_1}\right)\sqrt{T\log T\gamma_T \log\frac{1}{\delta}\left(\gamma_T + \log\frac{1}{\delta}\right)} + \sum^T_{t=1}\psi_t + B\sum^T_{t=1}\vartheta_t\right)\\
&=\tilde{\mathcal{O}}\left(\left(B+\frac{1}{p_1}\right)\gamma_T\sqrt{T} + \sum^T_{t=1}\psi_t + B\sum^T_{t=1}\vartheta_t\right),
\end{split}
\end{equation}
which finally completes the proof.

\section{Experiments}
\label{app:experiments}
As we have mentioned in the main text (last paragraph of Sec.~\ref{subsec:de}), we choose the weights for sub-region $i$ to be
$\varphi^{(i)}_n=\frac{\exp(( a\mathbb{I}^{(i)}_n + 1)/\mathcal{T})}{\sum^N_{n=1} \exp(( a\mathbb{I}^{(i)}_n + 1)/\mathcal{T})}$, where $\mathbb{I}^{(i)}_n$ is an indicator variable that equals $1$ if agent $n$ is assigned to explore $\mathcal{X}_i$ and equals $0$ otherwise. We set $a=15$ in all our experiments, and gradually increase the temperature $\mathcal{T}$ from $1$ to $+\infty$.
Specifically, for the synthetic experiments, we choose the temperature $\mathcal{T}$ as $\mathcal{T}_t = a / (a_t - 1),\forall t\geq 1$; we set $a_t=a+1=16$ for the first $5$ iterations ($t\leq 5$), decay the value of $a_t$ linearly to $1$ in the next $5$ iterations (i.e., $a_t=16,12.25,8.5,4.75,1$ for $t=6,\ldots,10$), and fix $a_t=1,\forall t > 10$. Note that when $a_t=1$, $\mathcal{T}_t=+\infty$ (i.e., after $10$ iterations) and the distribution becomes uniform among all agents.
Similarly, for all real-world experiments, we use the same softmax weighting scheme except that we fix $a_t=a+1=16$ for the first $10$ iterations, decay the value of $a_t$ linearly to $1$ in the next $30$ iterations, and fix $a_t=1$ afterwards. That is, the distribution becomes uniform among all agents after $40$ iterations.
All our experiments are performed on a computing cluster where each device has one NVIDIA Tesla T4 GPU and 48 cores of Xeon Silver 4116 (2.1Ghz) processors.

\subsection{Synthetic Experiments}
\label{app:synth_exp}
\subsubsection{Detailed Experimental Setting}
\label{app:synth:exp:setting}
Our synthetic experiments involve $N=200$ agents.
We define the domain of the synthetic functions to be $1$-dimensional and discrete, i.e., an equally spaced grid on the $1$-dimensional interval $[0,1]$ with a domain size of $|\mathcal{X}|=1000$.
To generate the objective functions for the $N=200$ different agents, we firstly sample a function $f$ from a GP with the SE kernel and a length scale of $0.03$, and normalize the function values into the range $[0, 1]$.
Next, for every agent $\mathcal{A}_n,\forall n\in[N]$, we go through all $|\mathcal{X}|=1000$ inputs in the entire domain, and for each input $\mathbf{x}$, we derive the function value for agent $\mathcal{A}_n$ as $f^n(\mathbf{x}) = f(\mathbf{x}) + d$, in which $d=0.02$ or $=-0.02$ with equal probability (i.e., a probability of $0.5$ each).
In this way, the objective functions of all agents are related to each other.
When observing a function value, we add a Gaussian noise $\zeta \sim \mathcal{N}(0,\sigma^2)$ with a variance of $\sigma^2=0.01$ (Sec.~\ref{sec:background}).

To construct the $P$ sub-regions to be used for distributed exploration (DE), we simply need to divide the interval $[0,1]$ into $P$ disjoint hyper-rectangles with equal volumes. For example, when $P=2$, the two sub-regions contain the inputs in the sub-regions $[0, 0.5)$ and $[0.5,1]$ respectively; 
when $P=3$, the three sub-regions include the inputs in the sub-regions $[0,1/3)$, $[1/3,2/3)$ and $[2/3,1]$ respectively.

\subsubsection{More Experimental Results}
\label{app:synth:more:results}
\textbf{Comparison between DP-FTS-DE and DP-FTS.}
We have shown in the main text (Fig.~\ref{fig:synth_1}a) that FTS-DE significantly outperforms FTS without DE.
Here, we demonstrate in Fig.~\ref{fig:synth_3} that after DP is integrated, DP-FTS-DE still yields a significantly better utility than DP-FTS for the same level of privacy guarantee (loss).
These results justify the practical benefit of the technique of DE (Sec.~\ref{subsec:de}).
Note that for a fair comparison, we have used a smaller value of $S$ for DP-FTS without DE such that a similar percentage of vectors are clipped for both DP-FTS-DE and DP-FTS.
\begin{figure}
     \centering
     \begin{tabular}{c}
         \includegraphics[width=0.4\linewidth]{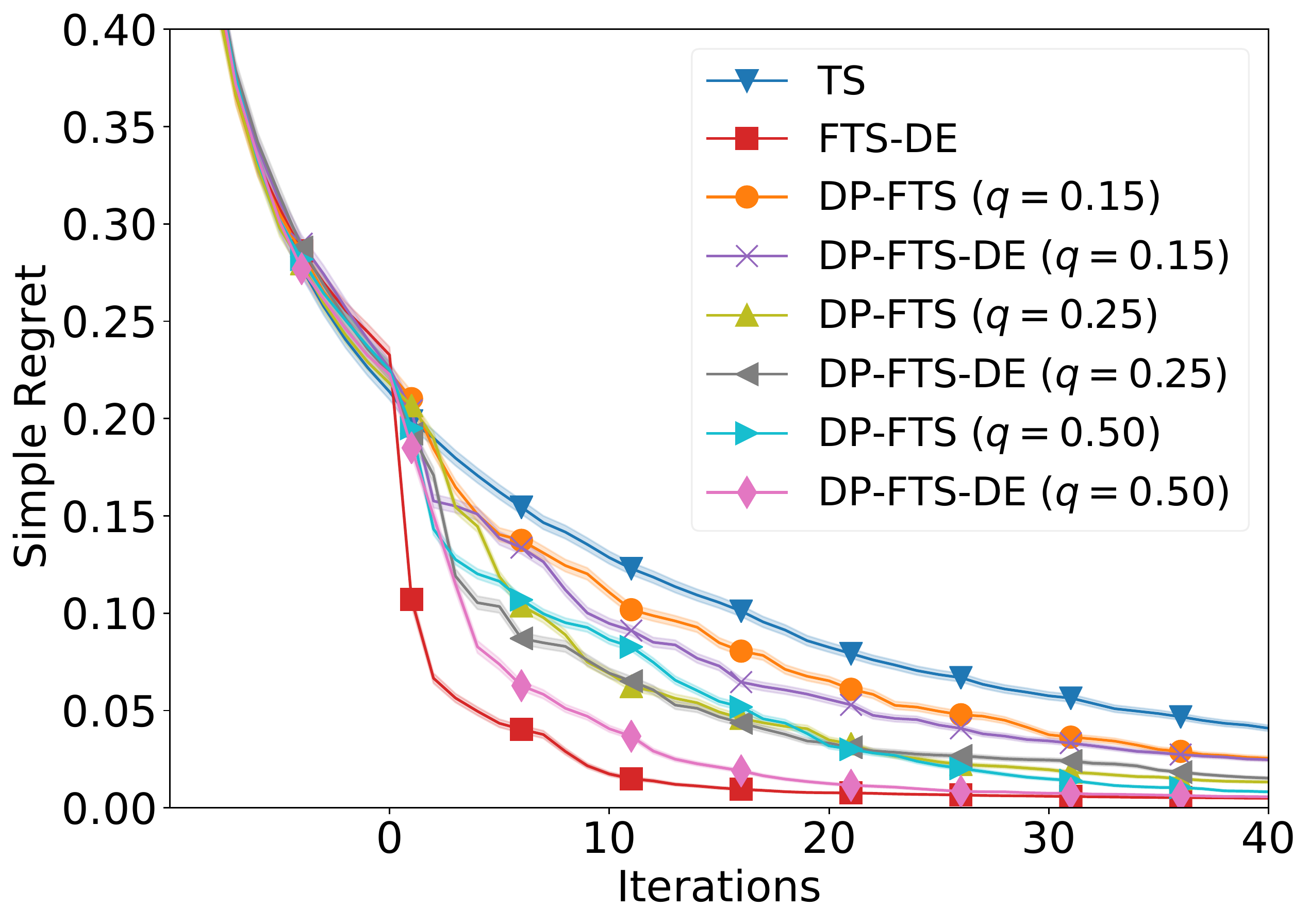}
     \end{tabular}
     \caption{Comparisons of the performances of DP-FTS (without DE) and DP-FTS-DE. For a fair comparison, we have used $S=8$ and $S=11$ for DP-FTS and DP-FTS-DE respective, such that a similarly small percentage vectors are clipped for both algorithms ($0.31\%$ for DP-FTS and $0.80\%$ for DP-FTS-DE). We have used $z=1.0$ for both algorithms.}
     \label{fig:synth_3}
\end{figure}

\textbf{Investigation of DE.}
We also investigate the importance of both of the major components of the DE technique (Sec.~\ref{subsec:de}): (a) assigning every agent to explore only a local sub-region instead of the entire domain, and (b) giving more weights to those agents exploring the particular sub-region. In Fig.~\ref{fig:synth_4}, the orange curve is obtained by removing component (b) (i.e., in every iteration and for each sub-region, we give equal weights to all agents), the purple curve is derived by removing component (a) (i.e., letting every agent explore the entire domain at initialization instead of a smaller local sub-region).
As the figure shows, the performances of both the orange and purple curves are significantly worse than our FTS-DE algorithm (red curve), which verifies that both of these components are critical for the practical performance of FTS-DE.
\begin{figure}
     \centering
     \begin{tabular}{c}
         \includegraphics[width=0.4\linewidth]{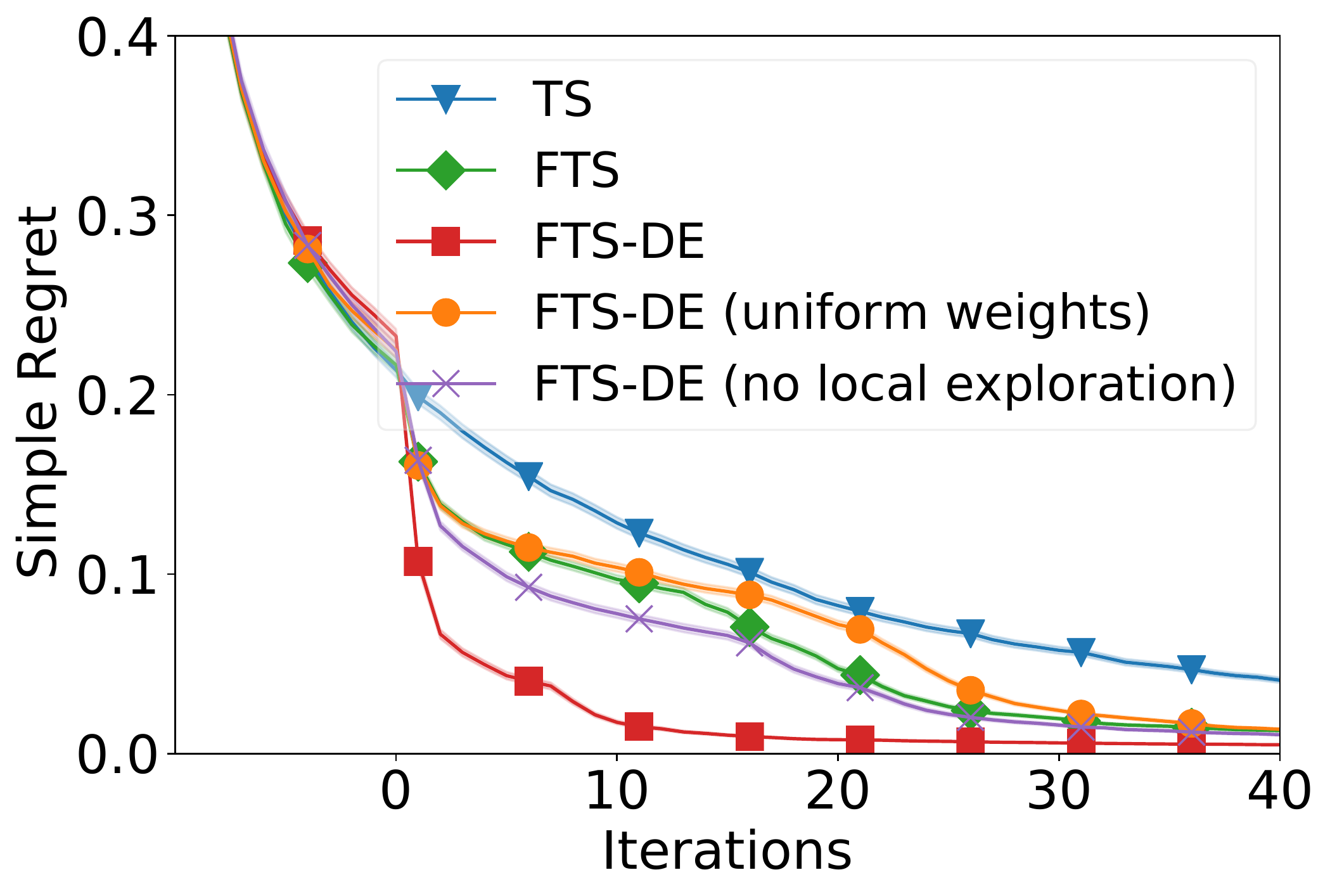}
     \end{tabular}
     \caption{Investigating the importance of both major components of the technique of distributed exploration (DE). 
     The orange curve is obtained by giving equal weights to all agent for every sub-region, and the purple curve is derived by letting every agent explore the entire domain at initialization instead of a local sub-region.}
     \label{fig:synth_4}
\end{figure}

\textbf{Trade-off Induced by $P$.}
As we have discussed at the end of Sec.~\ref{sec:theoretical_analysis}, the value of $P$ (i.e., the number of sub-regions) induces a trade-off about the practical performance of our DP-FTS-DE algorithm.
Here we empirically verify this trade-off in Fig.~\ref{fig:effect_of_p_with_dp}.
As shown in the figure, for the same values of $q$, $z$ and $S$, a smaller value of $P$ (i.e., larger local sub-regions) may deteriorate the performance (orange curve) since larger sub-regions are harder for the GP surrogate to model, however, a larger value of $P$ may also result in a worse performance (yellow curve) since it causes the vectors from more agents to be clipped (Sec.~\ref{sec:theoretical_analysis}).
These observations verify our discussions in the last paragraph of Sec.~\ref{sec:theoretical_analysis}.
\begin{figure}
     \centering
     \begin{tabular}{c}
         \includegraphics[width=0.4\linewidth]{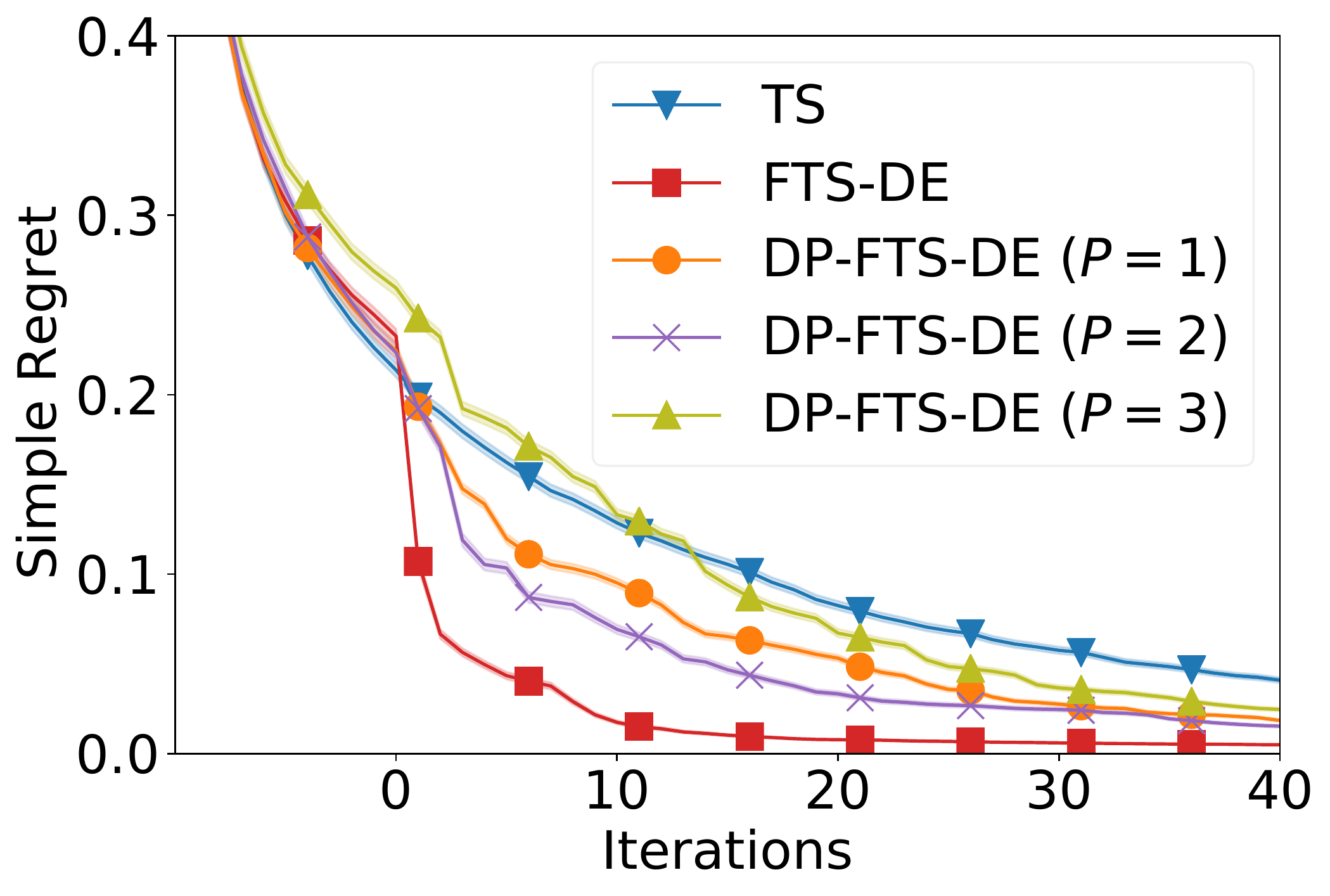}
     \end{tabular}
     \caption{Trade-off induced by $P$ regarding the practical performance of our DP-FTS-DE algorithm. Note that a larger $P$ reduces the size of every local sub-region and hence leads to a better modeling by the GP surrogates, yet also negatively impacts the performance by causing more vectors to be clipped. Here we have used $q=0.25,z=1.0,S=11.0$ for all values of $P$.}
     \label{fig:effect_of_p_with_dp}
\end{figure}

\textbf{Robustness Against Heterogeneous Agents.}
We use another experiment to explore the robustness of our algorithm against agent heterogeneity, i.e., how our algorithm performs when the objective functions of different agents are significantly different.
To begin with, we sample a function $f_{\text{base}}$ from a GP (the detailed setups such as the domain and the SE kernel are the same as those used in App.~\ref{app:synth:exp:setting}).
Next, for every agent $i=1,\ldots,50$, we independently sample another function $f^i$ and then use $f^i \leftarrow \alpha f^i + (1-\alpha) f_{\text{base}}$ as the objective function for agent $i$ in which $\alpha\in[0,1]$. As a result, the parameter $\alpha$ controls the difference between the objective functions $f^i$'s of different agents such that a larger $\alpha$ means that the $f^i$'s are more different. Fig.~\ref{fig:exp:hetero:agents}a shows the performances of standard TS and our FTS-DE when $\alpha=0.7$, in which FTS-DE is still able to outperform TS although the performance improvement is significantly smaller than that observed in Fig.~\ref{fig:synth_1}.
Fig.~\ref{fig:exp:hetero:agents}b plots the results when the objective functions $f^i$'s are extremely heterogeneous, i.e., when $\alpha=1.0$ which implies that the $f^i$'s are \emph{independent}.
The figure shows that in this adverse scenario, when $1-p_t=1/\sqrt{t}$, FTS-DE (green curve) performs worse than standard TS (blue curve) due to the extremely high degree of heterogeneity among the agents. However, as shown by the orange curve, we can improve the performance of FTS-DE to make it comparable with standard TS by letting $1-p_t$ decrease faster such that the impact of the other agents are diminished faster.
\begin{figure}
     \centering
     \begin{tabular}{cc}
         \hspace{-4mm} 
         \includegraphics[width=0.4\linewidth]{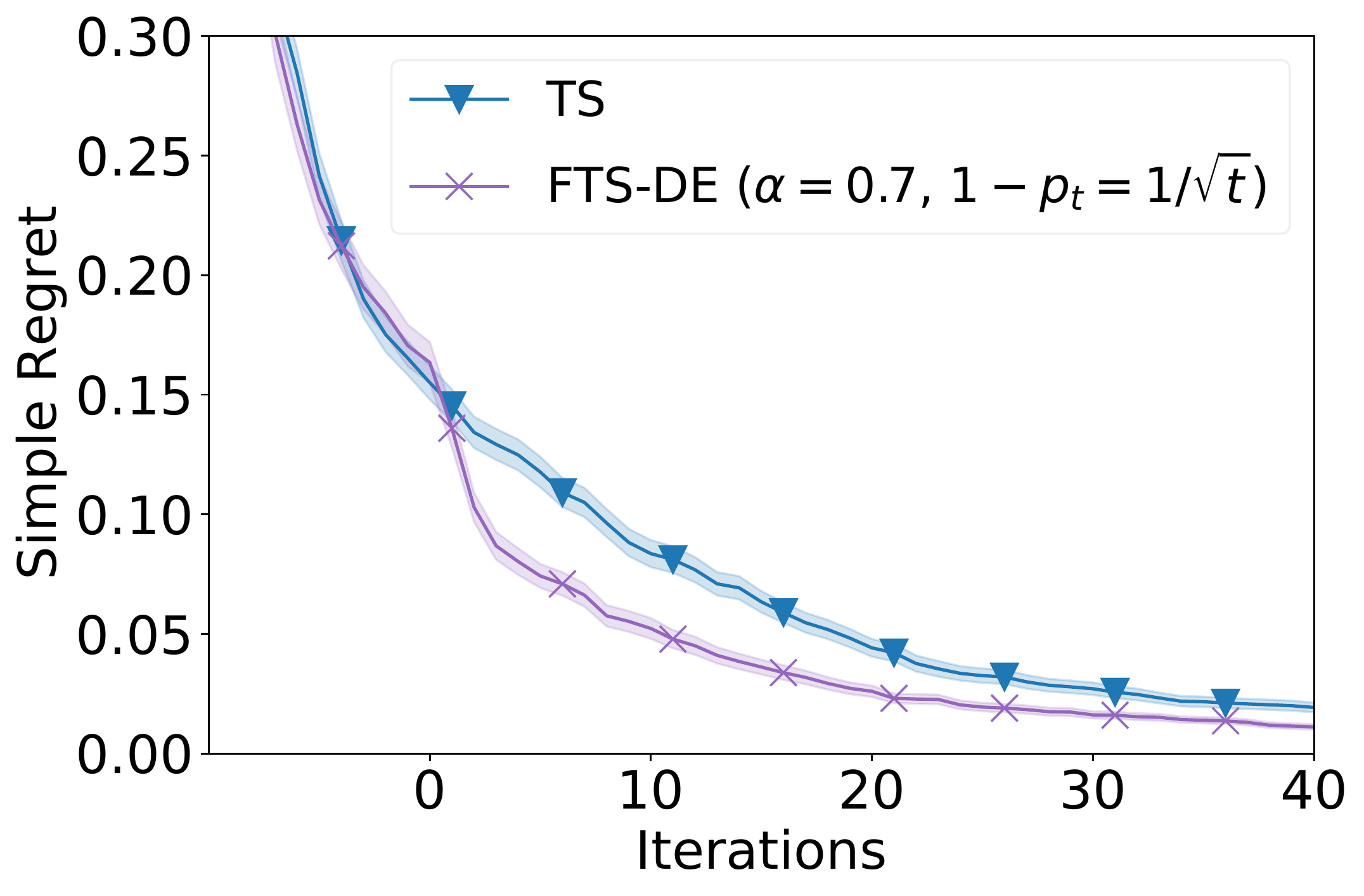} & \hspace{-4mm} 
         \includegraphics[width=0.4\linewidth]{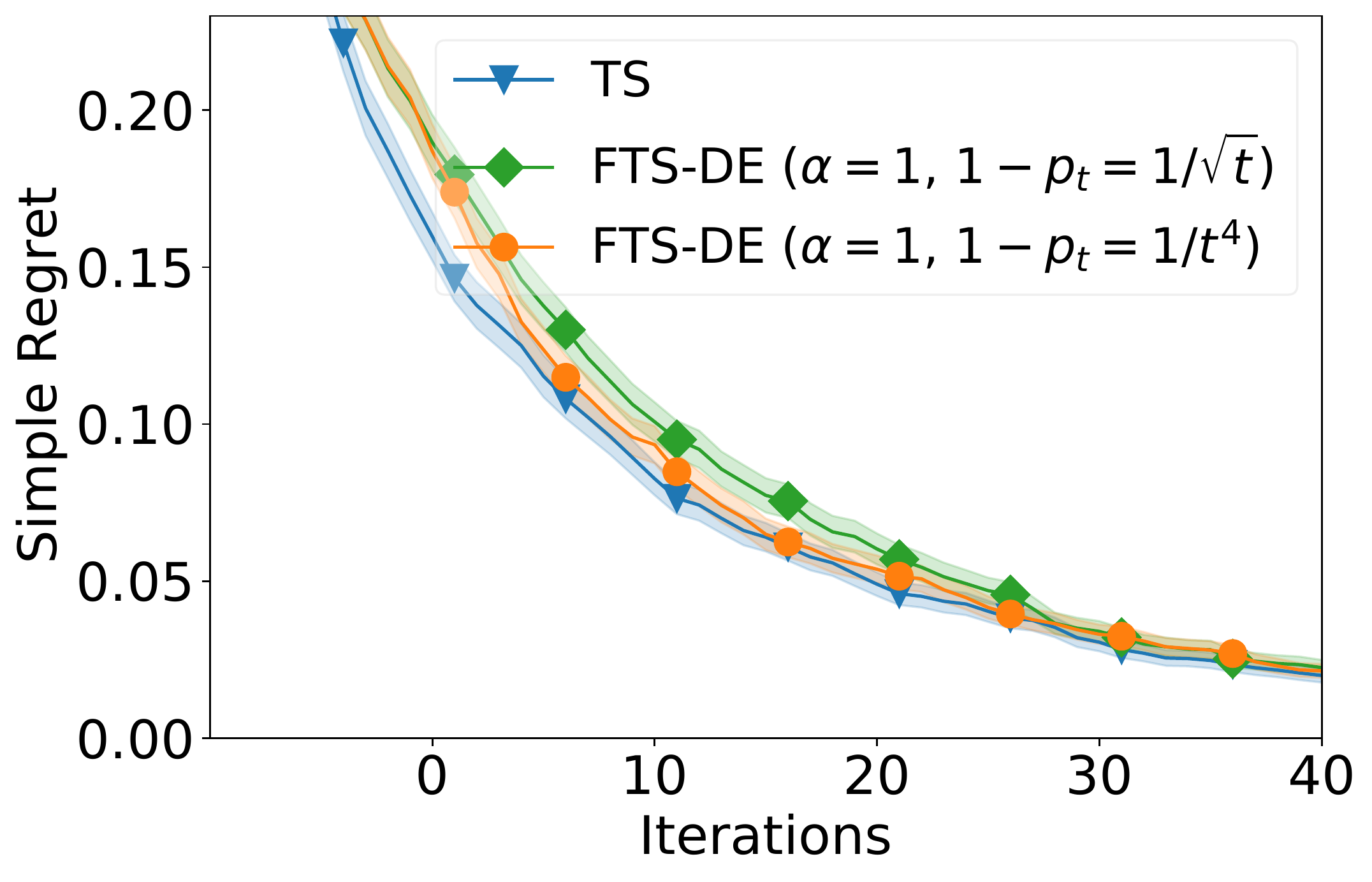} \\
         {(a)} & {(b)}
     \end{tabular}
     \caption{Results when the agents are heterogeneous, i.e., when the objective functions of different agents are significantly different. (a) and (b) correspond to $\alpha=0.7$ and $\alpha=1.0$, respectively.}
     \label{fig:exp:hetero:agents}
\end{figure}

\subsection{Real-world Experiments}
\label{app:real_exp}
\subsubsection{More Experimental Details}
In all real-world experiments, when generating the random features for the RFFs approximation, we use the SE kernel with a length scale of $0.01$ and a variance of $\sigma^2=10^{-6}$ for the observation noise.
Refer to~\cite{dai2020federated} and~\cite{rahimi2008random} for more details on how the random features are generated and how they are shared among all agents.

As we have mentioned in the main text, we use $P=4$ sub-regions in all three real-world experiments, and divide the entire domain into $P=4$ hyper-rectangles (i.e., sub-regions) with equal volumes.
Following the common practice in BO, we assume that the domain $\mathcal{X}\in\mathbb{R}^D$ is a $D$-dimensional hyper-rectangle, and w.l.o.g., assume that every dimension of the domain is normalized into the range $[0,1]$. That is, the domain can be represented as $[0,1]^D=\{[0,1],[0,1],\ldots,[0,1]\}$.
Note that every domain which is a hyper-rectangle can be normalized into this form.
As a result, when the input dimension is $D=2$ (i.e., the landmine detection experiment), we construct the $P=4$ hyper-rectangles such that $\mathcal{X}_1=\{[0, 0.5), [0, 0.5)\}$, $\mathcal{X}_2=\{[0, 0.5), [0.5, 1.0]\}$, $\mathcal{X}_3=\{[0.5, 1.0], [0, 0.5)\}$ and $\mathcal{X}_4=\{[0.5, 1.0], [0.5, 1.0]\}$.
Similarly, when the input dimension $D=3$ (i.e., the human activity recognition and EMNIST experiments), we construct the $P=4$ hyper-rectangles such that $\mathcal{X}_1=\{[0, 0.5), [0, 0.5), [0, 1]\}$, $\mathcal{X}_2=\{[0, 0.5), [0.5, 1.0], [0, 1]\}$, $\mathcal{X}_3=\{[0.5, 1.0], [0, 0.5), [0, 1]\}$ and $\mathcal{X}_4=\{[0.5, 1.0], [0.5, 1.0], [0, 1]\}$.

The \emph{landmine detection} dataset\footnote{\url{http://www.ee.duke.edu/~lcarin/LandmineData.zip}.} used in this experiment has also been used by the works of~\cite{dai2020federated,smith2017federated} which focus on FBO and FL respectively.
This dataset consists of the landmine detection data from $N=29$ landmine fields (agents), and the task of every agent is to use a support vector machine (SVM) to detect (classify) whether a location in its landmine field contains landmines or not (i.e., binary classification).
We tune two hyperparameters of SVM, i.e., the RBF kernel parameter ($[0.01,10.0]$) and the penalty parameter ($[10^{-4},10.0]$).
For every landmine field, we use half of its dataset as the training set and the remaining half as the validation set.
In this experiment, we report the area under the receiver operating curve (AUC) as the performance metric, instead of validation error, because this dataset is significantly imbalanced, i.e., the vast majority of the locations do not contain landmines. No data is excluded. Refer to~\cite{xue2007multi} for more details on this dataset.
The dataset is publicly available, and contains no personally identifiable information or offensive content.

The \emph{human activity recognition} dataset\footnote{\url{https://archive.ics.uci.edu/ml/datasets/Human+Activity+Recognition+Using+Smartphones}.} was originally introduced by the work of~\cite{anguita2013public} and has also been adopted by the works of~\cite{dai2020federated,smith2017federated}.
The licensing requirement for this dataset requires that the use of this dataset in publications must be acknowledged by referencing the paper of~\cite{anguita2013public}.
This dataset consists of the data collected using mobile phone sensors when $N=30$ subjects (agents) are performing six different activities.
The task of every agent (subject) is to use the dataset generated by the subject to perform activity recognition, i.e., to predict which one of the six activities the subject is performing using logistic regression (LR).
We tune three hyperparameters of LR: the batch size ($[128,512]$), L2 regularization parameter ($[10^{-6},10]$) and learning rate ($[10^{-6},1]$).
For every subject, we again use half of its data as the training set and the other half as the validation set, and the validation error is reported as the performance metric.
The inputs for every agent are standardized by removing the mean and dividing by the standard deviation of its training set, which is a common pre-processing step for LR.
No data is excluded.
Refer to~\cite{anguita2013public} for more details on this dataset.
The dataset is publicly available as described above, and does not contain personally identifiable information or offensive content.

\emph{EMNIST}\footnote{\url{https://www.nist.gov/itl/products-and-services/emnist-dataset}.} is a dataset of images of handwritten characters from different persons, and is a widely used benchmark in FL~\cite{kairouz2019advances}.
The EMNIST dataset is under the CC0 License.
Here we use the images from the first $N=50$ subjects (agents) which can be accessed from the TensorFlow Federated library\footnote{\url{https://www.tensorflow.org/federated}.}.
Every subject (agent) uses a convolutional neural network (CNN) to learn to classify an image into one of the ten classes corresponding to the digits $0-9$.
Here the task for every agent is to tune three CNN hyperparameters: learning rate, learning rate decay and L2 regularization parameter, all in the range of $[10^{-7},0.02]$.
We follow the standard training/validation split offered by the TensorFlow Federated library for every agent, and again use the validation error as the performance metric.
All images are pre-processed by normalizing all pixel values into the range of $[0,1]$, and no data is excluded.
Refer to~\cite{cohen2017emnist} for more details on this dataset.
The dataset is publicly available, and contains no personally identifiable information or offensive content.

\subsubsection{Comparison between DP-FTS-DE and DP-FTS}
We have shown in the main text (Figs.~\ref{fig:real_exp}a,b,c) that FTS-DE significantly outperforms FTS without DE.
Here we further verify in Fig.~\ref{fig:har_compare_with_dp_fts} the importance of the technique of DE after DP is integrated, using the human activity recognition experiment.
Specifically, the figures show that after the incorporation of DP, DP-FTS-DE (green curves in all three figures) still achieves a better utility than DP-FTS (purple curves) for the same level of privacy guarantee (loss).
Note that same as Fig.~\ref{fig:synth_3}, to facilitate a fair comparison, we have used a smaller value of $S$ for DP-FTS without DE such that a similar percentage of vectors are clipped for both DP-FTS-DE and DP-FTS.
\begin{figure}
     \centering
     \begin{tabular}{ccc}
         \includegraphics[width=0.31\linewidth]{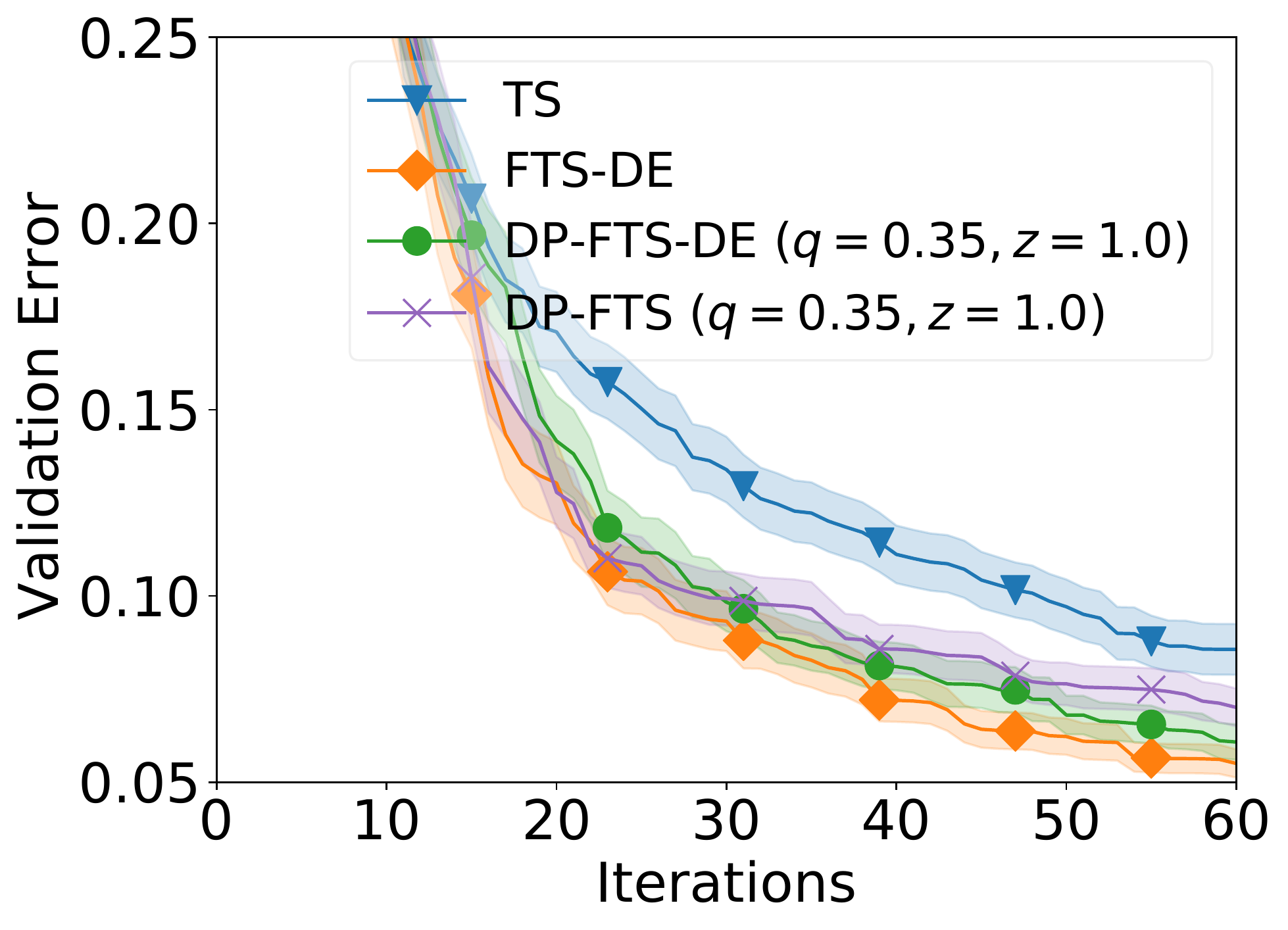} & \hspace{-4mm} 
         \includegraphics[width=0.31\linewidth]{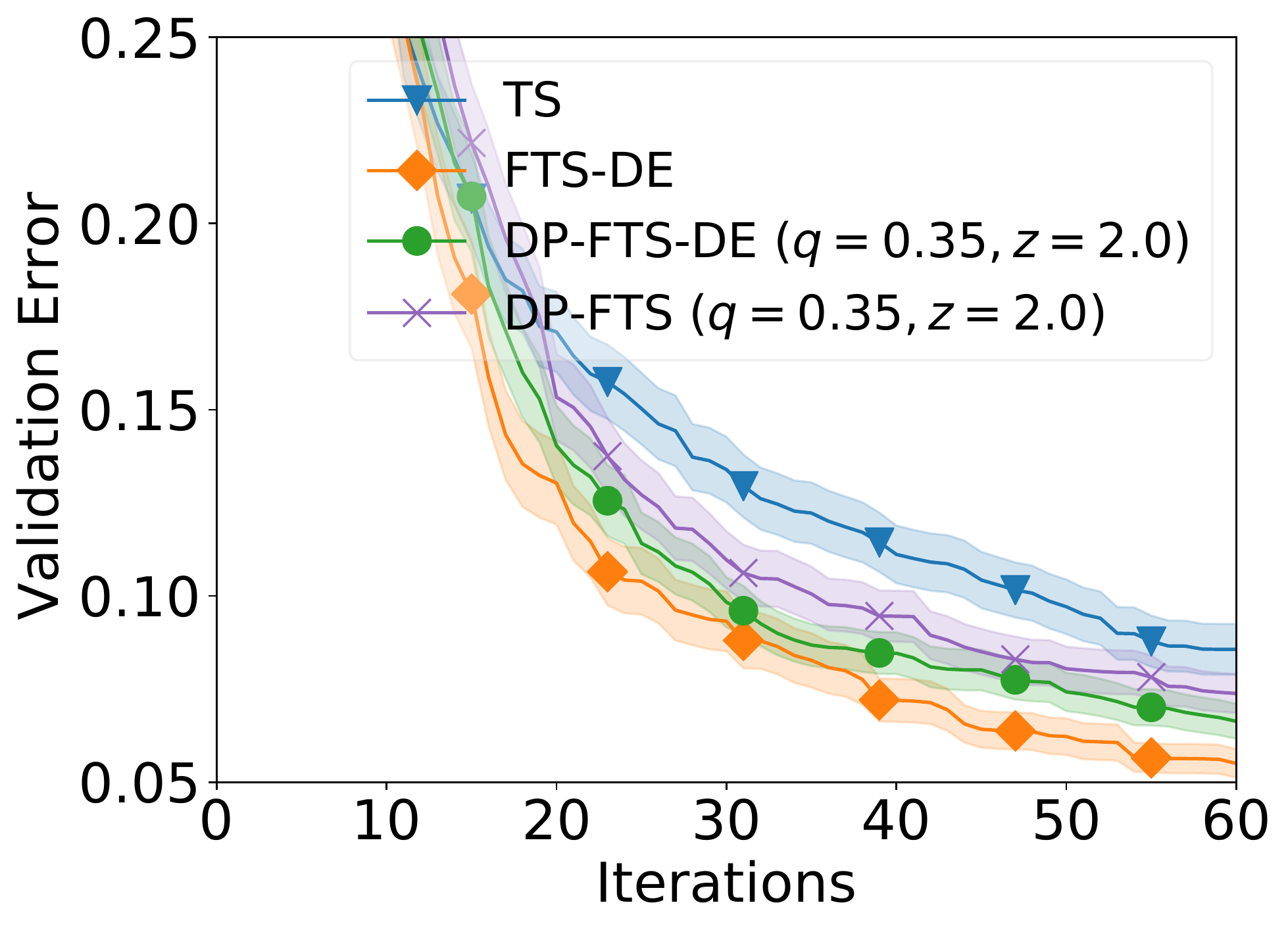} & \hspace{-4mm} 
         \includegraphics[width=0.31\linewidth]{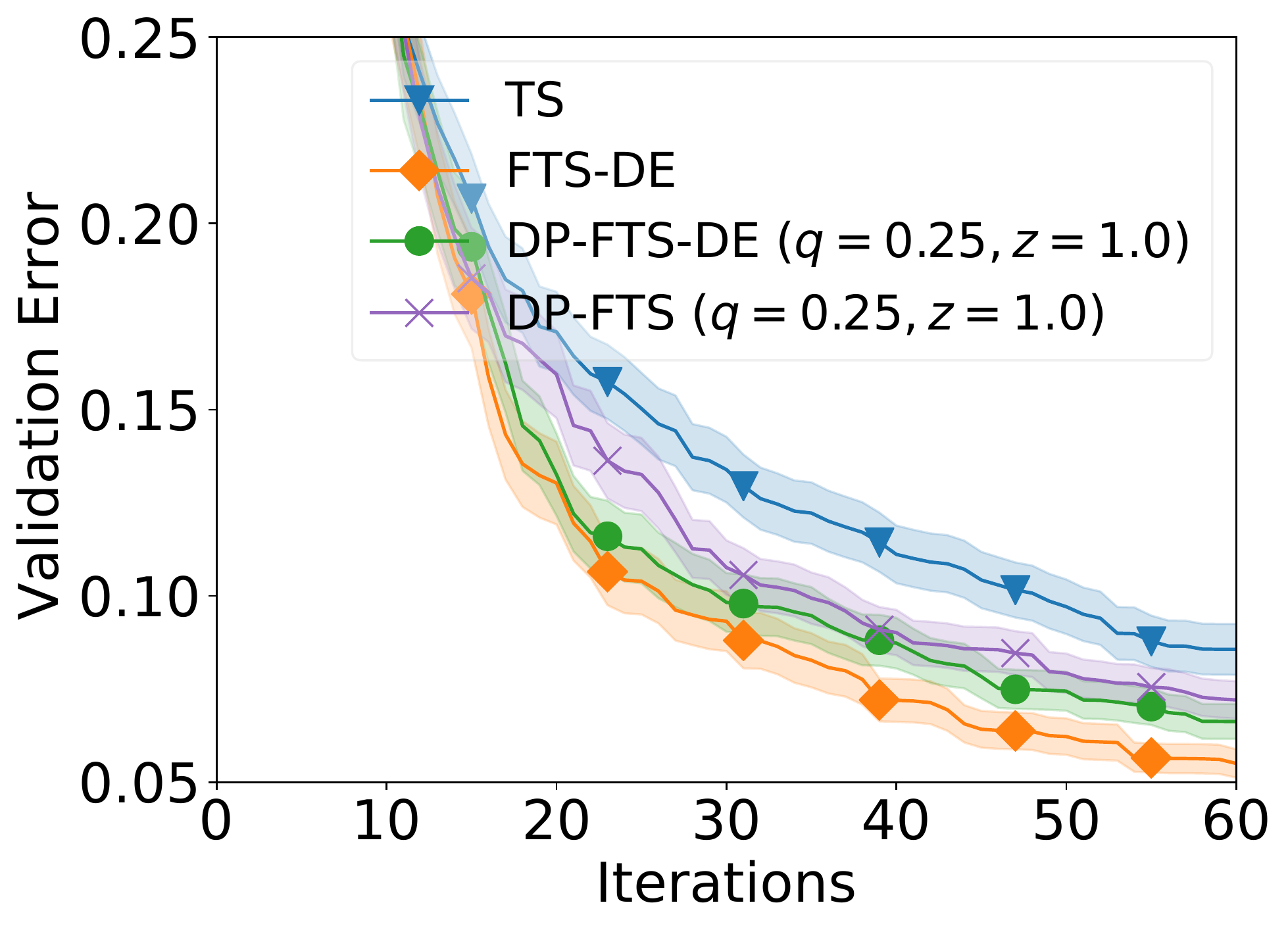} \\
         {(a)} & {(b)} & {(c)}
     \end{tabular}
     \caption{Comparison between DP-FTS and DP-FTS-DE for the same level of privacy guarantee (the human activity recognition experiment).
     We have used $S=22$ and $S=11$ for DP-FTS-DE and DP-FTS (without DE) respectively, such that a similar percentage of vectors are clipped in both cases: $1.02\%$ for DP-FTS-DE and $1.09\%$ for DP-FTS.}
     \label{fig:har_compare_with_dp_fts}
\end{figure}

\subsubsection{Robustness against the Choice of Weights and Number of Sub-regions}
\label{app:real:exp:different:weights}
In this section, we evaluate the robustness of our experimental results against the choice of the weights assigned to different agents and the number $P$ of sub-regions, using the human activity recognition and EMNIST experiments.

Here we test three other methods for designing the weights: 
(1) We use the same softmax weighting scheme with a different parameter $a=9$ instead of $a=15$ ("weights $1$" in Fig.~\ref{fig:different:weights}).
(2) For a sub-region $\mathcal{X}_i$, we assign a weight $\propto a$ to those agents exploring $\mathcal{X}_i$ and $\propto b$ to the other agents, and similarly gradually decay the value of $a=1,000$ to $b=1$ ("weights $2$" in Fig.~\ref{fig:different:weights}).
(3) For a sub-region $\mathcal{X}_i$, we assign a weight $\propto a^2$ to those agents exploring $\mathcal{X}_i$ and $\propto b$ to the other agents, and similarly gradually decay the value of $a=40$ to $b=1$ ("weights $3$" in Fig.~\ref{fig:different:weights}).
Moreover, in addition to the results using $P=4$ sub-regions reported in the main text, here we also evaluate the performance of FTS-DE and DP-FTS-DE with $P=2,3,6$ sub-regions (Fig.~\ref{fig:different:P}).
All DP-FTS-DE methods in Fig.~\ref{fig:different:weights} and Fig.~\ref{fig:different:P} correspond to $q=0.35,z=1.0$.
The results 
demonstrate the robustness of our experimental results against the choice of the weights and the number $P$ of sub-regions.

\begin{figure}
     \centering
     \begin{tabular}{cc}
         \hspace{-4mm} 
         \includegraphics[width=0.4\linewidth]{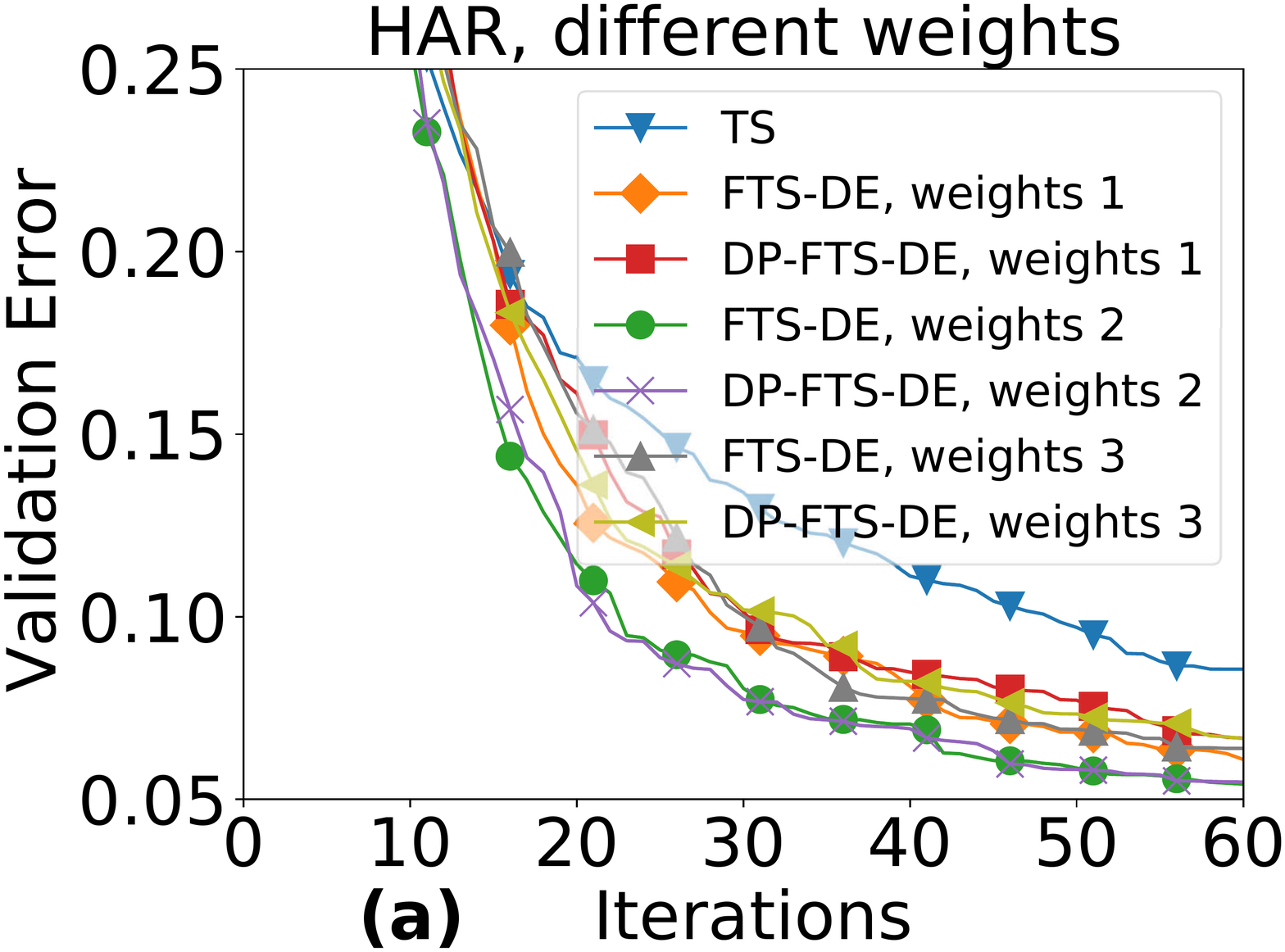} & \hspace{-4mm} 
         \includegraphics[width=0.4\linewidth]{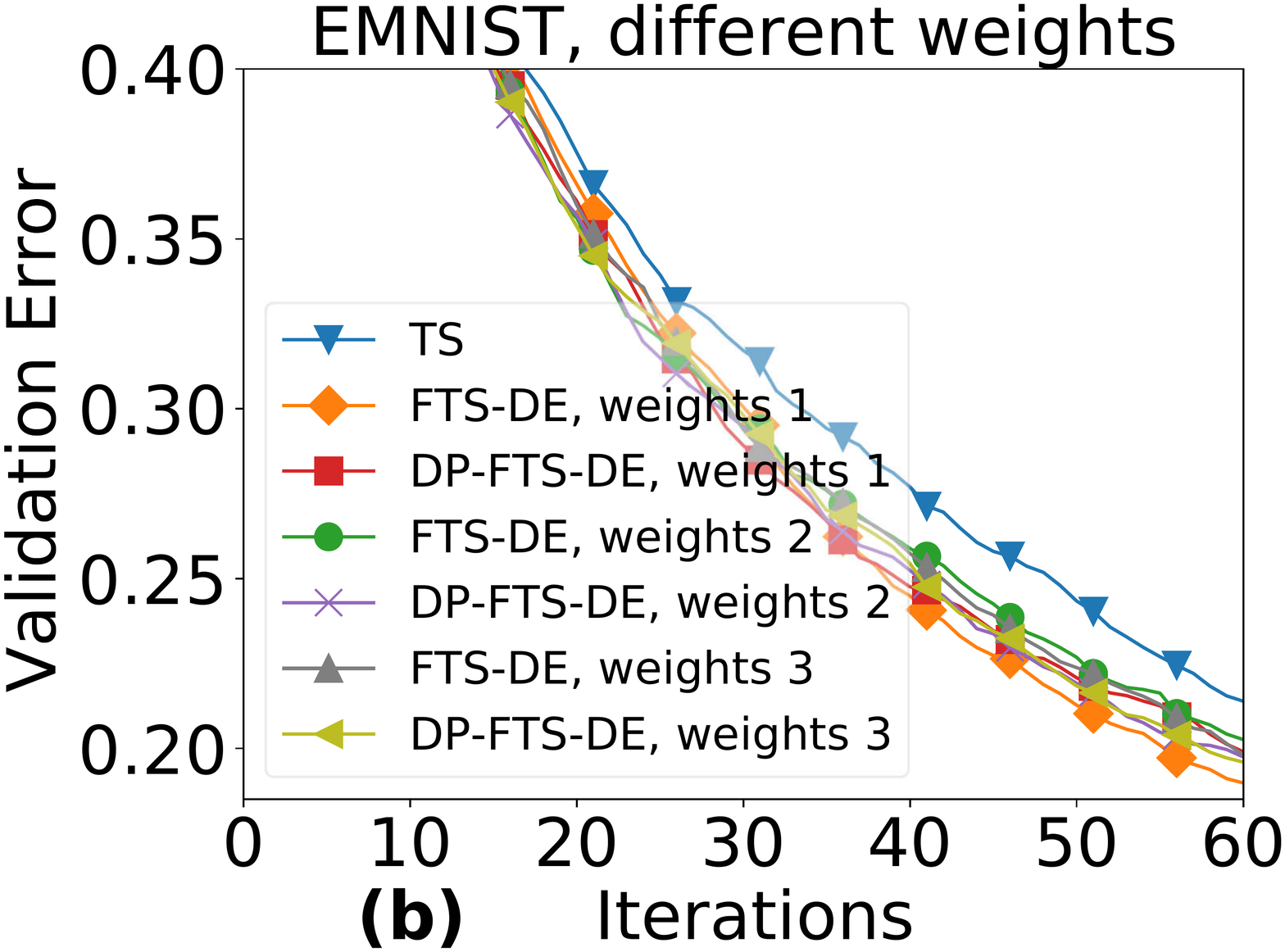} \\
         {(a)} & {(b)}
     \end{tabular}
     \caption{Robustness of our FTS-DE and DP-FTS-DE algorithms against the choice of weights. 
     }
     \label{fig:different:weights}
\end{figure}

\begin{figure}
     \centering
     \begin{tabular}{cc}
         \hspace{-4mm} 
         \includegraphics[width=0.4\linewidth]{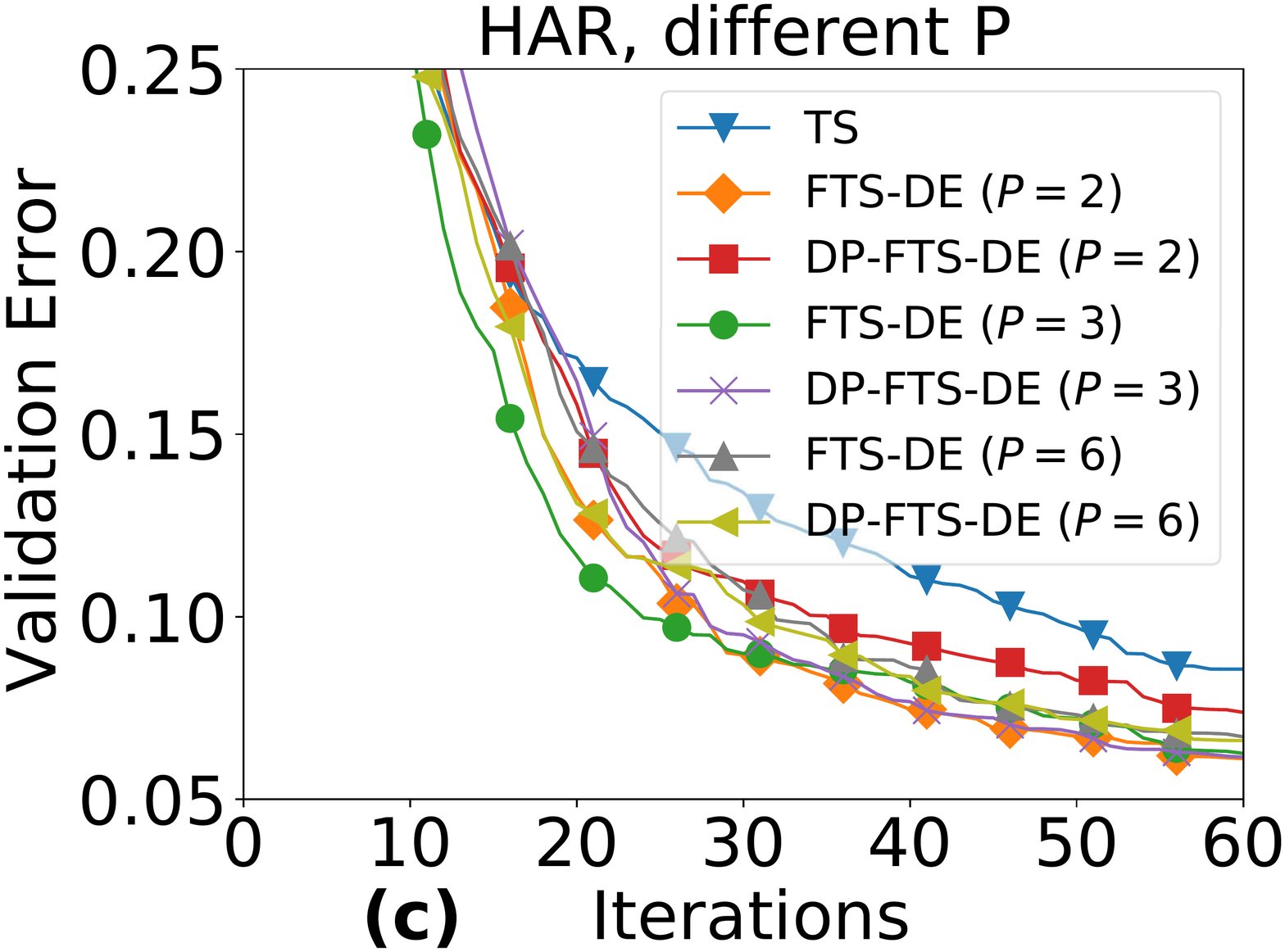} 
         & \hspace{-4mm} 
         \includegraphics[width=0.4\linewidth]{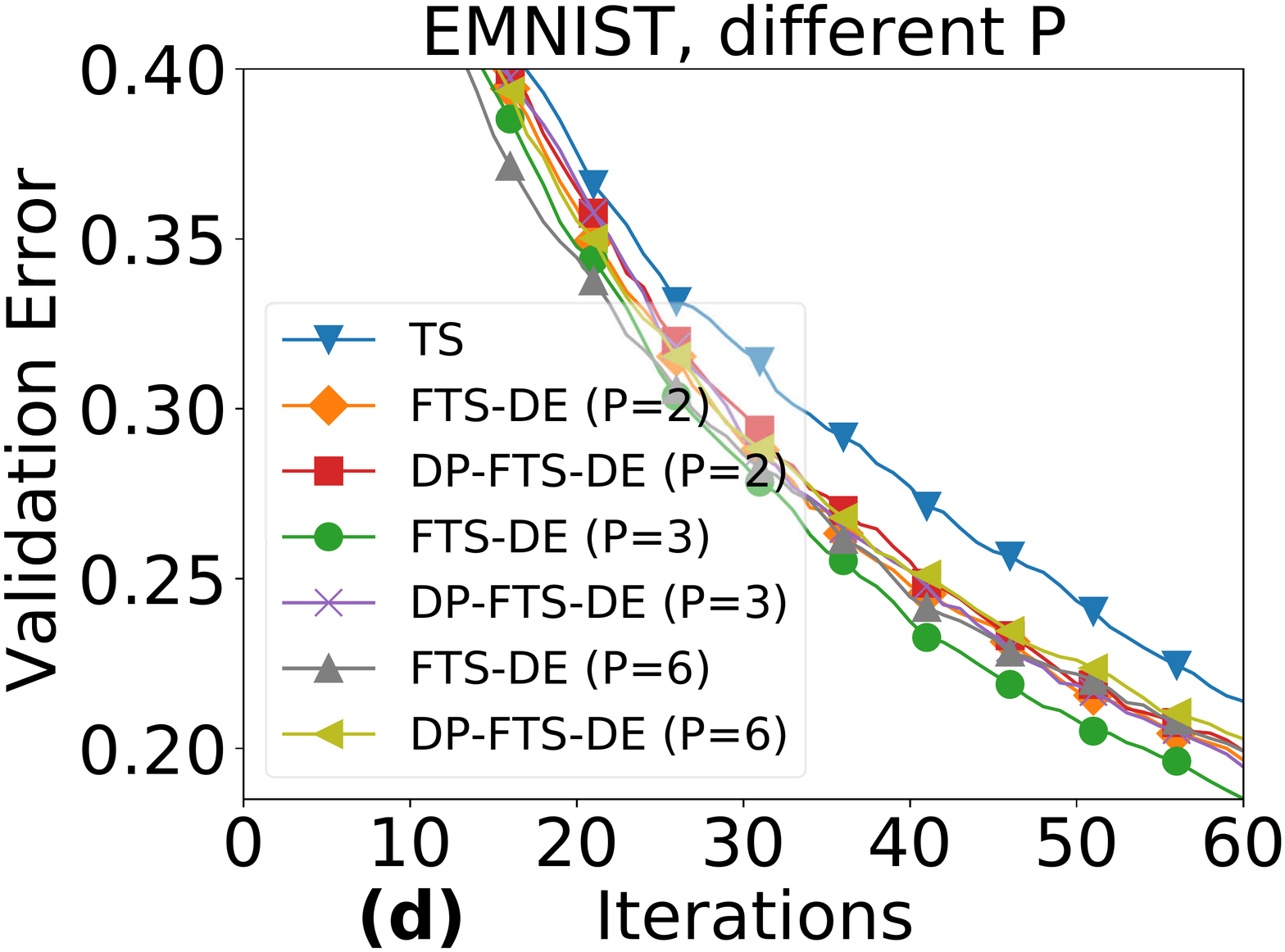} \\
         {(a)} & {(b)}
     \end{tabular}
     \caption{Robustness of our FTS-DE and DP-FTS-DE algorithms against the number $P$ of sub-regions. 
     }
     \label{fig:different:P}
\end{figure}

\subsubsection{R\'enyi DP}
Fig.~\ref{fig:rdp} shows the privacy-utility trade-off in the landmine detection experiment using R\'enyi DP~\cite{wang2019subsampled}. The results demonstrate that R\'enyi DP, despite requiring modifications to our theoretical analysis (i.e., proof of Theorem~\ref{theorem:dp_fts_de}), leads to slightly better privacy losses (compared with Fig.~\ref{fig:real_exp}a) with comparable utilities.
\begin{figure}
     \centering
     \begin{tabular}{c}
         \includegraphics[width=0.36\linewidth]{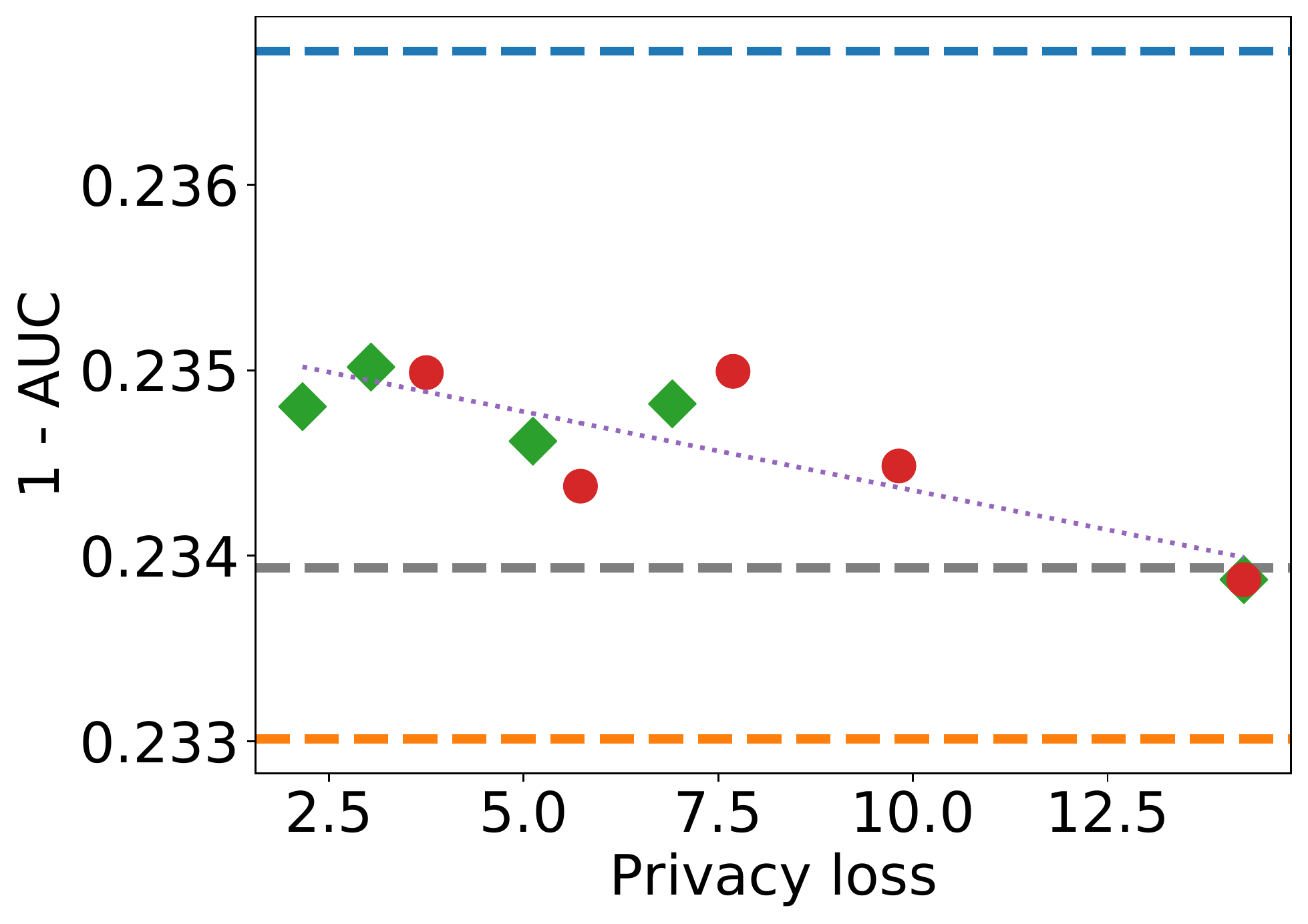}
     \end{tabular}
     \caption{Results for the landmine detection experiment using R\'enyi DP~\cite{wang2019subsampled}.}
     \label{fig:rdp}
\end{figure}

\subsubsection{Adaptive Weights vs. Non-adaptive Weights}
As we have discussed in the last paragraph of Sec.~\ref{subsec:de}, we have designed the set of weights for every sub-region to be adaptive such that they gradually become uniform among all agents as $t$ becomes large.
Here we explore the performance of our algorithm if the weights are \emph{non-adaptive}, i.e., for every sub-region $\mathcal{X}_i$, we fix the set of weights $\{\varphi^{(i)}_n,\forall n\in[N]\}$ for all $t\in[T]$.
In particular, we adopt the same softmax weighting scheme as described in the first paragraph of App.~\ref{app:experiments}, but fix the temperature $\mathcal{T}_t=1,\forall t\geq 1$ such that the same set of weights is used for all $t\geq 1$.
That is, for every sub-region $\mathcal{X}_i$, we assign more weights to those agents exploring $\mathcal{X}_i$ throughout all iterations $t\geq1$.

Fig.~\ref{fig:compare_with_non_adaptive_weights} shows the comparisons between adaptive and non-adaptive weights using (a) the synthetic experiment and (b) human activity recognition experiment.
Both figures show that although in the initial stage, DP-FTS-DE with non-adaptive weights performs similarly to DP-FTS-DE with adaptive weights, however, as $t$ becomes large, adaptive weights (red curves) lead to better performances than non-adaptive weights (green curves).
This can be attributed to the fact that as $t$ becomes large, every agent is likely to have explored (and become informative about) more sub-regions in addition to the sub-region that it is assigned to explore at initialization.
Therefore, if the weights are non-adaptive, i.e., for a sub-region $\mathcal{X}_i$, if after $t$ has become large, most weights are still given to those agents that are assigned to explore $\mathcal{X}_i$ at initialization, then \emph{the information from the other agents} who are likely to have become informative about $\mathcal{X}_i$ (i.e., have collected some observations in $\mathcal{X}_i$) \emph{is not utilized}.
This under-utilization of information might explain the performance deficit caused by the use of non-adaptive weights.
However, note that despite being outperformed by DP-FTS-DE with adaptive weights, DP-FTS-DE with non-adaptive weights (green curve) is still able to consistently outperform standard TS (blue curves).

\begin{figure}
     \centering
     \begin{tabular}{cc}
         \includegraphics[width=0.38\linewidth]{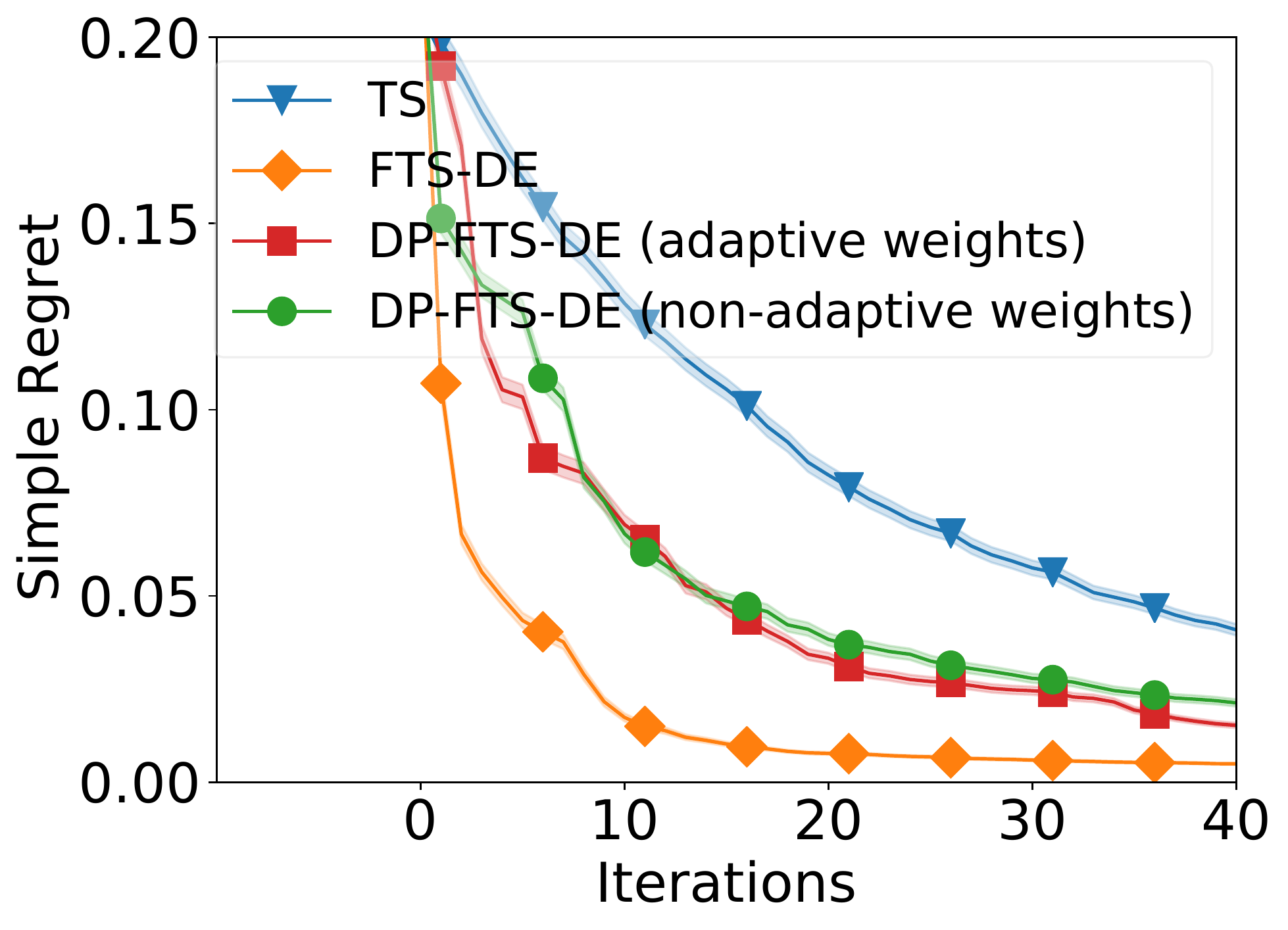} & \hspace{-4mm} 
         \includegraphics[width=0.38\linewidth]{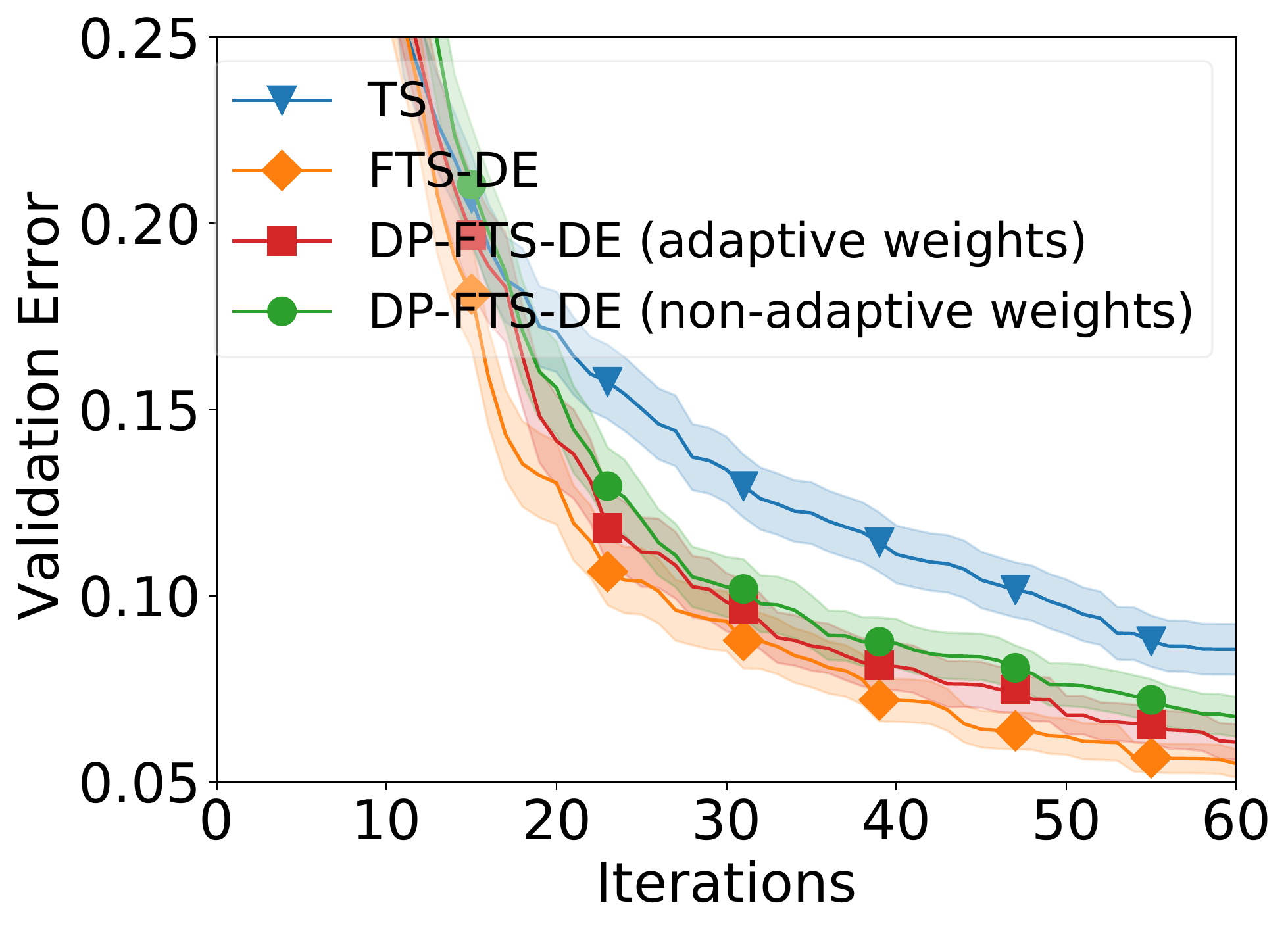} \\
         {(a)} & {(b)}
     \end{tabular}
     \caption{Comparison between DP-FTS-DE with adaptive weights and non-adaptive weights, using (a) the synthetic experiment and (b) human activity recognition experiment.}
     \label{fig:compare_with_non_adaptive_weights}
\end{figure}

\subsubsection{Computational Cost}
When maximizing the acquisition function to select the next query (lines 6 and 8 of Algo.~\ref{alg:BO}), firstly, we uniformly randomly sample a large number (i.e., 1000) of points from the entire domain; next, we use the L-BFGS-B method with $20$ random restarts to refine the optimization.

For the central server, the integration of DP and DE (in expectation) incurs an additional computational cost of $\mathcal{O}(PNq)$.
However, these additional computations are negligible since they only involve simple vector additions/multiplications (lines 5-11 of Algo.~\ref{alg:DP-FTS-DE}).
For agents, the incorporation of DP brings no additional computational cost to them.
Meanwhile, the addition of DE, which affects line 8 of Algo.~\ref{alg:BO}, only incurs minimal additional computations.
For example, in the landmine detection experiment, line 7 takes on average $14.47$s and $17.96$s to compute for $P=1$ and $P=4$, respectively.

\end{document}